\documentclass{article} 
\usepackage{iclr2027_conference,times}

\usepackage{amsmath,amsfonts,bm}

\def\eqref#1{equation~\ref{#1}}
\def\Eqref#1{Equation~\ref{#1}}

\def\1{\bm{1}}

\DeclareMathAlphabet{\mathsfit}{\encodingdefault}{\sfdefault}{m}{sl}
\SetMathAlphabet{\mathsfit}{bold}{\encodingdefault}{\sfdefault}{bx}{n}

\usepackage{hyperref}
\usepackage{url}

\usepackage[dvipsnames,table,xcdraw]{xcolor}
\usepackage{hyperref}       
\usepackage{url}            
\usepackage{booktabs}       
\usepackage{amsfonts}       
\usepackage{nicefrac}       
\usepackage{microtype}      
\usepackage{xcolor}         
\usepackage{algorithm}
\usepackage{algorithmic}

\usepackage{wrapfig}
\usepackage{amsmath}
\usepackage{amssymb}
\usepackage{mathtools}
\usepackage{amsthm}
\usepackage{tikz}
\usepackage{xspace}
\usepackage{adjustbox}

\usetikzlibrary{arrows.meta,positioning,fit}
\theoremstyle{plain}
\newtheorem{theorem}{Theorem}[section]
\newtheorem{proposition}[theorem]{Proposition}
\newtheorem{lemma}[theorem]{Lemma}

\theoremstyle{definition}
\newtheorem{definition}[theorem]{Definition}
\newtheorem{assumption}[theorem]{Assumption}
\newtheorem{remark}[theorem]{Remark}
\usepackage[textsize=tiny]{todonotes}
\usepackage[most]{tcolorbox}

\usepackage{enumitem}

\newcommand{\methodlong}{debiased neural operator\xspace}
\newcommand{\method}{\mbox{DOPE}\xspace}
\newcommand{\greentext}[1]{\textcolor{ForestGreen}{#1}}

\definecolor{scalarRed}{HTML}{A50040}

\newcommand{\rebuttal}[1]{\textcolor{black}{#1}}

\newcommand*\circledblue[1]{\tikz[baseline=(char.base)]{
            \node[shape=circle,draw=NavyBlue!60,fill=NavyBlue!10,thick,inner sep=1pt] (char) {\scriptsize\textsf#1};}}

\newcommand*\circledscalarRed[1]{\tikz[baseline=(char.base)]{
        \node[shape=circle,draw=scalarRed!60,fill=scalarRed!10,thick,inner sep=1pt] (char) {\scriptsize\textsf#1};}}

\newcommand*\diff{\mathop{}\!\mathrm{d}}
\newcommand*\Diff{\mathop{}\!\mathrm{D}}
\author{
\textbf{Konstantin Hess}\textsuperscript{1,2,*},
\textbf{Dennis Frauen}\textsuperscript{1,2},
\textbf{Niki Kilbertus}\textsuperscript{2,3,4},
\textbf{Stefan Feuerriegel}\textsuperscript{1,2}\\[0.8em]
\textsuperscript{1}LMU Munich \quad
\textsuperscript{2}Munich Center for Machine Learning \quad
\textsuperscript{3}TU Munich \quad
\textsuperscript{4}Helmholtz Munich \\
\\
\textsuperscript{*}{Corresponding author: \texttt{k.hess@lmu.de}}
}

\title{Debiased Neural Operators}

\iclrfinalcopy 
\begin{document}

\maketitle

\begin{abstract}
Neural operators are widely used to approximate solution maps of complex physical systems. In many applications, however, the goal is not to recover the full solution trajectory, but to summarize the solution trajectory via a scalar target quantity (e.g., a functional such as time spent in a target range, time above a threshold, accumulated cost, or total energy). In this paper, we introduce \method (\emph{\methodlong}): a semiparametric estimator for such target quantities of solution trajectories obtained from neural operators. \method is broadly applicable to settings with both partial and irregular observations and can be combined with arbitrary neural operator architectures. We make three main contributions. (1)~We show that, in contrast to \method, na\"ive plug-in estimation can suffer from first-order bias. (2)~To address this, we derive a novel one-step, Neyman-orthogonal estimator that treats the neural operator as a high-dimensional nuisance mapping between function spaces, and removes the leading bias term. For this, \method uses a weighting mechanism that simultaneously accounts for irregular observation designs and for how sensitive the target quantity is to perturbations of the underlying trajectory. (3)~To learn the weights, we extend automatic debiased machine learning to \emph{operator-valued nuisances} via Riesz regression. We demonstrate the benefits of \method across various numerical experiments. 
\end{abstract}

\section{Introduction}

\begin{wrapfigure}{r}{0.5\textwidth}
\vspace{-1.25cm}
  \centering
  \includegraphics[width=0.5\textwidth, trim=0cm 0cm 0cm 0cm, clip]{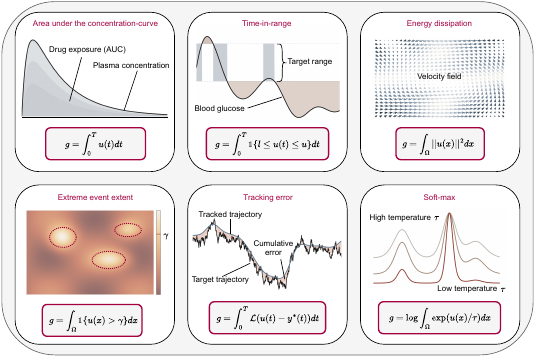}
\vspace{-0.5cm}
\caption{\textbf{Examples:} In many scientific applications, solution trajectories are summarized by \textcolor{scalarRed}{scalar quantities} (i.e., functionals $g$).} 
\vspace{-0.3cm}
\label{fig:functional_examples}
\end{wrapfigure}

Neural operators \citep{Kovachki.2023NO, Li.2021NO, Lu.2021NO} are a powerful framework for learning solution maps of complex physical systems governed by partial differential equations (PDEs). In particular, neural operators can efficiently learn mesh-independent approximations of PDE solutions that generalize across varying input functions, including different boundary conditions and forcing terms \citep{Azizzadenesheli.2024NO, Rosofksy.2023NO}. As a result, they are increasingly used in various applications such as scientific computing \citep{Azizzadenesheli.2024NO, Cao.2024NO, Faroughi.2024NO}, climate science \citep{Kurth.2023NO,  Peng.2024NO, Wen.2023NO}, and engineering \citep{Kobayashi.2024NO, Shukla.2024NO, Wang.2025NO}.


In this paper, \textbf{\textit{we aim to learn a scalar target quantity (i.e., a functional) of the solution trajectory}}. This is relevant in many scientific applications, where the main quantity of interest is not the full solution trajectory produced by a PDE solver or its neural operator approximation. Instead, the aim is often to summarize the trajectory into a \textcolor{scalarRed}{scalar target quantity} that is directly relevant for decision-making or scientific interpretation (see Figure~\ref{fig:functional_examples}). 

\underline{\textbf{Examples:}} \circledscalarRed{1} In medicine, physicians care about whether a patient’s vital signs remain within a target range, such as maintaining blood glucose above a critical threshold or keeping heart rate and respiratory rate within a safe range \citep{Hashiguchi.2023FUNC, Iqbal.2022FUNC, Yang2024.FUNC, Zhou.2025FUNC}. \circledscalarRed{2} In climate modeling, researchers quantify whether forecasts exceed critical thresholds, such as extreme temperatures \citep{Ilan.2025FUNC} \circledscalarRed{3} In fluid dynamics and energy systems, the goal is to measure total fluxes or energy over time \citep{Grachev.2020FUNC}. \circledscalarRed{4} In control systems, engineers quantify the accumulated tracking error or other stability metrics over time \citep{Dong.2025FUNC, Widdershins.2026FUNC}. Importantly, such functionals typically depend on the trajectory in complex, nonlinear ways.

{\emph{Why a na{\"i}ve estimation is problematic:}} A na{\"i}ve strategy for estimating such functionals is the following (which we later refer to as \textit{plug-in estimator}): first train a neural operator, then evaluate the functional of interest on the predicted trajectory, and use the result as the final estimate. However, we later show that this na\"ive plug-in estimator can be systematically biased, even when the neural operator predicts the solution trajectory reasonably well. The reason is that the quantity of interest is a nonlinear summary of the trajectory, so approximation errors in the learned solution map do not simply average out but instead propagate through the functional and thus accumulate into a non-negligible error in the final estimate. We refer to this systematic error as  \emph{plug-in bias} \citep{Curth.2020, Kennedy.2022}.



In this paper, we introduce the \emph{\methodlong} (\method): a semiparametric framework for estimating scalar summaries of solution trajectories with neural operators. Our \method framework is broadly applicable to settings with both partial and irregular observations. Further, \method is model-agnostic and can be instantiated with \emph{any} neural operator architecture (e.g., such as Fourier neural operators \citep{Li.2021NO} or DeepONet \citep{Lu.2021NO}), and is therefore broadly applicable across scientific applications. 

To construct \method, we derive a novel one-step, Neyman-orthogonal estimator that treats the neural operator as a high-dimensional nuisance mapping between function spaces. Specifically, we use the Riesz representer to build a score that is first-order insensitive to errors in the learned neural operator. As a result, approximation errors of the neural operator affect the estimated functional only through higher-order terms, and thus remove the plug-in bias. Methodologically, \method corrects the neural-operator-induced bias using pointwise residuals of the solution trajectory, and, for this, uses a weighting mechanism that simultaneously accounts for irregular observation designs and for how sensitive the target quantity is to perturbations of the underlying trajectory. To learn the weights, we propose a tailored automatic debiasing approach based on Riesz regression and automatic differentiation.

\textbf{Contributions.}
In summary, our contributions are as follows:\footnote{Code is available at \url{https://anonymous.4open.science/r/1C71}.}
\begin{description}[leftmargin=*]
    \item[\circledblue{1}] We develop a novel semiparametric framework called \method for inference on \rebuttal{both population average and input-specific} scalar functionals of neural operator outputs. Crucially, \method removes the plug-in bias of na{\"i}ve estimation strategies. 

    \item[\circledblue{2}] We further develop an automatic debiasing procedure for operator-valued nuisance functions via Riesz regression and automatic differentiation. In this way, we provide a non-trivial extension of automatic debiased machine learning (DML) \citep{Chernozhukov.2021, Chernozhukov.2022} from \textit{function-valued} to \textit{operator-valued} nuisances.

    \item[\circledblue{3}] We establish various theoretical properties of \method, including asymptotic normality and confidence intervals (under standard regularity conditions). Further, our framework easily integrates prediction-powered inference (PPI) \citep{angelopoulos.2023, angelopoulos.2023a} for neural operators. 
    
\end{description}

\section{Related Work}

$\bullet$~\textbf{Plug-in bias:} 
In many statistical problems, estimating the target parameter of interest requires first estimating nuisance functions (i.e., auxiliary functions that are not themselves of direct interest). A common issue is that errors in these nuisance estimates propagate into the final estimate, often leading to systematic bias---commonly referred to as \emph{plug-in bias} \citep{Curth.2020, Kennedy.2022}. 

Orthogonal estimation \citep{Foster.2023, Morzywolek.2023} addresses the problem of plug-in bias by constructing estimators that are insensitive to small errors in the nuisance functions. Intuitively, this means that approximation errors do not propagate as first-order biases, but only at second order in the target quantity. In semiparametric theory, this is typically achieved by deriving influence-function-based moment conditions \citep{Hines2022, Ichimura.2022}. To remove the first-order bias, one then aims to derive a one-step estimator \citep{Kennedy.2022}, which ``corrects'' (i.e., debiases) the plug-in estimate with the sample average of a debiasing score, and thereby removes the first-order bias term.

The problem of plug-in bias has been studied in several settings, such as  causal machine learning \citep{Curth.2020, Kennedy.2022}, but \underline{\textbf{not}} for neural operators. Importantly, existing works typically focus on settings that have \emph{function-valued} nuisances \cite[e.g.,][]{Hess.2026b, Hess.2026c, Melnychuk.2024b}. $\Rightarrow$ \textit{In contrast, our setting involves \textbf{operator-valued} nuisances, which introduces new challenges and requires tailored estimation techniques.}

$\bullet$~\textbf{Automatic debiasing:}  
Orthogonal estimation has attractive second-order bias properties, but applying it often requires deriving the \textit{debiasing score} analytically. A remedy is automatic debiasing, which aims to construct this correction directly from the target functional \citep{Carone.2019, Luedtke.2024}. This idea can be implemented via automatic debiased machine learning, which reduces debiasing to estimation of an auxiliary function through Riesz regression rather than requiring a closed-form derivation of the correction term \citep{Chernozhukov.2021, Chernozhukov.2022d}. Prior work studies this idea for \emph{function-valued} nuisances using random forests and multi-layer perceptrons \citep{ Chernozhukov.2022}, for regularized Riesz regression \citep{Chernozhukov.2022e}, under covariate shift \citep{Chernozhukov.2023d}, and for smooth functionals of M-estimands \citep{vdL.2025}. \rebuttal{Finally, works such as \citep{Koltchinskii.2026} study smooth Banach-valued functionals; we derive observable corrections for sparsely observed solution fields.}


\section{Problem Formulation}\label{sec:setup}

$\bullet~$\textbf{PDE setup:} 
We consider the standard setting of operator learning \citep{Kovachki.2023NO}. Let $\Omega \subset \mathbb{R}^d$ be a bounded domain.\footnote{We make no formal distinction between time and spatial domain, and w.l.o.g. $\Omega$ can refer to both.}
Let $\mathcal{A}$ denote a Banach space of admissible input functions (e.g., coefficient functions, boundary conditions, or forcing terms). 
For each admissible coefficient function $a \in \mathcal{A}$, we consider a PDE of the generic form
\begin{align}
\mathcal{T}_a u = f,
\end{align}
with homogeneous boundary conditions. We assume that, for each $a \in \mathcal{A}$, the PDE admits a unique weak solution
\begin{align}
 \mathcal T_a^{-1}f = u \in \mathcal{U},
\end{align}
where, throughout, we take $\mathcal U=L^2(\Omega,\mu)$, which is the Hilbert space of square-integrable functions on $\Omega$ with inner product induced by Lebesgue measure $\mu$. We denote by
\begin{align}
S_0 : \mathcal{A} \to \mathcal{U}, 
\qquad 
S_0(a)  = u(a, \cdot),
\end{align}
the \emph{unknown solution operator}, which maps input functions to solution trajectories.

$\bullet~$\textbf{Neural operators:} Neural operators are learned approximations of the solution operator $S_0$ \citep{Kovachki.2021NO, Li.2021NO, Lu.2021NO}. Given training data consisting of input functions and partial solution observations, neural operators learn a mapping
$\hat{S} : \mathcal{A} \to \mathcal{U},$ 
such that $\hat{S}(a) \approx S_0(a)$. They operate directly on functional inputs and outputs defined on the continuous domain $\Omega$, which can be evaluated at arbitrary locations $x\in\Omega$. We provide an extended related work on neural operators in Supplement~\ref{sec:extended_related_work}.

$\bullet~$\textbf{Partial and irregular measurements:}
In many scientific workflows, \emph{a central challenge} is that the evaluation locations depend on the input through measurement protocols, adaptive solvers, or experimental constraints \citep{Berg.2019, Zhou.2025FUNC}. Hence,  solution trajectories are often observed only \emph{partially and irregularly} \citep{Chen.2018, Hess.2025, Hess.2024, Iakovlev.2023}. 
Accordingly, let $A$ denote a random coefficient function with associated solution trajectory $S_0(A)\in\mathcal U$. Instead of observing the full trajectory, we observe noisy point evaluations
\begin{align}
Y_k = S_0(A)(X_k) + \varepsilon_k,
\qquad k=1,\ldots,K,
\end{align}
where $K\in\mathbb N$ may be random and depend on $A$, where $X_k\in\Omega$ is an evaluation location, and where the noise satisfies
$\mathbb E[\varepsilon_k\mid A,X_k]=0$ 
and
$\mathbb E[\varepsilon_k^2\mid A,X_k]<\infty$.
We model the sampling mechanism through a joint law $\mathbb P$ on $\mathcal A\times\Omega$ with factorization
$\mathbb P(\diff a,\diff x)=\mathbb P(\diff x\mid a)\,\mathbb P(\diff a)$,
where $\mathbb P(\cdot\mid a)$ is an input-dependent design distribution on $\Omega$. Given $(A,K)$, we assume that the pairs $\{(X_k,Y_k)\}_{k=1}^K$ are conditionally i.i.d.\ draws from  a generic pair $(X,Y)$ with
$X\mid A=a\sim \mathbb P(\cdot\mid a)$.
We observe $n$ i.i.d.\ realizations
\begin{align}
O_i=\Big(A_i,K_i,\{(X_{ik},Y_{ik})\}_{k=1}^{K_i}\Big),
\qquad i=1,\ldots,n,
\end{align}
and write
$\mathcal D=\{O_i\}_{i=1}^n$.
Throughout, all $L^2(\mathbb P)$ norms are taken with respect to the joint law of $(A,X)$. This setting captures the common regime where generating new inputs is relatively cheap, while dense evaluations of the corresponding solution trajectory are costly \citep{Meyer.2023PDE, Kim.2025PDE}.

$\bullet$~\textbf{Identifiability:}
In order to ensure identifiability from observations sampled under $\diff\mathbb P$, we impose a standard design overlap assumption \citep{Curth.2021, Hess.2025}. 
\rebuttal{Crucially, overlap is required for nonparametric identification of the underlying solution operator itself. Without it, \emph{neither the full operator nor arbitrary functionals thereof are identifiable without additional structural assumptions, \textbf{regardless} of the estimator used.}}

\begin{assumption}[Design overlap]\label{assumption:design}
\emph{Let $\mu$ denote a fixed reference measure on $\Omega$ (e.g., Lebesgue measure). Whenever it exists, we define the corresponding {inverse design weight} as}
\begin{align}
\xi_0(a)(x) \coloneqq \frac{\diff \mu}{\diff \mathbb{P}(\cdot\mid a)}(x).
\end{align}

\emph{Then, for $\mathbb{P}$-almost every $a \in \mathcal{A}$, the measure $\mu$ is absolutely continuous with respect to $\mathbb{P}(\cdot\mid a)$. Moreover, there exists a constant $C<\infty$ such that}
$0 < \xi_0(a)(x) \le C$ for 
$\mathbb{P}(\cdot\mid a)\text{-a.e. }x \in \Omega$.
\end{assumption}

Finally, for $\mathbb P$-almost every $a\in\mathcal A$, we define the conditional Hilbert space
$\mathcal H_a \coloneqq L^2(\Omega,\mathbb P(\cdot\mid a))$
with inner product
$\langle f,h\rangle_a \coloneqq \int_\Omega f(x)h(x)\,\diff\mathbb P(x\mid a)$.
Under Assumption~\ref{assumption:design}, $\mathcal H_a$ {embeds continuously into} $\mathcal U=L^2(\Omega,\mu)$.

$\bullet~$\textbf{Objective:}
Let $g:\mathcal{U}\to\mathbb{R}$ be our scalar-valued functional of interest, and let $g$ be twice Fr\'echet differentiable in a neighborhood of $S_0(A)$ with uniformly bounded second derivative, i.e.,
\begin{align}\label{ass:g_smooth}
g(u+h) = g(u)+\Diff g_{u}(h) + R(u,h)
\end{align}
with $|R(u,h)| \lesssim \| h\|_{\mathcal{U}}^2$ for all $u$ in this neighborhood and all sufficiently small $h$.\footnote{We use standard notation $a(x)\lesssim b(x)$ to denote that there exists $C>0$ such that $a(x)\leq C b(x) $ for all $x$.} This is naturally satisfied by many common functionals in practice, including linear functionals, polynomial energy-type functionals, norm-based functionals, or weighted transport functionals.

\begin{tcolorbox}[
  colback=BrickRed!8,
  colframe=BrickRed!75!black,
  boxrule=0.8pt,
  arc=1.5mm,
  left=1.2mm,
  right=1.2mm,
  top=1mm,
  bottom=1mm,
  title=\textbf{Objectives},
  fonttitle=\bfseries,
  coltitle=white,
  enhanced,
  breakable
]
We aim to perform inference on the population average scalar functional
\begin{align}\label{eq:target}
\theta \coloneqq \mathbb{E}\!\left[g(S_0(A))\right],
\end{align}
\rebuttal{and later extend our method to the input-specific functional 
\begin{align}\label{eq:target_input_specific}
    \theta(a) \coloneqq \mathbb{E}\!\left[g(S_0(A)) \mid A=a \right],
\end{align}}
where the expectation is taken with respect to the population distribution of $A$. \rebuttal{The population average functional quantifies, for example, expected physical responses under uncertain inputs and have been central to uncertainty quantification, stochastic PDEs, and engineering optimization \citep{Ghanem.1991, Vidal-Codina.2016}. The latter characterizes the expected conditional functional for an individual input. Our framework accommodates both objectives.}

\end{tcolorbox}

\rebuttal{The key difficulty is that neither the full solution field $S_0(A_i)$ nor its functional value $g(S_0(A_i))$ is directly observed. Unlike standard settings in automatic debiased ML, our framework must therefore recover the relevant first-order functional information from sparse, input-dependently sampled measurements.}

\begin{tcolorbox}[
  colback=Gray!8,
  colframe=Gray!75!black,
  boxrule=0.8pt,
  arc=1.5mm,
  left=1.2mm,
  right=1.2mm,
  top=1mm,
  bottom=1mm,
  title=\textbf{Examples (see Figure~\ref{fig:functional_examples})},
  fonttitle=\bfseries,
  coltitle=white,
  enhanced,
  breakable
]

$\bullet~$\emph{Medicine.} 
$A$ encodes subject- or drug-specific parameters (e.g., from pharmacokinetics); $S_0(A)(t)$ is the concentration--time curve; and $g(S_0(A))$ represents the time spent in a target range or above a clinically relevant threshold.

$\bullet~$\emph{Fluid dynamics and energy systems.}
$A$ represents coefficients, forcing terms, or boundary conditions; $S_0(A)(x)$ is the resulting field; and $g(S_0(A))$ represents the flux over a spatial region or energy over time.

$\bullet~$\emph{Control systems.}
$A$ parameterizes system dynamics or disturbances; $S_0(A)(t)$ is the resulting trajectory; and $g(S_0(A))$ represents accumulated tracking error or another scalar measure of stability and performance over time.

\end{tcolorbox}

By the Riesz representation theorem \citep{Rieszrepresenter}, for every Fr\'echet differentiable functional 
$g:\mathcal U\to\mathbb R$ and every $u\in\mathcal U$, there exists a unique 
$w_g(u)\in\mathcal U$ such that
\begin{align}
\Diff g_u(h)
=
\langle w_g(u),h\rangle_{\mathcal U}
=
\int_\Omega w_g(u)(x)\,h(x)\, \diff \mu(x).
\end{align}

\section{\method: Debiased estimation of functionals for neural operators}

In the following, we focus on debiased estimation of the scalar functional in \Eqref{eq:target}. 
For this, we proceed as follows: \circledblue{1}~We first show why the na\"ive estimator leads to plug-in bias (§\ref{sec:plugin}). \circledblue{2}~We then introduce \method: a \textit{one-step Neyman-orthogonal estimator} (see Figure \ref{fig:method}). For this, we derive the \emph{debiasing score} by expressing the Fréchet derivative as an inner product with a Riesz representer, which yields a first-order debiasing procedure (§\ref{sec:onestep}). \circledblue{3}~We show several desirable \emph{theoretical properties} of \method (§\ref{sec:theory}). \rebuttal{Further, we show how our theory easily extends to the input-specific functionals.} \circledblue{4}~Finally, we present a tailored approach for \emph{automatic DML with operator-valued nuisances} using the Riesz representer (§\ref{sec:autodml}). We provide the pseudo-code in in Supplement~\ref{sec:algorithm}.


\subsection{Why na\"ive plug-in estimation is problematic}\label{sec:plugin}

A na\"ive estimation strategy for \Eqref{eq:target} is to estimate $S_0(A)$ with a neural operator $\hat{S}(A)$ and treat the surrogate as the ground-truth operator. However, such plug-in estimation leads to severe plug-in bias:
\begin{tcolorbox}[
  colback=Gray!8,
  colframe=Gray!75!black,
  boxrule=0.8pt,
  arc=1.5mm,
  left=1.2mm,
  right=1.2mm,
  top=1mm,
  bottom=1mm,
  fonttitle=\bfseries,
  coltitle=white,
  enhanced,
  breakable
]
\begin{proposition}[First-order bias of plug-in estimators]\label{prop:pluginbias}
Let $\Delta(\cdot)=\hat S(\cdot)-S_0(\cdot)$.
Then
\begin{align}
\big| \mathbb{E}[g(\hat S(A))]-\theta \big|
\;\lesssim\;
\big| \mathbb{E}\!\left[\Diff g_{S_0(A)}(\Delta(A))\right] \big|
\;+\;
\|\Delta\|_{L^2(\mathbb P)}^2,
\end{align}
i.e., the neural operator error propagates linearly through the Fr\'echet differential, and induces \textbf{first-order bias}.
\end{proposition}
\begin{proof}
    See Supplement~\ref{appendix:proofs_plugin}.
\end{proof}
\end{tcolorbox}

Proposition~\ref{prop:pluginbias} shows that the estimation error of the plug-in estimator is, in general, linear in $\mathbb{E}[\Diff g_{S_0(A)}(\Delta(A))]$ and thus does not vanish. As a result, root-$n$ inference from the plug-in estimator requires this first-order bias to be negligible, i.e., of order $o(n^{-1/2})$. To understand what this implies, suppose the Fr\'echet derivative admits a uniformly bounded Riesz representer, so that 
$|\Diff g_{S_0(A)}(\Delta(A))|\lesssim \|\Delta(A)\|_{L^2(\mu)}$. Then, by overlap, this implies that the linear term can be controlled by
$
\big|\mathbb{E}[\Diff g_{S_0(A)}(\Delta(A))]\big|
\lesssim
\|\Delta\|_{L^2(\mathbb P)} .
$
Hence, a sufficient condition for root-$n$ inference is that the neural operator error satisfies
$
\|\Delta\|_{L^2(\mathbb P)} = o(n^{-1/2}).
$
However, such rates are typically \rebuttal{\emph{difficult to guarantee in practice}} in practice. 
Intuitively, the problem of the plug-in estimator is that nonlinear functionals propagate operator errors \emph{linearly} through the Fr\'echet derivative, so even small approximation errors accumulate.


\subsection{Debiased estimation via \method}\label{sec:onestep}

As a remedy, we now derive a \textbf{one-step Neyman-orthogonal estimator} (\method). 
In contrast to the na\"ive plug-in estimator from Proposition~\ref{prop:pluginbias}, \method achieves root-$n$ convergence under substantially weaker accuracy requirements on the neural operator.

The starting point is the leading bias term of the na\"ive plug-in estimator,
$\mathbb{E}\!\left[\Diff g_{S_0(A)}\big(\Delta(A)\big)\right]$,
where $\Delta(\cdot) = S_0(\cdot) - \hat{S}(\cdot)$ denotes the operator error. 
To eliminate the plug-in bias, we seek a correction term that \emph{recovers the value of the functional derivative} $\Diff g_{S_0(A)}(\Delta(A))$ using observable quantities.

\begin{wrapfigure}{r}{0.56\textwidth}[h]
\vspace{-0.6cm}
  \centering
  {\includegraphics[width=0.56\textwidth, trim=0.6cm 0.2cm 0.6cm 0.2cm, clip]{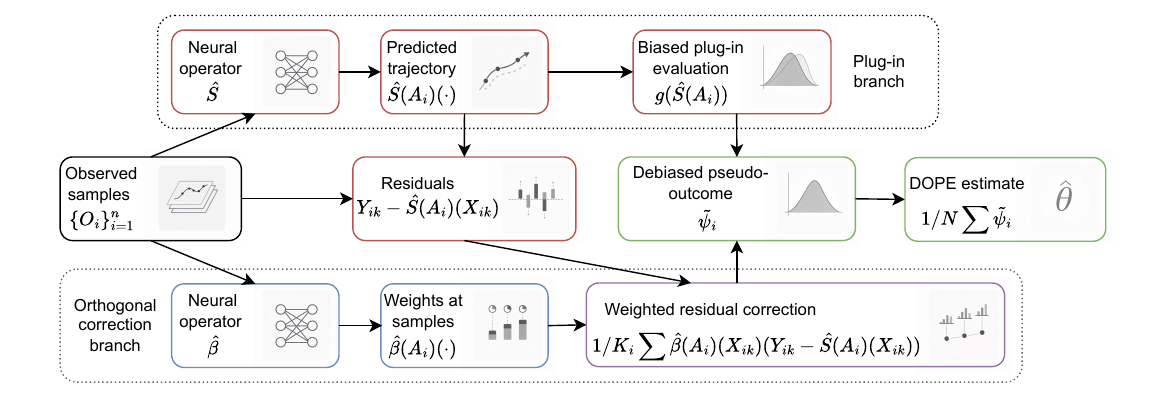}}
\vspace{-0.6cm}
\caption{\textbf{Overview of \method.}} 
\vspace{-0.3cm}
\label{fig:method}
\end{wrapfigure}

However, this is \textbf{non-trivial}: the perturbation $\Delta(A)$ is \textit{not} observed as a function, but only through noisy point evaluations $Y_k - \hat S(A)(X_k)$ at locations $X_k$ drawn from the sampling design. Hence, we can \textit{not} directly evaluate the functional derivative.

Instead, we leverage that $\Diff g_{S_0(A)}(\cdot)$ is a linear functional and can therefore be represented via a Riesz representer. That is, for $\mathbb P$-almost every $a\in\mathcal A$, there exists a function $\beta_0(a)(\cdot)\in \mathcal H_a$ which is the unique Riesz representer satisfying
\begin{align}\label{eq:riesz}
\mathbb{E}\!\left[\beta_0(A)(X)\,h(X)\mid A=a\right]
=
\Diff g_{S_0(a)}(h),
\end{align}
for all $h\in \mathcal H_a$. This representation is crucial: it expresses the functional derivative as an expectation \textit{over pointwise evaluations under the sampling design}. Applying \Eqref{eq:riesz} with generic $A\in \mathcal A$ and $h=\Delta(A)$ shows that the weighted residual $\beta_0(A)(X)\big(Y - \hat S(A)(X)\big)$ recovers the leading plug-in bias in expectation. This suggests the \textit{debiasing score}
\begin{align}\label{eq:score}
\psi(O)
=
g(S_0(A))-\theta
+
\frac{1}{K}\sum_{k=1}^K
\beta_0(A)(X_k)\big(Y_k-S_0(A)(X_k)\big),
\end{align}
which we formally derive in Lemma~\ref{lem:score_derivation}. Intuitively, the weighted residual term reconstructs the first-order bias $\Diff g_{S_0(A)}(\Delta(A))$ and thus provides the required correction. 
In Supplement~\ref{appendix:proofs_onestep}, we show that the debiasing score $\psi(O)$ is Neyman-orthogonal at the truth. 
Motivated by this, we now define \method.

\begin{tcolorbox}[
  colback=NavyBlue!8,
  colframe=NavyBlue!75!black,
  boxrule=0.8pt,
  arc=1.5mm,
  left=1.2mm,
  right=1.2mm,
  top=1mm,
  bottom=1mm,
  fonttitle=\bfseries,
  coltitle=white,
  enhanced,
  breakable
]
\begin{definition}[\method]\label{def:onestep}
Given i.i.d.\ observations $\{O_i\}_{i=1}^n$, consider the debiasing score in \Eqref{eq:score}.
We replace the unknown quantities $(S_0,\beta_0)$ by estimators $(\hat S,\hat\beta)$. We further define the pseudo-outcome
\begin{align}
\tilde\psi_i
=
g(\hat S(A_i))
+
\frac{1}{K_i}\sum_{k=1}^{K_i}
\hat\beta(A_i)(X_{ik})
\big(Y_{ik}-\hat S(A_i)(X_{ik})\big).
\end{align}
In \method, the \textbf{one-step estimator} of the target functional in \Eqref{eq:target} is given by 
\begin{align}
\hat\theta
=
\frac{1}{n}\sum_{i=1}^n \tilde\psi_i .
\label{eq:onestep}
\end{align}
\end{definition}
\end{tcolorbox}

\rebuttal{The factor $1/K_i$ normalizes the residual correction within each subject, so that subjects with different numbers of observations do
not contribute disproportionately. In contrast, the factor $1/n$ assigns equal weight to subjects drawn from the target
population. Hence, reweighting by $K_i$ could distort this population distribution and thereby change the estimand.}


\subsection{Theoretical properties}
\label{sec:theory}

The key property of \method is \textit{Neyman-orthogonality}: first-order errors
in the neural operator $\hat{S}$ and in the correction term $\hat{\beta}$ do \textbf{not} affect the leading asymptotics. As a result, nuisance errors enter the remainder only at \emph{second-order}. We show this formally in Theorem~\ref{thm:remainder}.

\begin{tcolorbox}[
  colback=ForestGreen!8,
  colframe=ForestGreen!75!black,
  boxrule=0.8pt,
  arc=1.5mm,
  left=1.2mm,
  right=1.2mm,
  top=1mm,
  bottom=1mm,
  fonttitle=\bfseries,
  coltitle=white,
  enhanced,
  breakable
]
\begin{theorem}[Second-order remainder]\label{thm:remainder}
Let $\hat\theta$ be the \method estimate in \Eqref{eq:onestep}. Let $\hat S$ and $\hat\beta$ be estimated via cross-fitting. Then, we have
\begin{align}
\big|\mathbb{E}[\hat\theta]-\theta\big|
\;\lesssim\;
\|\hat S-S_0\|_{L^2(\mathbb P)}^2
+
\|\hat\beta-\beta_0\|_{L^2(\mathbb P)}^2.
\end{align}
In particular, there is \textbf{no first-order bias} from either nuisance component.
\end{theorem}
\begin{proof}
    See Supplement~\ref{appendix:proofs_onestep}.
\end{proof}
\end{tcolorbox}

As a result, by construction, \method \emph{removes the linear bias term \rebuttal{despite imperfect nuisance estimation}}, and leaves only a second-order remainder. This is \textit{in contrast} to the plug-in estimator, for which the bias is generally linear in
$\mathbb{E}\!\left[\Diff g_{S_0(A)}(\Delta(A))\right]$ (cf. Proposition~\ref{prop:pluginbias}).
Consequently, root-$n$ inference for \method is possible under \emph{substantially weaker nuisance rates}. For example, it suffices that
\begin{align}
\|\hat S-S_0\|_{L^2(\mathbb P)}^2=o_p(n^{-1/2})
\quad\text{and}\quad
\|\hat\beta-\beta_0\|_{L^2(\mathbb P)}^2=o_p(n^{-1/2}),
\end{align}
which corresponds to $n^{-1/4}$-type rates in $L^2(\mathbb P)$ root-mean-squared error.

\begin{tcolorbox}[
  colback=ForestGreen!8,
  colframe=ForestGreen!75!black,
  boxrule=0.8pt,
  arc=1.5mm,
  left=1.2mm,
  right=1.2mm,
  top=1mm,
  bottom=1mm,
  fonttitle=\bfseries,
  coltitle=white,
  enhanced,
  breakable
]
\rebuttal{
\begin{theorem}[Conditional debiasing for input-specific functionals]
\label{thm:input_specific}
Let $\theta_0(a)=g(S_0(a))$ denote the true input-specific functional.
For nuisance operators $S$ and $\beta$, define the
\textbf{conditional correction}
\begin{align}
r_{S,\beta}(a)
\coloneqq
\mathbb E\big[
\frac{1}{K}\sum_{k=1}^K
\beta(A)(X_k)
\big(Y_k-S(A)(X_k)\big)
\;\mid\; A=a
\big].
\end{align}
Then, the corrected functional
$\theta_{S,\beta}(a)\coloneqq g(S(a))+r_{S,\beta}(a)$
is \textbf{Neyman-orthogonal} at $(S_0,\beta_0)$
with respect to both nuisance components.
Moreover, if $\hat S$ and $\hat\beta$ are estimated via
cross-fitting, their conditional bias is \textbf{second-order}:
\begin{align}
\mid \theta_{\hat S,\hat\beta}(a)-\theta_0(a)\mid
\lesssim
\|\hat S(a)-S_0(a)\|_{L^2(\mathbb P(\cdot\mid a))}^2
+
\|\hat\beta(a)-\beta_0(a)\|_{L^2(\mathbb P(\cdot\mid a))}^2.
\end{align}
\end{theorem}
\begin{proof}
See Supplement~\ref{sec:input_specific}.
\end{proof}
}
\end{tcolorbox}
\rebuttal{Theorem~\ref{thm:input_specific} extends our debiasing framework to the input-specific functional $\theta_0(a)=g(S_0(a))$. It shows that the \emph{conditional correction} also eliminates first-order nuisance estimation bias. In practice, we regress the cross-fitted residual corrections $\tilde\psi_i-g(\hat S(A_i))$ on the inputs $A_i$ to obtain $\hat r(a)$, and construct the debiased predictor
$\hat \theta(a)=g(\hat S(a))+\hat r(a)$. This preserves the plug-in prediction and learns only its
systematic conditional correction, and retains only second-order nuisance bias. {Algorithmic details, including cross-fitting, are in \textbf{Algorithm~\ref{alg:algorithm_input_specific}}.}}


\begin{remark}[Asymptotic normality]
\emph{One can further establish asymptotic normality of \method. We state a formal result in Supplement~\ref{appendix:proofs_onestep}, and provide empirical coverage results in Supplement~\ref{sec:additional_results}.}
\end{remark}

\begin{remark}[PPI with neural operators]
\emph{Our framework easily integrates prediction-powered inference (PPI) \citep{angelopoulos.2023, angelopoulos.2023a}, which preserves the othrogonal correction when additional unlabeled inputs are available. Unlike standard PPI, which uses only \underline{whether} a label is observed, our setting also exploits \underline{where} labels are observed. 
Thus, with $n_1$ partially observed and $n_2$ fully unlabeled inputs, a PPI-enhanced version of \method is}
\begin{align}\label{eq:PPI}
\hat\theta_{\mathrm{PPI}}
=
\frac{1}{n_1+n_2}\sum_{i=1}^{n_1+n_2} g\!\big(\hat S(A_i)\big)
+
\frac{1}{n_1}\sum_{i=1}^{n_1}
\frac{1}{K_i}\sum_{k=1}^{K_i}
\hat\beta(A_i)(X_{ik})\big(Y_{ik}-\hat S(A_i)(X_{ik})\big).
\end{align}
\end{remark}

\subsection{Automatic DML with \method}\label{sec:autodml}

We now turn to the practical question of how to estimate the debiasing weight $\beta_0$ from data. The key insight is that $\beta_0$ admits a structured decomposition that separates properties of the functional from the sampling design.

\begin{tcolorbox}[
  colback=Gray!8,
  colframe=Gray!75!black,
  boxrule=0.8pt,
  arc=1.5mm,
  left=1.2mm,
  right=1.2mm,
  top=1mm,
  bottom=1mm,
  fonttitle=\bfseries,
  coltitle=white,
  enhanced,
  breakable
]
\begin{proposition}[Decomposition of the debiasing weight]\label{prop:beta_decomp}
Suppose $\mathcal U=L^2(\Omega,\mu)$, and let $w_g(u)\in L^2(\Omega,\mu)$ denote the Riesz representer of $\Diff g_u$ with respect to $\mu$. Then, the unique solution $\beta_0(A)(x)$ satisfying \Eqref{eq:riesz} admits, $\mathbb P (\cdot \mid A)$-a.e., the decomposition
\begin{align}
&\beta_0(A)(x)
=
\xi_0(A)(x)\,w_g(S_0(A))(x).\label{eq:decomposition}
\end{align}
\end{proposition}
\end{tcolorbox}
\begin{proof}
    See Supplement~\ref{appendix:proofs_autoDML}.
\end{proof}

This proposition shows that the debiasing weight $\beta_0$ factorizes into two components: (i)~a \emph{functional sensitivity term} $w_g$, which captures how the target functional reacts to local perturbations of the trajectory, and (ii)~a \emph{design correction} $\xi_0$, which accounts for irregular and input-dependent sampling. Hence, a na\"ive way to estimate $\beta_0$ would be to try to derive $w_g$ analytically, and to obtain $\xi_0$ by inverting an estimated sampling design $\mathbb P(\cdot \mid A)$. However, this requires an inversion of the estimated sampling design, which can be \emph{highly unstable} due to division close to $0$.

As a remedy, we show how to estimate the correction term $\beta_0$ in the debiasing score \emph{without} explicitly estimating the sampling distribution  $\mathbb P(\cdot\mid A)$ and performing density inversion. Further, we do not require an analytical derivation of $w_g$. Instead, our approach directly targets the Riesz representer via a \emph{primal variational characterization}, and thus extends automatic DML to operator-valued nuisances and autodiff-based moments.

Recall from \Eqref{eq:riesz} that the correction term $\beta_0$ is defined by
\begin{align}
\mathbb{E}[\beta_0(A)(X)\,h(X)\mid A=a]
=
\Diff g_{S_0(a)}(h),
\qquad \forall h\in \mathcal{H}_a.
\label{eq:riesz_beta_primal}
\end{align}

Rather than enforcing this equality over a class of test functions $h$, we exploit the variational characterization of the Riesz representer. For each $A$, we define the linear functional
$\mathcal L_A(h) \coloneqq \Diff g_{S_0(A)}(h)$.
Then, $\beta_0(A)(\cdot)$ is the unique minimizer of the quadratic functional
\begin{align}
\beta_0(A)(\cdot)
=
\arg\min_{\beta(A)(\cdot)}
\left\{
\mathbb{E}[\beta(A)(X)^2\mid A]
-
2\,\mathcal L_A\big(\beta(A)(\cdot)\big)
\right\}.
\end{align}
Averaging over $A$ yields the population risk
\begin{align}
\beta_0
=
\arg\min_{\beta}
\mathbb{E}\left[
\mathbb{E}[\beta(A)(X)^2\mid A]
-
2\,\Diff g_{S_0(A)}\big(\beta(A)(\cdot)\big)
\right].
\label{eq:primal_riesz}
\end{align}

\begin{tcolorbox}[
  colback=ForestGreen!8,
  colframe=ForestGreen!75!black,
  boxrule=0.8pt,
  arc=1.5mm,
  left=1.2mm,
  right=1.2mm,
  top=1mm,
  bottom=1mm,
  fonttitle=\bfseries,
  coltitle=white,
  enhanced,
  breakable
]
\begin{theorem}[Primal characterization of the Riesz representer]
\label{thm:primal_riesz}
The unique minimizer of \Eqref{eq:primal_riesz} is the Riesz representer $\beta_0$ satisfying
\Eqref{eq:riesz_beta_primal}.
\end{theorem}
\begin{proof}
    See Supplement~\ref{appendix:proofs_autoDML}.
\end{proof}
\end{tcolorbox}

The objective of \Eqref{eq:primal_riesz} directly targets the correction weight
$\beta_0$ required for bias cancellation, without introducing auxiliary test functions $h$ or attempting to estimate the sampling density. In particular, the Riesz identity is enforced implicitly through the optimality conditions of \Eqref{eq:primal_riesz}, which avoids the need for numerically unstable density inversion. \rebuttal{Learning $\hat\beta$ incurs only constant-factor training overhead using automatic  differentiation. When $\xi_0$ and $w_g$ are known, Proposition~\ref{prop:beta_decomp} eliminates this additional training entirely.} \rebuttal{In practice, the empirical Riesz objective can be regularized via by cross-validation using the held-out Riesz loss (see Algorithm~\ref{alg:algorithm}) Importantly, regularization does \textbf{not} alter the Neyman-orthogonal score. Any resulting nuisance estimation error enters the bias only at second order (Theorem~\ref{thm:remainder}).}.

\begin{remark}[Structured parameterization when $w_g$ is known]
{In some cases, the decomposition $\beta_0(A)(x)=\xi_0(A)(x)\,w_g(S_0(A))(x)$ from \Eqref{eq:decomposition} in Proposition~\ref{prop:beta_decomp} admits a closed-form expression for the Riesz representer $w_g(u)$ of the functional derivative $\Diff g_u$ (see Supplement~\ref{sec:riesz_examples}). 
In this case, one may restrict the parameterization of the debiasing weight to} 
$\hat\beta(A)(x)=\hat\xi(A)(x)\,w_g(\hat S(A))(x)$, 
which reduces the complexity of the estimation problem, and can lead to \emph{improved numerical stability and variance reduction.} 
\end{remark}

$\bullet$~\textbf{Implementation details:} We provide a \textbf{pseudo-code} in Algorithm~\ref{alg:algorithm} in Supplement~\ref{sec:algorithm}, and we explain how to efficiently compute the empirical analogue of \Eqref{eq:primal_riesz} via \emph{automatic differentiation} in Supplement~\ref{sec:automatic_differentiation}. Finally, we provide \emph{details on the neural operator backbones} in Supplement~\ref{sec:implementation_details}.

\section{Experiments}\label{sec:experiments}


The purpose of our experiments is three-fold: (i)~We demonstrate \textit{benefits over simple plug-in estimation} across different functionals and sampling irregularities. (ii)~We validate the benefits of debiasing, namely, that it offers \textit{robustness to nuisance errors} as derived in Section~\ref{sec:onestep}. \rebuttal{(iii)~We show that \method transfers its robustness properties to input-specific functionals (Theorem~\ref{thm:input_specific})}. In Supplement~\ref{sec:rwd}, we demonstrate the applicability of \method on \emph{real-world data}. Additional experiments are in Supplement~\ref{sec:additional_results}.


$\bullet$~\textbf{\underline{Implementation:}}  \method is model agnostic and can be instantiated with \textit{any} neural operator architecture. Hence, we thus use a \emph{state-of-the-art backbone}, the Fourier neural operator \citep{Li.2021NO}, in our main results. 
We evaluate \method against the plug-in estimator \rebuttal{as our primary baseline for isolating the effect of debiasing, and additionally against an observation-only estimator using oracle design weights}. For fairness reasons, and to isolate the gains from our debiasing, we use the same neural operator backbones for the plug-in baseline and \method. \rebuttal{The observation-only baseline uses inverse-design weighting  and smooth interpolation of the sparse measurements.} \rebuttal{We use numerical solvers provide the ground-truth references needed for rigorous benchmarking.} Full implementation details are in Supplement~\ref{sec:implementation_details}. All results are averaged over $50$ runs. 

\circledblue{1}~\textbf{\underline{Pharmacokinetics dataset:}} The pharmacokinetics dataset models the drug concentration under a one-compartment model \citep{Kwon.2002PK, Upton.2016PK}. Details are in Supplement~\ref{sec:pk_dgp}.

\begin{table*}[h]
\centering
\setlength{\tabcolsep}{2.5pt} 
\small
\resizebox{\textwidth}{!}{%
\begin{tabular}{@{}lcccccccccccccccccc@{}}
\toprule
& \multicolumn{2}{c}{$\rho=0$}
& \multicolumn{2}{c}{$\rho=0.125$}
& \multicolumn{2}{c}{$\rho=0.25$}
& \multicolumn{2}{c}{$\rho=0.375$}
& \multicolumn{2}{c}{$\rho=0.5$}
& \multicolumn{2}{c}{$\rho=0.625$}
& \multicolumn{2}{c}{$\rho=0.75$}
& \multicolumn{2}{c}{$\rho=0.875$}
& \multicolumn{2}{c}{$\rho=1$} \\
\cmidrule(lr){2-3}\cmidrule(lr){4-5}\cmidrule(lr){6-7}\cmidrule(lr){8-9}
\cmidrule(lr){10-11}\cmidrule(lr){12-13}\cmidrule(lr){14-15}\cmidrule(lr){16-17}\cmidrule(lr){18-19}

& AUC & TAT
& AUC & TAT
& AUC & TAT
& AUC & TAT
& AUC & TAT
& AUC & TAT
& AUC & TAT
& AUC & TAT
& AUC & TAT \\

\midrule
\rebuttal{Observation-only (oracle $\xi_0$)}
& \rebuttal{$0.72\pm0.06$} & \rebuttal{$0.66\pm0.06$}
& \rebuttal{$0.75\pm0.08$} & \rebuttal{$0.64\pm0.06$}
& \rebuttal{$0.81\pm0.08$} & \rebuttal{$0.82\pm0.07$}
& \rebuttal{$1.01\pm0.09$} & \rebuttal{$0.99\pm0.07$}
& \rebuttal{$0.85\pm0.08$} & \rebuttal{$0.90\pm0.08$}
& \rebuttal{$0.81\pm0.07$} & \rebuttal{$0.87\pm0.09$}
& \rebuttal{$0.91\pm0.11$} & \rebuttal{$1.05\pm0.10$}
& \rebuttal{$0.92\pm0.09$} & \rebuttal{$0.99\pm0.07$}
& \rebuttal{$0.87\pm0.10$} & \rebuttal{$1.21\pm0.09$} \\

\midrule
\method: structured $w_g$ (ours)
& $0.72\pm0.06$ & $0.65\pm0.06$
& $0.64\pm0.08$ & $0.59\pm0.07$
& $0.75\pm0.06$ & $0.67\pm0.06$
& $0.89\pm0.09$ & $0.81\pm0.07$
& $0.75\pm0.06$ & $0.67\pm0.06$
& $0.70\pm0.09$ & $0.64\pm0.08$
& $0.75\pm0.06$ & $0.71\pm0.06$
& $0.80\pm0.09$ & $0.70\pm0.08$
& $0.75\pm0.07$ & $0.70\pm0.06$ \\
\method: oracle $\beta_0$ (ours)
& $0.72\pm0.06$ & $0.63\pm0.05$
& $0.65\pm0.08$ & $0.61\pm0.08$
& $0.75\pm0.06$ & $0.65\pm0.06$
& $0.89\pm0.09$ & $0.78\pm0.07$
& $0.75\pm0.06$ & $0.64\pm0.05$
& $0.69\pm0.09$ & $0.62\pm0.08$
& $0.76\pm0.06$ & $0.70\pm0.06$
& $0.80\pm0.09$ & $0.71\pm0.08$
& $0.75\pm0.07$ & $0.68\pm0.06$ \\
\midrule
Plug-in
& $1.12\pm0.09$ & $0.94\pm0.08$
& $0.92\pm0.12$ & $0.76\pm0.09$
& $1.04\pm0.08$ & $0.83\pm0.07$
& $1.08\pm0.10$ & $0.96\pm0.08$
& $1.06\pm0.10$ & $0.87\pm0.08$
& $0.96\pm0.11$ & $0.80\pm0.11$
& $0.98\pm0.09$ & $0.88\pm0.08$
& $1.08\pm0.10$ & $0.86\pm0.08$
& $0.97\pm0.10$ & $0.86\pm0.08$ \\
\method (\textbf{ours})
& $0.72\pm0.06$ & $0.65\pm0.06$
& $0.65\pm0.08$ & $0.59\pm0.07$
& $0.75\pm0.06$ & $0.68\pm0.06$
& $0.89\pm0.09$ & $0.80\pm0.07$
& $0.75\pm0.06$ & $0.68\pm0.06$
& $0.70\pm0.09$ & $0.65\pm0.09$
& $0.75\pm0.06$ & $0.71\pm0.06$
& $0.80\pm0.09$ & $0.68\pm0.07$
& $0.75\pm0.07$ & $0.71\pm0.06$ \\
\quad Rel. improvement
& $\greentext{35.4 \%}$ & $\greentext{30.3 \%}$
& $\greentext{29.7 \%}$ & $\greentext{21.9 \%}$
& $\greentext{28.0 \%}$ & $\greentext{18.7 \%}$
& $\greentext{17.3 \%}$ & $\greentext{17.1 \%}$
& $\greentext{29.6 \%}$ & $\greentext{21.3 \%}$
& $\greentext{27.1 \%}$ & $\greentext{18.9 \%}$
& $\greentext{23.6 \%}$ & $\greentext{19.4 \%}$
& $\greentext{26.2 \%}$ & $\greentext{21.1 \%}$
& $\greentext{22.8 \%}$ & $\greentext{18.1 \%}$ \\
\bottomrule
\end{tabular}%
}
\vspace{-0.3cm}
\caption{\textbf{Advantages of debiasing.} RMSE $\pm$ SE ($\times 100$) across different functionals (area under the concentration curve [AUC] and time-above-threshold [TAT]) for different overlap levels $\rho$. $\Rightarrow$ \emph{\method outperforms the plug-in estimator.}}
\label{tab:rmse_auc_smooth_tat}
\end{table*}

$\bullet$~\textbf{\underline{Functional and sampling irregularity:}} 
We analyze $18$ settings of different complexity, where we vary (i)~the \textit{target functional}: area under the concentration curve (\textbf{AUC}) and time-above-threshold (\textbf{TAT}) as a more complex functional; (ii)~the \textit{sampling irregularity} of the output trajectory for both estimation tasks. For the latter, we vary $\rho$, where $\rho=0$ corresponds to uniform sampling, and larger values of $\rho$ increasingly concentrate observations near the subject-specific peak region.  We further report a structured parameterization of \method with known $w_g$ to show the additional benefits when a closed-form version of the Riesz representer is available. Finally, we report an oracle \method where we use the oracle $\beta_0$ to understand the upper-bound on the performance.

\begin{wrapfigure}{r}{0.45\textwidth}
  \centering
\vspace{-0.3cm}  
  {\includegraphics[width=0.45\textwidth]{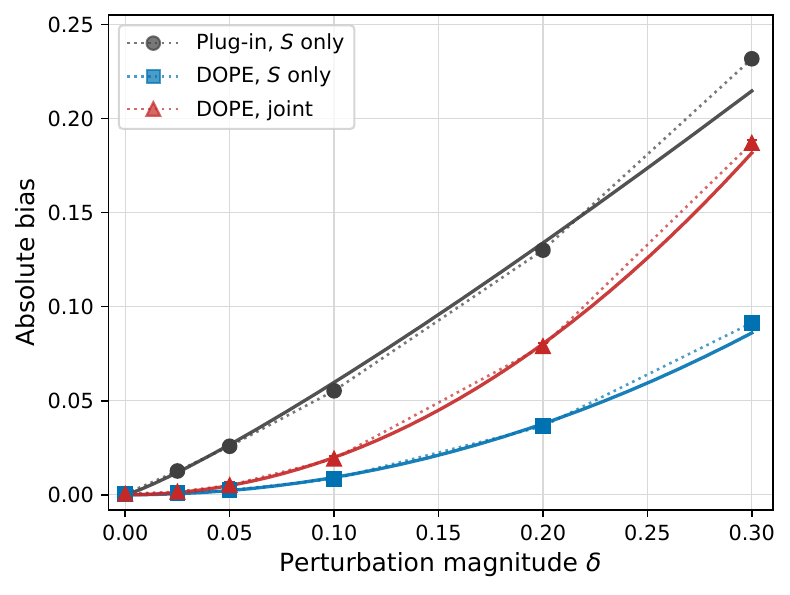}}
\vspace{-0.5cm}
     \caption{\textbf{Oracle-centered nuisance perturbations:} Reported is absolute bias for $\theta$ as we perturb $S_0$ alone or jointly perturb $(S_0,\beta_0)$. $\Rightarrow$ \emph{The plug-in bias is first order, while \method exhibits only second-order bias.}}\label{fig:nuisance_corruption}
     \vspace{-0.3cm}
\end{wrapfigure}

\underline{Results:}  \method outperforms the plug-in baseline by a clear margin (see Table~\ref{tab:rmse_auc_smooth_tat}). The \method variant with structured $w_g$ yields further improvements, which shows that, for known $w_g$, the structured parametrization simplifies the estimation task notably. \rebuttal{Further, \method achieves lower RMSE than the observation-only baseline, although the baseline has oracle access to observation-design weights. At $\rho=1$, \method reduces RMSE relative to observation-only estimation by $14.4\%$ and
$41.8\%$, respectively. This demonstrates the benefit of combining learned operator predictions with our Riesz-weighted correction, beyond simply using the observed measurements.}

$\bullet$~\textbf{\underline{Robustness to nuisance errors:}} To directly verify the local robustness implied by Neyman orthogonality, we perturb the nuisance functions around their oracle values. Specifically, we perturb the solution operator as $S_\delta=S_0+\delta h_S$, either alone or jointly with $\beta_\delta=\beta_0+\delta h_\beta$, for increasing magnitude $\delta$. We then report the resulting absolute bias of the plug-in estimator and \method for the TAT functional.

\underline{Results:} Figure~\ref{fig:nuisance_corruption} shows a clear difference in the local bias order. The plug-in estimator grows approximately linearly with the perturbation magnitude, whereas \method exhibits has only quadratic local behavior under both $S_0$-only and joint $(S_0,\beta_0)$ perturbations. This is confirmed by the corresponding log--log slopes: $1.16\pm0.03$ for the plug-in estimator, versus $2.03\pm0.25$ ($S_0$-only) and $2.02\pm0.12$ (joint $S_0,\beta_0$) for \method, respectively. These results directly confirm our theory \rebuttal{and isolate the effect of Neyman orthogonality: 
Under identical perturbations of the solution operator, our Riesz-weighted correction cancels the first-order bias. Thus, the observed improvement cannot be explained merely by access to additional measurements, but follows from the orthogonal construction of \method.}


\rebuttal{$\bullet$~\textbf{\underline{Input-specific functional estimation:}}}
\rebuttal{
We additionally evaluate \method for predicting input-specific functionals $\theta_0(a)=g(S_0(a))$. Following Theorem~\ref{thm:input_specific}, we learn the conditional residual correction on independent second-stage subjects and add it to the original plug-in prediction. We compare against the direct operator plug-in using the same learned solution operator and test subjects. Here, the plug-in baseline directly evaluates $g(\hat S(a))$, whereas \method adds the conditional correction that eliminates first-order bias from nuisance estimation errors.
}

\begin{table*}[h]
\vspace{-0.2cm}
\centering
\setlength{\tabcolsep}{2.5pt}
\small
\resizebox{\textwidth}{!}{%
\begin{tabular}{@{}lcccccccccccccccccc@{}}
\toprule
& \multicolumn{2}{c}{$\rho=0$}
& \multicolumn{2}{c}{$\rho=0.125$}
& \multicolumn{2}{c}{$\rho=0.25$}
& \multicolumn{2}{c}{$\rho=0.375$}
& \multicolumn{2}{c}{$\rho=0.5$}
& \multicolumn{2}{c}{$\rho=0.625$}
& \multicolumn{2}{c}{$\rho=0.75$}
& \multicolumn{2}{c}{$\rho=0.875$}
& \multicolumn{2}{c}{$\rho=1$} \\
\cmidrule(lr){2-3}\cmidrule(lr){4-5}\cmidrule(lr){6-7}\cmidrule(lr){8-9}
\cmidrule(lr){10-11}\cmidrule(lr){12-13}\cmidrule(lr){14-15}\cmidrule(lr){16-17}\cmidrule(lr){18-19}

& AUC & TAT
& AUC & TAT
& AUC & TAT
& AUC & TAT
& AUC & TAT
& AUC & TAT
& AUC & TAT
& AUC & TAT
& AUC & TAT \\

\midrule
\rebuttal{Plug-in}
& \rebuttal{$2.34\pm0.06$} & \rebuttal{$1.85\pm0.04$}
& \rebuttal{$2.45\pm0.06$} & \rebuttal{$1.92\pm0.04$}
& \rebuttal{$2.43\pm0.06$} & \rebuttal{$1.91\pm0.04$}
& \rebuttal{$2.41\pm0.05$} & \rebuttal{$1.88\pm0.04$}
& \rebuttal{$2.40\pm0.05$} & \rebuttal{$1.89\pm0.03$}
& \rebuttal{$2.48\pm0.05$} & \rebuttal{$1.95\pm0.05$}
& \rebuttal{$2.47\pm0.07$} & \rebuttal{$1.97\pm0.06$}
& \rebuttal{$2.56\pm0.06$} & \rebuttal{$2.00\pm0.05$}
& \rebuttal{$2.49\pm0.05$} & \rebuttal{$1.93\pm0.04$} \\

\rebuttal{\method (\textbf{ours})}
& \rebuttal{$\mathbf{2.14\pm0.05}$} & \rebuttal{$\mathbf{1.74\pm0.04}$}
& \rebuttal{$\mathbf{2.21\pm0.05}$} & \rebuttal{$\mathbf{1.78\pm0.03}$}
& \rebuttal{$\mathbf{2.20\pm0.05}$} & \rebuttal{$\mathbf{1.76\pm0.03}$}
& \rebuttal{$\mathbf{2.17\pm0.05}$} & \rebuttal{$\mathbf{1.74\pm0.03}$}
& \rebuttal{$\mathbf{2.18\pm0.05}$} & \rebuttal{$\mathbf{1.78\pm0.03}$}
& \rebuttal{$\mathbf{2.24\pm0.05}$} & \rebuttal{$\mathbf{1.80\pm0.04}$}
& \rebuttal{$\mathbf{2.25\pm0.06}$} & \rebuttal{$\mathbf{1.86\pm0.05}$}
& \rebuttal{$\mathbf{2.31\pm0.06}$} & \rebuttal{$\mathbf{1.87\pm0.04}$}
& \rebuttal{$\mathbf{2.29\pm0.05}$} & \rebuttal{$\mathbf{1.83\pm0.04}$} \\

\rebuttal{\quad Rel. improvement}
& \rebuttal{$\greentext{8.6\%}$} & \rebuttal{$\greentext{6.3\%}$}
& \rebuttal{$\greentext{9.7\%}$} & \rebuttal{$\greentext{7.4\%}$}
& \rebuttal{$\greentext{9.5\%}$} & \rebuttal{$\greentext{8.0\%}$}
& \rebuttal{$\greentext{10.0\%}$} & \rebuttal{$\greentext{7.1\%}$}
& \rebuttal{$\greentext{9.2\%}$} & \rebuttal{$\greentext{6.1\%}$}
& \rebuttal{$\greentext{9.6\%}$} & \rebuttal{$\greentext{7.8\%}$}
& \rebuttal{$\greentext{8.9\%}$} & \rebuttal{$\greentext{5.5\%}$}
& \rebuttal{$\greentext{9.9\%}$} & \rebuttal{$\greentext{6.3\%}$}
& \rebuttal{$\greentext{8.2\%}$} & \rebuttal{$\greentext{5.4\%}$} \\
\bottomrule
\end{tabular}%
}
\vspace{-0.3cm}
\caption{\rebuttal{
\textbf{Input-specific functional estimation.}
RMSE $\pm$ SE ($\times 100$) for input-specific AUC and TAT prediction across sampling regimes $\rho$.
$\Rightarrow$ \emph{\method consistently improves input-specific functional prediction across all sampling regimes and functionals.}
}}
\label{tab:input_specific}
\vspace{-0.2cm}
\end{table*}

\rebuttal{
\underline{Results:}
Table~\ref{tab:input_specific} shows that \method consistently improves input-specific prediction accuracy across all $18$ functional--sampling configurations. Relative to the direct plug-in estimator that computes $g(\hat S (a))$ directly, \method clearly reduces
strong RMSE by and outperforms the na\"ive estimator across \textbf{all} configurations. These improvements demonstrate that our debiasing framework extends beyond population-average estimation. Learning the residual correction improves the accuracy of input-specific functional estimation, and preserves the debiasing property of \method.}

\circledblue{2}~\textbf{\underline{Darcy flow dataset:}} We consider a two-dimensional Darcy-flow benchmark dataset, which is standard for evaluating neural operator outputs \citep{Kovachki.2021NO, Kovachki.2023NO}. Details are in Supplement~\ref{sec:darcy_dgp}. We evaluate a smooth excess-above-threshold functional of the Darcy solution field, which measures the spatial extent and intensity of regions where the pressure or state field exceeds a fixed level $c$. For this, we vary the smoothness parameter $\kappa$ in the functional.

\begin{wraptable}{r}{0.6\textwidth} 
\vspace{-0.5cm}              
\setlength{\intextsep}{0pt}          
\setlength{\columnsep}{1em}          
\centering
\begin{adjustbox}{width=\linewidth}
\centering
\small
\begin{tabular}{lcccccc}
\toprule
& $\kappa=0.0$ & $\kappa=0.2$ & $\kappa=0.4$ & $\kappa=0.6$ & $\kappa=0.8$ & $\kappa=1.0$ \\
\midrule

Plug-in
& $0.59 \pm 0.06$
& $0.49 \pm 0.06$
& $0.41 \pm 0.05$
& $0.35 \pm 0.04$
& $0.30 \pm 0.04$
& $0.26 \pm 0.04$ \\
\method
& $0.45 \pm 0.06$
& $0.39 \pm 0.06$
& $0.33 \pm 0.05$
& $0.29 \pm 0.05$
& $0.26 \pm 0.05$
& $0.23 \pm 0.05$ \\
\quad Rel. improvement
& \greentext{$23.5\%$}
& \greentext{$21.5\%$}
& \greentext{$19.3\%$}
& \greentext{$16.8\%$}
& \greentext{$14.2\%$}
& \greentext{$11.7\%$}\\
\bottomrule
\end{tabular}%
\end{adjustbox}
\vspace{-0.3cm}
\caption{\textbf{Darcy flow results.} RMSE $\pm$ SE ($\times 10^4$) for plug-in and \method across different values of functional smoothness $\kappa$. $\Rightarrow$ \emph{\method improves over the plug-in estimator across all smoothness levels.}}
\label{tab:darcy_kappa_rmse}
\vspace{-0.5cm}
\end{wraptable}

\underline{Results:} Table~\ref{tab:darcy_kappa_rmse} reports the RMSE as we vary the smoothness parameter $\kappa$ in the target functional. Thereby, we make estimating the target functional more challenging. Due to debiasing, \method outperforms the plug-in estimator for all smoothness levels $\kappa$.


\circledblue{3}~\textbf{\underline{Real-world data:}} We report results for the ETTh1 electricity-transformer dataset \citep{haoyietal-informer-2021} in Supplement~\ref{sec:rwd}. Our results confirm that \method is applicable to real-world data and consistently outperforms plug-in approaches for different functionals.

\rebuttal{\circledblue{4}~\textbf{\underline{Additional experiments:}} \textbf{(i)}~We provide results for \textbf{input-specific functional estimation on Darcy flow} in Supplement~\ref{sec:darcy_input_specific}. \textbf{(ii)}~We evaluate our \textbf{CLT-based confidence intervals}, including robustness under sampling irregularity, in Supplement~\ref{sec:confidence_intervals}. \textbf{(iii)}~We compare \method against \textbf{direct solvers} in Supplement~\ref{sec:direct_solver}. \textbf{(iv)}~Empirical results for \textbf{regularizing the Riesz-correction} term are in Supplement~\ref{sec:riesz_regularization_sensitivity}.} \textbf{(v)}~We provide results with \textbf{DeepONet} \citep{Lu.2021NO} in Supplement~\ref{sec:deeponet}. \textbf{(vi)}~We report performance for a \textbf{PPI-style variant} of \method in Supplement~\ref{sec:ppi}.

\rebuttal{$\bullet$~\textbf{\underline{Limitations:}} \method does not require knowledge of the governing PDE, yet its guarantees rely on consistent nuisance estimation and need not hold under severe structural misspecification that prevents $\hat S$ from approximating $S_0$, as is the case for any consistent neural operator estimator.}

$\bullet$~\textbf{\underline{Conclusion:}} 
In this work, we propose \method, a debiased estimator for averaged statistical functionals of neural operator outputs under partial and input-dependent observations. Our approach extends automated DML from function-valued to operator-valued nuisance estimation, and thereby enables bias correction and valid inference for nonlinear functionals of learned solution operators. 

\clearpage



\subsection*{Reproducibility statement}

We provide complete proofs for the theoretical results and detailed descriptions of the experimental setup, including the data-generating processes, model configurations, training procedures, evaluation protocols, and random seeds. Code for reproducing the experiments will be made available with the paper. Synthetic datasets are generated from fully specified procedures, and the reported results can be reproduced using the provided implementation.

\clearpage

\bibliography{bibliography}
\bibliographystyle{iclr2027_conference}
\clearpage
\appendix


\section{Extended related work}\label{sec:extended_related_work}
Neural operators provide a framework for learning mappings between function spaces \citep{Kovachki.2023NO}. This allows one to learn solution operators for PDEs that map input functions (e.g., coefficients or boundary conditions) to solution trajectories. Hence, neural operators are widely used to solve scientific problems in scientific computing \citep{Azizzadenesheli.2024NO, Cao.2024NO, Faroughi.2024NO}, climate science \citep{Kurth.2023NO, Peng.2024NO, Wen.2023NO}, and engineering \citep{Faroughi.2024NO, Kobayashi.2024NO, Shukla.2024NO, Wang.2025NO}. They enjoy several favorable properties, such as discretization invariance, the ability to evaluate functions at arbitrary locations, and universal approximation properties \citep{Kovachki.2021NO, Kovachki.2023NO}. 

Several architectures for neural networks have been proposed \cite[e.g.,][]{Bonev.2023, Cao.2024NO, Fanaskov.2023NO, Lee.2026NO, Lu.2021NO, Rahman2023NO, Raonic.2023NO, Zhao.2025NO}. Two state-of-the-art examples are the Fourier neural operator \citep{Li.2021NO} and DeepONet \citep{Lu.2021NO}, which we use both to demonstrate our framework. Importantly, our \method framework is agnostic to the underlying architecture and can be instantiated with \emph{arbitrary} neural operator backbones.

A common downstream use of neural operators is inference on low-dimensional functionals of the predicted trajectory rather than recovery of the full field itself (see Figure~\ref{fig:functional_examples}). $\Rightarrow$ \emph{However, the only generally available strategy is na{\"i}ve plug-in evaluation of the functional on the predicted trajectory, which suffers from \textbf{plug-in bias} under operator approximation error.}

\rebuttal{Related work by \citet{Huang.2025} studies end-to-end parameter-to-observable learning, and directly predicts scalar
functionals from input functions. However, this requires ground-truth functional labels $g(S_0(A_i))$, whereas our setting
provides only sparse, noisy point observations of the solution field. Thus, these approaches address a complementary supervision regime, while \method enables debiased functional estimation from partial observations without requiring full solution fields or ground-truth functional labels.}

\clearpage

\section{Proofs}\label{sec:proofs}

\subsection{Plug-in bias}\label{appendix:proofs_plugin}

\begin{proposition}[First-order bias of plug-in estimators]\label{prop:pluginbias_appendix}
Let $\Delta(\cdot)=\hat S(\cdot)-S_0(\cdot)$.
Then
\begin{align}
\big| \mathbb{E}[g(\hat S(A))]-\theta \big|
\;\lesssim\;
\big| \mathbb{E}\!\left[\Diff g_{S_0(A)}(\Delta(A))\right] \big|
\;+\;
\|\Delta\|_{L^2(\mathbb P)}^2,
\end{align}
i.e., the neural operator error propagates linearly through the Fr\'echet differential and induces first-order bias.
\end{proposition}

\begin{proof}
We apply the second-order expansion in \Eqref{ass:g_smooth} with $h=\Delta(A)$:
\begin{align}
g(\hat S(A))
=
g(S_0(A)+\Delta(A))
=
g(S_0(A))+\Diff g_{S_0(A)}(\Delta(A))+R(S_0(A),\Delta(A)).
\end{align}
Taking expectations and using $\theta=\mathbb E[g(S_0(A))]$ yields
\begin{align}
\mathbb{E}[g(\hat S(A))]-\theta
=
\mathbb{E}\!\left[\Diff g_{S_0(A)}(\Delta(A))\right]
+
\mathbb{E}[R(S_0(A),\Delta(A))].
\label{eq:plugin_decomp}
\end{align}

\medskip
In the Hilbert space $\mathcal U=L^2(\Omega,\mu)$, the remainder bound in \Eqref{ass:g_smooth} gives
\begin{align}
|R(S_0(A),\Delta(A))|\lesssim \|\Delta(A)\|_{L^2(\mu)}^2,
\end{align}
which depends on $A$ only.
In order to express this bound in the sampling geometry, we use Assumption~\ref{assumption:design}.
That is, since $\mu\ll \mathbb P(\cdot\mid A)$ with Radon--Nikodym derivative
$\xi_0(A)(\cdot)=\frac{\diff\mu}{\diff\mathbb P(\cdot\mid A)}\le C$, we have, for each realization of $A$,
\begin{align}
\|\Delta(A)\|_{L^2(\mu)}^2
&=
\int_\Omega \Delta(A)(x)^2\,\diff\mu(x)
=
\int_\Omega \Delta(A)(x)^2\,\xi_0(A)(x)\,\diff\mathbb P(x\mid A)\\
&\le
C\int_\Omega \Delta(A)(x)^2\,\diff\mathbb P(x\mid A) 
= C
\mathbb E[\Delta(A)(X)^2\mid A].
\label{eq:cond_int_as_exp}
\end{align}
Taking expectation over $A$ and using the definition of the joint law
$\mathbb P(\diff a,\diff x)=\mathbb P(\diff a)\,\mathbb P(\diff x\mid a)$ yields
\begin{align}
\mathbb E\|\Delta(A)\|_{L^2(\mu)}^2
\le
C\,\mathbb E[\Delta(A)(X)^2]
=
C\,\|\Delta\|_{L^2(\mathbb P)}^2,
\qquad
\|\Delta\|_{L^2(\mathbb P)}^2\coloneqq\mathbb E[\Delta(A)(X)^2].
\label{eq:mu_to_jointQ}
\end{align}
Therefore, applying Jensen's inequality and \Eqref{eq:mu_to_jointQ}, yields
\begin{align}
\big|\mathbb{E}[R(S_0(A),\Delta(A))]\big|
\le
\mathbb E\big[|R(S_0(A),\Delta(A))|\big]
\lesssim
\mathbb E\|\Delta(A)\|_{L^2(\mu)}^2
\lesssim
\|\Delta\|_{L^2(\mathbb P)}^2.
\label{eq:plugin_remainder_bound}
\end{align}

Finally, combining \Eqref{eq:plugin_decomp} and \Eqref{eq:plugin_remainder_bound} and applying the triangle inequality gives
\begin{align}
\big| \mathbb{E}[g(\hat S(A))]-\theta \big|
\;\lesssim\;
\big| \mathbb{E}\!\left[\Diff g_{S_0(A)}(\Delta(A))\right] \big|
+
\|\Delta\|_{L^2(\mathbb P)}^2.
\end{align}

Unless the linear term vanishes (which it need not for nonlinear $g$), $\sqrt{n}$-inference requires this bias to be $o(n^{-1/2})$, i.e.,
\begin{align}
\mathbb{E}\!\left[\Diff g_{S_0(A)}(\Delta(A))\right]=o(n^{-1/2}),
\end{align}
which is a first-order (linear) accuracy requirement on the operator error $\Delta(A)$.

\end{proof}

\clearpage


\clearpage

\subsection{One-step estimator}\label{appendix:proofs_onestep}

\begin{lemma}[Debiasing identity for the score]
\label{lem:score_derivation}
Let $\Delta$ be any measurable perturbation such that $\Delta(A,\cdot)\in\mathcal U$
almost surely and
\begin{align}
\mathbb E[\Delta(A)(X)^2] < \infty.
\end{align}
If $\beta_0\in L^2(\mathbb P)$ satisfies
\begin{align}
\mathbb E[\beta_0(A)(X)h(X)\mid A]
=
\Diff g_{S_0(A)}(h),
\qquad
\forall h\in\mathcal{ H}_A,
\label{eq:appendix_riesz_score}
\end{align}
then
\begin{align}
\mathbb E\!\left[
\frac{1}{K}\sum_{k=1}^K \beta_0(A)(X_k)\Delta(A)(X_k)
\right]
=
\mathbb E\!\left[\Diff g_{S_0(A)}(\Delta(A))\right].
\label{eq:debiasing_identity_appendix}
\end{align}
Moreover, if $\hat S$ is any estimator independent of the evaluation fold
(for instance, by cross-fitting), and $\Delta(A)=\hat S(A)-S_0(A)$, then
\begin{align}
\mathbb E\!\left[
\frac{1}{K}\sum_{k=1}^K
\beta_0(A)(X_k)\big(Y_k-\hat S(A)(X_k)\big)
\,\middle|\,
\hat S
\right]
=
-
\mathbb E\!\left[
\Diff g_{S_0(A)}(\Delta(A))
\,\middle|\,
\hat S
\right].
\label{eq:residual_debiasing_identity_appendix}
\end{align}
Hence, the weighted residual term in \Eqref{eq:score} cancels the linear plug-in
bias.
\end{lemma}

\begin{proof}
For the first claim, conditional on $(A,K)$ the pairs
$\{(X_k,Y_k)\}_{k=1}^K$ are i.i.d., so applying exchangeability  yields
\begin{align}
\mathbb E\!\left[
\frac{1}{K}\sum_{k=1}^K
\beta_0(A)(X_k)\Delta(A)(X_k)
\,\middle|\,
A,K
\right]
=
\mathbb E[\beta_0(A)(X)\Delta(A)(X)\mid A].
\end{align}
Taking expectations and applying the Riesz identity from \Eqref{eq:appendix_riesz_score} with $h=\Delta(A)$ yields
\begin{align}
\mathbb E\!\left[
\frac{1}{K}\sum_{k=1}^K \beta_0(A)(X_k)\Delta(A)(X_k)
\right]
&=
\mathbb E\!\left[
\mathbb E[\beta_0(A)(X)\Delta(A)(X)\mid A]
\right] \\
&=
\mathbb E\!\left[
\Diff g_{S_0(A)}(\Delta(A))
\right],
\end{align}
which proves \Eqref{eq:debiasing_identity_appendix}.

For the second claim, write
\begin{align}
Y_k-\hat S(A)(X_k)
=
\big(Y_k-S_0(A)(X_k)\big)-\Delta(A)(X_k),
\qquad
\Delta(A)=\hat S(A)-S_0(A).
\end{align}
Conditioning on $\hat S$ treats $\Delta$ as fixed, so
\begin{align}
&\mathbb E\!\left[
\frac{1}{K}\sum_{k=1}^K
\beta_0(A)(X_k)\big(Y_k-\hat S(A)(X_k)\big)
\,\middle|\,
\hat S
\right]
\nonumber\\
&=
\mathbb E\!\left[
\frac{1}{K}\sum_{k=1}^K
\beta_0(A)(X_k)\big(Y_k-S_0(A)(X_k)\big)
\,\middle|\,
\hat S
\right]
-
\mathbb E\!\left[
\frac{1}{K}\sum_{k=1}^K
\beta_0(A)(X_k)\Delta(A)(X_k)
\,\middle|\,
\hat S
\right].
\label{eq:residual_split_appendix}
\end{align}
For the first term, iterated expectations together with
$\mathbb E[Y_k-S_0(A)(X_k)\mid A,X_k]=0$ give
\begin{align}
\mathbb E\!\left[
\frac{1}{K}\sum_{k=1}^K
\beta_0(A)(X_k)\big(Y_k-S_0(A)(X_k)\big)
\,\middle|\,
\hat S
\right]
=
0.
\end{align}
Applying \Eqref{eq:debiasing_identity_appendix} conditionally on $\hat S$ to the
second term in \Eqref{eq:residual_split_appendix} yields
\begin{align}
\mathbb E\!\left[
\frac{1}{K}\sum_{k=1}^K
\beta_0(A)(X_k)\big(Y_k-\hat S(A)(X_k)\big)
\,\middle|\,
\hat S
\right]
=
-
\mathbb E\!\left[
\Diff g_{S_0(A)}(\Delta(A))
\,\middle|\,
\hat S
\right],
\end{align}
which proves \Eqref{eq:residual_debiasing_identity_appendix}.
\end{proof}

\clearpage

\begin{proposition}[Neyman orthogonality of the debiasing score]
\label{prop:orthogonality}
Define the score map
\begin{align}
m(O;\theta,S,\beta)
\coloneqq
g(S(A))-\theta
+
\frac{1}{K}\sum_{k=1}^K
\beta(A)(X_k)\big(Y_k-S(A)(X_k)\big).
\label{eq:score_map_appendix}
\end{align}
Let $(\theta_0,S_0,\beta_0)$ denote the truth, where
\begin{align}
\theta_0=\mathbb E[g(S_0(A))]
\end{align}
and $\beta_0\in L^2(\mathbb P)$ satisfies
\begin{align}
\mathbb E[\beta_0(A)(X)h(X)\mid A]
=
\Diff g_{S_0(A)}(h)
\qquad
\forall h\in\mathcal{H}_A.
\label{eq:orthogonality_riesz}
\end{align}
Then, the score is Neyman-orthogonal at $(\theta_0,S_0,\beta_0)$ in the sense that, for any admissible perturbations
$h$ and $b$ such that $h(A,\cdot)\in\mathcal H_A$ almost surely, $b\in L^2(\mathbb P)$, and $S_0+t h$ remains in the neighborhood from \Eqref{ass:g_smooth} for sufficiently small $t$, i.e.,
\begin{align}
\left.
\frac{\diff}{\diff t}
\mathbb E\!\left[
m(O;\theta_0,S_0+t h,\beta_0)
\right]
\right|_{t=0}
&=0,
\label{eq:orthogonality_S}
\\
\left.
\frac{\diff}{\diff t}
\mathbb E\!\left[
m(O;\theta_0,S_0,\beta_0+t b)
\right]
\right|_{t=0}
&=0.
\label{eq:orthogonality_beta}
\end{align}
\end{proposition}

\begin{proof}
We first consider perturbations of $S$. By definition of the score map,
\begin{align}
\mathbb E[m(O;\theta_0,S_0+t h,\beta_0)]
&=
\mathbb E[g(S_0(A)+t h(A))]
-\theta_0
\\
&\quad+
\mathbb E\!\left[
\frac{1}{K}\sum_{k=1}^K
\beta_0(A)(X_k)\big(Y_k-S_0(A)(X_k)-t h(A)(X_k)\big)
\right].
\end{align}
Differentiating at $t=0$ using Fr\'echet differentiability of $g$ gives
\begin{align}
\left.
\frac{\diff}{\diff t}
\mathbb E[m(O;\theta_0,S_0+t h,\beta_0)]
\right|_{t=0}
&=
\mathbb E[\Diff g_{S_0(A)}(h(A))]
-
\mathbb E\!\left[
\frac{1}{K}\sum_{k=1}^K
\beta_0(A)(X_k)h(A)(X_k)
\right].
\label{eq:orth_S_step1}
\end{align}
Conditional on $(A,K)$, the pairs $\{(X_k,Y_k)\}_{k=1}^K$ are i.i.d., so by exchangeability, we arrive at
\begin{align}
\mathbb E\!\left[
\frac{1}{K}\sum_{k=1}^K
\beta_0(A)(X_k)h(A)(X_k)
\,\middle|\,
A,K
\right]
=
\mathbb E[\beta_0(A)(X)h(A)(X)\mid A].
\end{align}
Hence, we obtain
\begin{align}
\mathbb E\!\left[
\frac{1}{K}\sum_{k=1}^K
\beta_0(A)(X_k)h(A)(X_k)
\right]
&=
\mathbb E\!\left[
\mathbb E[\beta_0(A)(X)h(A)(X)\mid A]
\right]
\\
&=
\mathbb E[\Diff g_{S_0(A)}(h(A))],
\end{align}
where the last step uses the Riesz identity from \Eqref{eq:orthogonality_riesz} with
$h=h(A,\cdot)$.
Substituting into \Eqref{eq:orth_S_step1} proves \Eqref{eq:orthogonality_S}.

Next, consider perturbations of $\beta$. We have
\begin{align}
& \mathbb E[m(O;\theta_0,S_0,\beta_0+t b)]
=&
\mathbb E[m(O;\theta_0,S_0,\beta_0)]
+
t\,
\mathbb E\!\left[
\frac{1}{K}\sum_{k=1}^K
b(A,X_k)\big(Y_k-S_0(A)(X_k)\big)
\right].
\end{align}
Therefore,
\begin{align}
\left.
\frac{\diff}{\diff t}
\mathbb E[m(O;\theta_0,S_0,\beta_0+t b)]
\right|_{t=0}
&=
\mathbb E\!\left[
\frac{1}{K}\sum_{k=1}^K
b(A,X_k)\big(Y_k-S_0(A)(X_k)\big)
\right].
\end{align}
Using iterated expectations and the regression property
$\mathbb E[Y_k-S_0(A)(X_k)\mid A,X_k]=0$, we obtain
\begin{align}
\mathbb E\!\left[
b(A,X_k)\big(Y_k-S_0(A)(X_k)\big)
\right]
&=
\mathbb E\!\left[
b(A,X_k)\,
\mathbb E[Y_k-S_0(A)(X_k)\mid A,X_k]
\right]
=0.
\end{align}
Averaging over $k$ yields \Eqref{eq:orthogonality_beta}.
Thus, the score is Neyman-orthogonal at $(\theta_0,S_0,\beta_0)$.
\end{proof}

\clearpage

\clearpage
\begin{theorem}[Second-order remainder]\label{thm:remainder_appendix}
Let $\hat\theta$ be the one-step estimator in \Eqref{eq:onestep}. If $\hat S$ and $\hat\beta$ are cross-fitted, then
\begin{align}
\big|\mathbb{E}[\hat\theta]-\theta\big|
\;\lesssim\;
\|\hat S-S_0\|_{L^2(\mathbb P)}^2
+
\|\hat\beta-\beta_0\|_{L^2(\mathbb P)}^2.
\end{align}
In particular, there is \textbf{no first-order bias} in either nuisance component.
\end{theorem}

\begin{proof}
Recall that
\begin{align}
\hat\theta=\frac1n\sum_{i=1}^n \tilde\psi_i,
\qquad
\tilde\psi
=
g(\hat S(A))
+
\frac{1}{K}\sum_{k=1}^{K}
\hat\beta(A)(X_k)\big(Y_k-\hat S(A)(X_k)\big),
\end{align}
where $(A,K,\{(X_k,Y_k)\}_{k=1}^K)$ denotes a generic observation in the evaluation fold.

\medskip
\noindent\textbf{Step 0: Cross-fitting.}
By cross-fitting, $(\hat S,\hat\beta)$ are measurable with respect to the training sample
$\mathcal D^{\mathrm{train}}$, and independent of the evaluation fold. Thus,
\begin{align}
\mathbb E[\hat\theta-\theta]
=
\mathbb E\!\left[\mathbb E[\tilde\psi-\theta\mid \mathcal D^{\mathrm{train}}]\right],
\end{align}
and it therefore suffices to bound the conditional bias $\mathbb E[\tilde\psi-\theta\mid \mathcal D^{\mathrm{train}}]$.

\medskip
\noindent\textbf{Step 1: Expansion of functional.}
Let $\Delta(A)=\hat S(A)-S_0(A)$. By \Eqref{ass:g_smooth}, we have that
\begin{align}
g(\hat S(A))
=
g(S_0(A))
+
\Diff g_{S_0(A)}(\Delta(A))
+
R(S_0(A),\Delta(A)),
\qquad
|R(S_0(A),\Delta(A))|\lesssim \|\Delta(A)\|_{\mathcal U}^2.
\label{eq:g_taylor_Q}
\end{align}
Hence, taking conditional expectation given $\mathcal D^{\mathrm{train}}$ yields
\begin{align}
\mathbb E[g(\hat S(A))-g(S_0(A))\mid \mathcal D^{\mathrm{train}}]
=
\mathbb E[\Diff g_{S_0(A)}(\Delta(A))\mid \mathcal D^{\mathrm{train}}]
+
O\!\left(\mathbb E[\|\Delta(A)\|_{\mathcal U}^2\mid \mathcal D^{\mathrm{train}}]\right).
\label{eq:term1_Q}
\end{align}

\medskip
\noindent\textbf{Step 2: Expansion of the residual correction.}
Next, we decompose
\begin{align}
Y_k-\hat S(A)(X_k)=(Y_k-S_0(A)(X_k))-\Delta(A)(X_k)=:\varepsilon_k-\Delta(A)(X_k).
\end{align}
Since $\mathbb E[\varepsilon_k\mid A,X_k]=0$ and $\hat\beta$ is measurable w.r.t.\ $\mathcal D^{\mathrm{train}}$,
iterated expectations give
\begin{align}
\mathbb E[\hat\beta(A)(X_k)\varepsilon_k\mid \mathcal D^{\mathrm{train}}]=0
\end{align}
Therefore, we have that
\begin{align}
\mathbb E\!\left[\frac1K\sum_{k=1}^K \hat\beta(A)(X_k)\big(Y_k-\hat S(A)(X_k)\big)\,\middle|\,\mathcal D^{\mathrm{train}}\right]
=
-\mathbb E[\hat\beta(A)(X)\Delta(A)(X)\mid \mathcal D^{\mathrm{train}}],
\label{eq:term2_Q}
\end{align}
where $X\mid A\sim \mathbb P(\cdot\mid A)$ denotes a generic draw from the sampling design.

With $\hat\beta=\beta_0+(\hat\beta-\beta_0)$, we have
\begin{align}
&-\mathbb E[\hat\beta(A)(X)\Delta(A)(X)\mid \mathcal D^{\mathrm{train}}]\\
=&
-\mathbb E[\beta_0(A)(X)\Delta(A)(X)\mid \mathcal D^{\mathrm{train}}]
-\mathbb E[(\hat\beta-\beta_0)(A)(X)\Delta(A)(X)\mid \mathcal D^{\mathrm{train}}].
\label{eq:term2_decomp_Q}
\end{align}

\medskip
\noindent\textbf{Step 3: First-order cancellation.}
By the Riesz identity in \Eqref{eq:riesz}, we obtain
\begin{align}
\mathbb E[\beta_0(A)(X)\Delta(A)(X)\mid A]=\Diff g_{S_0(A)}(\Delta(A)).
\end{align}
Hence, we take expectations conditionally on $\mathcal D^{\mathrm{train}}$, which yields
\begin{align}
\mathbb E[\beta_0(A)(X)\Delta(A)(X)\mid \mathcal D^{\mathrm{train}}]
=
\mathbb E[\Diff g_{S_0(A)}(\Delta(A))\mid \mathcal D^{\mathrm{train}}].
\label{eq:riesz_cancel_Q}
\end{align}

Substituting \Eqref{eq:riesz_cancel_Q} into \Eqref{eq:term2_decomp_Q} and \Eqref{eq:term2_Q}, we obtain
\begin{align}
\mathbb E\!\left[\frac1K\sum_{k=1}^K \hat\beta(A,X_k)\big(Y_k-\hat S(A)(X_k)\big)\,\middle|\,\mathcal D^{\mathrm{train}}\right]
&=
-\mathbb E[\Diff g_{S_0(A)}(\Delta(A))\mid \mathcal D^{\mathrm{train}}]
\nonumber\\
&\quad
-\mathbb E[(\hat\beta-\beta_0)(A,X)\Delta(A)(X)\mid \mathcal D^{\mathrm{train}}].
\label{eq:correction_expanded_Q}
\end{align}
Hence, the $\beta_0$-part of the correction term exactly cancels the first-order term
\begin{align}
\mathbb E[\Diff g_{S_0(A)}(\Delta(A))\mid \mathcal D^{\mathrm{train}}]
\end{align}
from the Taylor expansion in \Eqref{eq:term1_Q}, while the additional cross-term involving
$(\hat\beta-\beta_0)\Delta$ remains and is controlled in the next step.

\medskip
\noindent\textbf{Step 4: Bounding the remaining second-order terms.}
Combining \Eqref{eq:term1_Q} with \Eqref{eq:correction_expanded_Q}, after cancellation of the linear term, yields
\begin{align}
\mathbb E[\tilde\psi-\theta\mid \mathcal D^{\mathrm{train}}]
=
O\!\left(\mathbb E[\|\Delta(A)\|_{\mathcal U}^2\mid \mathcal D^{\mathrm{train}}]\right)
-
\mathbb E[(\hat\beta-\beta_0)(A,X)\Delta(A)(X)\mid \mathcal D^{\mathrm{train}}].
\label{eq:cond_bias_Q}
\end{align}
For the cross-term, conditional Cauchy--Schwarz gives
\begin{align}
\Big|
\mathbb E[(\hat\beta-\beta_0)(A)(X)\Delta(A)(X)\mid \mathcal D^{\mathrm{train}}]
\Big|
\le
\|\hat\beta-\beta_0\|_{L^2(\mathbb P\mid \mathcal D^{\mathrm{train}})}
\cdot
\|\Delta(A)(X)\|_{L^2(\mathbb P\mid \mathcal D^{\mathrm{train}})}.
\label{eq:cs_Q}
\end{align}
Hence, using $2ab\le a^2+b^2$ yields
\begin{align}
\Big|
\mathbb E[(\hat\beta-\beta_0)(A)(X)\Delta(A)(X)\mid \mathcal D^{\mathrm{train}}]
\Big|
\le
\frac12\|\hat\beta-\beta_0\|_{L^2(\mathbb P\mid \mathcal D^{\mathrm{train}})}^2
+
\frac12\,\mathbb E[\Delta(A)(X)^2\mid \mathcal D^{\mathrm{train}}].
\label{eq:ab_Q}
\end{align}

Finally, by Assumption~\ref{assumption:design} (i.e., bounded $\xi_0(A)=\diff\mu/\diff\mathbb P(\cdot\mid A)$) and $\mathcal U=L^2(\Omega,\mu)$,
\begin{align}
&\mathbb E[\|\Delta(A)\|_{\mathcal U}^2\mid \mathcal D^{\mathrm{train}}]
=
\mathbb E\!\left[\int_\Omega \Delta(A)(x)^2\,\diff\mu(x)\,\middle|\,\mathcal D^{\mathrm{train}}\right]\\
\le&
C\,\mathbb E\!\left[\int_\Omega \Delta(A)(x)^2\,\diff\mathbb P(x\mid A)\,\middle|\,\mathcal D^{\mathrm{train}}\right]
=
C\,\mathbb E[\Delta(A)(X)^2\mid \mathcal D^{\mathrm{train}}].
\label{eq:mu_to_Q_remainder}
\end{align}
Substituting \Eqref{eq:ab_Q} and \Eqref{eq:mu_to_Q_remainder} into \Eqref{eq:cond_bias_Q} then gives
\begin{align}
\big|\mathbb E[\tilde\psi-\theta\mid \mathcal D^{\mathrm{train}}]\big|
\;\lesssim\;
\|\hat\beta-\beta_0\|_{L^2(\mathbb P\mid \mathcal D^{\mathrm{train}})}^2
+
\mathbb E[\Delta(A)(X)^2\mid \mathcal D^{\mathrm{train}}].
\end{align}
Finally, taking expectation over $\mathcal D^{\mathrm{train}}$ together with the tower property yields
\begin{align}
\big|\mathbb E[\hat\theta]-\theta\big|
\;\lesssim\;
\|\hat\beta-\beta_0\|_{L^2(\mathbb P)}^2
+
\mathbb E[\Delta(A)(X)^2]
=
\|\hat\beta-\beta_0\|_{L^2(\mathbb P)}^2
+
\|\hat S-S_0\|_{L^2(\mathbb P)}^2,
\end{align}
which concludes the proof.
\end{proof}

\clearpage


\subsection{Input-specific functionals}\label{sec:input_specific}

\rebuttal{\begin{theorem}[Conditional debiasing for input-specific functionals]
Let $\theta_0(a)=g(S_0(a))$ denote the true input-specific functional.
For nuisance operators $S$ and $\beta$, define the
\textbf{conditional correction}
\begin{align}
r_{S,\beta}(a)
\coloneqq
\mathbb E\big[
\frac{1}{K}\sum_{k=1}^K
\beta(A)(X_k)
\big(Y_k-S(A)(X_k)\big)
\;\mid\; A=a
\big].
\end{align}
Then, the corrected functional
$\theta_{S,\beta}(a)\coloneqq g(S(a))+r_{S,\beta}(a)$
is \textbf{Neyman-orthogonal} at $(S_0,\beta_0)$
with respect to both nuisance components.
Moreover, if $\hat S$ and $\hat\beta$ are estimated via
cross-fitting, their conditional bias is \textbf{second-order}:
\begin{align}
\mid \theta_{\hat S,\hat\beta}(a)-\theta_0(a)\mid
\lesssim
\|\hat S(a)-S_0(a)\|_{L^2(\mathbb P(\cdot\mid a))}^2
+
\|\hat\beta(a)-\beta_0(a)\|_{L^2(\mathbb P(\cdot\mid a))}^2.
\end{align}
\end{theorem}
}

\begin{proof}
\rebuttal{
We fix $a$ outside a $\mathbb P_A$-null set where the assumptions and the conditional Riesz identity hold. By the conditional mean-zero noise assumption and the conditionally i.i.d.\ observation design given $(A,K)$, the conditional correction for fixed nuisance operators $S$ and $\beta$ satisfies 
\begin{align}
r_{S,\beta}(a)
&=
\mathbb E\left[
\beta(a)(X)
\big(S_0(a)(X)-S(a)(X)\big)
\mid A=a
\right].
\label{eq:conditional_correction_proof}
\end{align}
In particular, $r_{S_0,\beta_0}(a)=0$ and hence $\theta_{S_0,\beta_0}(a)=\theta_0(a)$.
}

\rebuttal{
\noindent\textbf{Neyman orthogonality.}
Consider admissible perturbations $S_t=S_0+th$ and $\beta_t=\beta_0+tb$. By \Eqref{eq:conditional_correction_proof}, 
\begin{align}
r_{S_t,\beta_t}(a)
&=
-t\,\mathbb E[
\beta_0(a)(X)h(a)(X)\mid A=a
]
\nonumber\\
&\quad
-t^2\,\mathbb E[
b(a)(X)h(a)(X)\mid A=a
].
\end{align}
Therefore, by Fr\'echet differentiability of $g$,
\begin{align}
&\left.
\frac{\diff}{\diff t}
\theta_{S_t,\beta_t}(a)
\right|_{t=0}
\nonumber\\
&\quad=
\Diff g_{S_0(a)}(h(a))
-
\mathbb E[
\beta_0(a)(X)h(a)(X)\mid A=a
]
\\
&\quad=0,
\end{align}
where the last equality follows from the conditional Riesz identity in \Eqref{eq:riesz}. Since $h$ and $b$ are arbitrary admissible  perturbations, the corrected functional is Neyman-orthogonal at $(S_0,\beta_0)$ with respect to both nuisance components.
}

\rebuttal{
\noindent\textbf{Second-order remainder.}
By cross-fitting, $\hat S$ and $\hat\beta$ are independent of the evaluation observations. We condition on the corresponding nuisance-training sample $\mathcal D^{\mathrm{train}}$, treating both fitted operators as fixed.}

\rebuttal{Let $\Delta(a)=\hat S(a)-S_0(a)$.
By Assumption~\ref{ass:g_smooth}, the second-order Taylor expansion gives
\begin{align}
g(\hat S(a))
&=
g(S_0(a))
+
\Diff g_{S_0(a)}(\Delta(a))
+
R(S_0(a),\Delta(a)),
\end{align}
where
\begin{align}
|R(S_0(a),\Delta(a))|
\lesssim
\|\Delta(a)\|_{\mathcal U}^2.
\end{align}
}

\rebuttal{
On the other hand, applying \Eqref{eq:conditional_correction_proof} conditionally on $\mathcal D^{\mathrm{train}}$ yields
\begin{align}
r_{\hat S,\hat\beta}(a)
&=
-\mathbb E\left[
\hat\beta(a)(X)\Delta(a)(X)
\mid A=a,\mathcal D^{\mathrm{train}}
\right]
\\
&=
-\mathbb E\left[
\beta_0(a)(X)\Delta(a)(X)
\mid A=a,\mathcal D^{\mathrm{train}}
\right]
\nonumber\\
&\quad-
\mathbb E\left[
(\hat\beta-\beta_0)(a)(X)\Delta(a)(X)
\mid A=a,\mathcal D^{\mathrm{train}}
\right]
\\
&=
-\Diff g_{S_0(a)}(\Delta(a))
\nonumber\\
&\quad-
\mathbb E\left[
(\hat\beta-\beta_0)(a)(X)\Delta(a)(X)
\mid A=a,\mathcal D^{\mathrm{train}}
\right],
\end{align}
where the last equality follows from the conditional Riesz identity.
}

\rebuttal{
Combining both expressions, the first-order term cancels exactly, yielding
\begin{align}
&\theta_{\hat S,\hat\beta}(a)-\theta_0(a)
\nonumber\\
&\quad=
R(S_0(a),\Delta(a))
\nonumber\\
&\qquad-
\mathbb E\left[
(\hat\beta-\beta_0)(a)(X)\Delta(a)(X)
\mid A=a,\mathcal D^{\mathrm{train}}
\right].
\label{eq:conditional_remainder_proof}
\end{align}
}

\rebuttal{
By conditional Cauchy--Schwarz and $2xy\leq x^2+y^2$, the remaining cross-term satisfies
\begin{align}
&\Big|
\mathbb E\left[
(\hat\beta-\beta_0)(a)(X)\Delta(a)(X)
\mid A=a,\mathcal D^{\mathrm{train}}
\right]
\Big|
\nonumber\\
&\quad\leq
\frac12
\|\hat\beta(a)-\beta_0(a)\|_{L^2(\mathbb P(\cdot\mid a))}^2
\nonumber\\
&\qquad+
\frac12
\|\Delta(a)\|_{L^2(\mathbb P(\cdot\mid a))}^2.
\end{align}
Furthermore, by Assumption~\ref{assumption:design},
the bounded density ratio
$\xi_0(a)=\diff\mu/\diff\mathbb P(\cdot\mid a)$
implies
\begin{align}
\|\Delta(a)\|_{\mathcal U}^2
&=
\mathbb E\left[
\xi_0(a)(X)\Delta(a)(X)^2
\mid A=a
\right]
\nonumber\\
&\lesssim
\|\Delta(a)\|_{L^2(\mathbb P(\cdot\mid a))}^2.
\end{align}
}

\rebuttal{
Substituting these bounds into \Eqref{eq:conditional_remainder_proof} gives
\begin{align}
&\left| \theta_{\hat S,\hat\beta}(a)-\theta_0(a)\right|
\nonumber\\
&\quad\lesssim
\|\hat S(a)-S_0(a)\|_{L^2(\mathbb P(\cdot\mid a))}^2
\nonumber\\
&\qquad+
\|\hat\beta(a)-\beta_0(a)\|_{L^2(\mathbb P(\cdot\mid a))}^2.
\end{align}
Thus, conditional on the nuisance-training sample, the corrected functional has no first-order bias from either nuisance component, establishing the claim.
}
\end{proof}

\clearpage
\begin{remark}[Asymptotic normality]\label{prop:asy_appendix}
If $\hat S$ and $\hat\beta$ are cross-fitted, $\mathbb E[\psi_i^2]<\infty$, and the nuisance estimators satisfy
\begin{align}
\|\hat S-S_0\|_{L^2(\mathbb P)}^2 = o_p(n^{-1/2}),
\qquad
\|\hat\beta-\beta_0\|_{L^2(\mathbb P)}^2 = o_p(n^{-1/2}),\label{eq:minimal_rates_appendix}
\end{align}
and, in addition, the estimated score is conditionally mean-square consistent:
\begin{align}
\mathbb E\!\left[(\tilde\psi_i-\psi_i)^2\mid \mathcal D^{\mathrm{train}}\right]
=
o_p(1),
\label{eq:score_ms_consistency_appendix}
\end{align}
then
\begin{align}
\sqrt{n}(\hat\theta-\theta)\xrightarrow{d}\mathcal N\!\big(0,\mathbb V[\psi_i]\big).
\end{align}
Moreover, a consistent estimator of the asymptotic variance of $\hat\theta$ is
\begin{align}
\widehat{\mathbb V}(\hat\theta)
=
\frac{1}{n^2}\sum_{i=1}^n(\tilde\psi_i-\hat\theta)^2.
\end{align}
\end{remark}

\begin{proof}
Let
\begin{align}
\delta_i\coloneqq\tilde\psi_i-\psi_i.
\end{align}
Then
\begin{align}
\hat\theta-\theta
=
\frac{1}{n}\sum_{i=1}^n\psi_i
+
\frac{1}{n}\sum_{i=1}^n\delta_i.
\end{align}
Since $\mathbb E[\psi_i]=0$, we may further decompose
\begin{align}
\sqrt n(\hat\theta-\theta)
&=
\frac{1}{\sqrt n}\sum_{i=1}^n\psi_i
+
\frac{1}{\sqrt n}\sum_{i=1}^n
\Big(\delta_i-\mathbb E[\delta_i\mid\mathcal D^{\mathrm{train}}]\Big)
+
\sqrt n\,\mathbb E[\delta_i\mid\mathcal D^{\mathrm{train}}].
\label{eq:clt_decomp_corrected}
\end{align}

\medskip
\noindent\textbf{Step 1: Bias term.}
By definition, $\delta_i=\tilde\psi_i-\psi_i$, and since $\mathbb E[\psi_i]=0$, we have
\begin{align}
\mathbb E[\delta_i\mid\mathcal D^{\mathrm{train}}]
=
\mathbb E[\tilde\psi_i-\theta\mid\mathcal D^{\mathrm{train}}].
\end{align}
Hence, by the second-order remainder result (Theorem~\ref{thm:remainder_appendix}), we obtain
\begin{align}
\mathbb E[\delta_i\mid\mathcal D^{\mathrm{train}}]
=
O_p\!\left(
\|\hat S-S_0\|_{L^2(\mathbb P)}^2
+
\|\hat\beta-\beta_0\|_{L^2(\mathbb P)}^2
\right).
\end{align}
Under \Eqref{eq:minimal_rates_appendix}, this yields
\begin{align}
\sqrt n\,\mathbb E[\delta_i\mid\mathcal D^{\mathrm{train}}]
=
o_p(1).
\label{eq:bias_negligible_corrected}
\end{align}

\medskip
\noindent\textbf{Step 2: Stochastic remainder.}
Conditional on $\mathcal D^{\mathrm{train}}$, the variables
$\{\delta_i\}_{i=1}^n$ are i.i.d. Therefore,
\begin{align}
\mathbb V\!\left(
\frac{1}{\sqrt n}\sum_{i=1}^n
\big(\delta_i-\mathbb E[\delta_i\mid\mathcal D^{\mathrm{train}}]\big)
\,\middle|\,
\mathcal D^{\mathrm{train}}
\right)
=
\mathbb V (\delta_i\mid\mathcal D^{\mathrm{train}})
\le
\mathbb E[\delta_i^2\mid\mathcal D^{\mathrm{train}}].
\end{align}
By \Eqref{eq:score_ms_consistency_appendix}, the right-hand side is $o_p(1)$, so conditional Chebyshev implies
\begin{align}
\frac{1}{\sqrt n}\sum_{i=1}^n
\big(\delta_i-\mathbb E[\delta_i\mid\mathcal D^{\mathrm{train}}]\big)
=
o_p(1).
\label{eq:stochastic_remainder_negligible}
\end{align}

\medskip
\noindent\textbf{Step 3: Central limit theorem (CLT).}
Combining \Eqref{eq:clt_decomp_corrected}, \Eqref{eq:bias_negligible_corrected}, and
\Eqref{eq:stochastic_remainder_negligible}, we obtain
\begin{align}
\sqrt n(\hat\theta-\theta)
=
\frac{1}{\sqrt n}\sum_{i=1}^n\psi_i
+
o_p(1).
\end{align}
Since $\{\psi_i\}_{i=1}^n$ are i.i.d., mean zero, and have finite variance, the classical CLT gives
\begin{align}
\frac{1}{\sqrt n}\sum_{i=1}^n\psi_i
\xrightarrow{d}
\mathcal N\!\big(0,\mathbb V[\psi_i]\big).
\end{align}
Slutsky's theorem then yields
\begin{align}
\sqrt n(\hat\theta-\theta)
\xrightarrow{d}
\mathcal N\!\big(0,\mathbb V[\psi_i]\big).
\end{align}

\medskip
\noindent\textbf{Step 4: Variance estimation.}
We first show that
\begin{align}
\frac{1}{n}\sum_{i=1}^n(\tilde\psi_i-\hat\theta)^2
\xrightarrow{p}
\mathbb V[\psi_i].
\label{eq:sample_var_score_consistency}
\end{align}
Write again $\tilde\psi_i=\psi_i+\delta_i$. Then, conditional on
$\mathcal D^{\mathrm{train}}$, the variables $\delta_i$ are i.i.d., so by the weak law of large numbers and \Eqref{eq:score_ms_consistency_appendix},
\begin{align}
\frac{1}{n}\sum_{i=1}^n\delta_i^2
-
\mathbb E[\delta_i^2\mid\mathcal D^{\mathrm{train}}]
=
o_p(1),
\end{align}
hence
\begin{align}
\frac{1}{n}\sum_{i=1}^n\delta_i^2=o_p(1).
\label{eq:delta_square_avg}
\end{align}
Moreover, by Cauchy--Schwarz inequality, we have
\begin{align}
\left|
\frac{1}{n}\sum_{i=1}^n\psi_i\delta_i
\right|
\le
\left(\frac{1}{n}\sum_{i=1}^n\psi_i^2\right)^{1/2}
\left(\frac{1}{n}\sum_{i=1}^n\delta_i^2\right)^{1/2}
=
o_p(1),
\label{eq:cross_term_var}
\end{align}
since $n^{-1}\sum_i\psi_i^2\xrightarrow{p}\mathbb E[\psi_i^2]$ and
\Eqref{eq:delta_square_avg} holds.
Therefore,
\begin{align}
\frac{1}{n}\sum_{i=1}^n\tilde\psi_i^2
&=
\frac{1}{n}\sum_{i=1}^n\psi_i^2
+
\frac{2}{n}\sum_{i=1}^n\psi_i\delta_i
+
\frac{1}{n}\sum_{i=1}^n\delta_i^2
\\
&=
\frac{1}{n}\sum_{i=1}^n\psi_i^2+o_p(1)
\xrightarrow{p}
\mathbb E[\psi_i^2].
\end{align}
Also, $\hat\theta\xrightarrow{p}\theta$, so
\begin{align}
\frac{1}{n}\sum_{i=1}^n(\tilde\psi_i-\hat\theta)^2
=
\frac{1}{n}\sum_{i=1}^n\tilde\psi_i^2-\hat\theta^2
\xrightarrow{p}
\mathbb E[\psi_i^2]-\theta^2
=
\mathbb V[\psi_i],
\end{align}
which proves \Eqref{eq:sample_var_score_consistency}. Consequently,
\begin{align}
\widehat{\mathbb V}(\hat\theta)
=
\frac{1}{n^2}\sum_{i=1}^n(\tilde\psi_i-\hat\theta)^2
=
\frac{1}{n}
\left(
\frac{1}{n}\sum_{i=1}^n(\tilde\psi_i-\hat\theta)^2
\right)
\end{align}
satisfies
\begin{align}
n\,\widehat{\mathbb V}(\hat\theta)\xrightarrow{p}\mathbb V[\psi_i].
\end{align}
Therefore, $\widehat{\mathbb V}(\hat\theta)$ is a consistent estimator of the asymptotic variance $\mathbb V[\psi_i]/n$ of $\hat\theta$.
\end{proof}

\clearpage

\subsection{Automatic DML for operator-valued nuisances}\label{appendix:proofs_autoDML}

\begin{proposition}[Decomposition of the debiasing weight]\label{prop:beta_decomp_appendix}
Suppose $\mathcal U=L^2(\Omega,\mu)$ and let $w_g(u)\in L^2(\Omega,\mu)$ denote the Riesz
representer of $\Diff g_u$ with respect to $\mu$. Then the unique solution $\beta_0(A)(x)$ satisfying \Eqref{eq:riesz} admits ($\mathbb P (\cdot \mid A)$-a.e.) the decomposition
\begin{align}
&\beta_0(A)(x)
=
\xi_0(A)(x)\,w_g(S_0(A))(x).\label{eq:decomposition_appendix}
\end{align}
\end{proposition}

\begin{proof}
By definition of the Riesz representer, we have
\begin{align}
\Diff g_{S_0(A)}(h)
=
\int_\Omega w_g(S_0(A))(x)\,h(x)\,\diff\mu(x).
\end{align}
On the other hand,
\begin{align}
\mathbb{E}\!\left[\beta_0(A)(X)\,h(X)\mid A\right]
=
\int_\Omega \beta_0(A)(x)\,h(x)\,\diff \mathbb{P}(x\mid A).
\end{align}
Equating the two expressions for all $h\in \mathcal H_A$ implies
\begin{align}
\beta_0(A)(x)\,\diff \mathbb{P}(x\mid A)
=
w_g(S_0(A))(x)\,\diff\mu(x),
\end{align}
which yields the decomposition in \Eqref{eq:decomposition_appendix} by Radon--Nikodym differentiation.
\end{proof}

\clearpage

\begin{theorem}[Primal characterization of the Riesz representer]
\label{thm:primal_riesz_appendix}
The unique minimizer of \Eqref{eq:primal_riesz} is the Riesz representer $\beta_0$ satisfying
\Eqref{eq:riesz_beta_primal}.
\end{theorem}

\begin{proof}\label{sec:proofs_primal_riesz}

We prove that the unique minimizer of the population risk
\begin{align}
\mathcal R(\beta)
\coloneqq
\mathbb{E}\left[
\mathbb{E}[\beta(A)(X)^2\mid A]
-
2\,\Diff g_{S_0(A)}\big(\beta(A)(\cdot)\big)
\right]
\label{eq:appendix_risk}
\end{align}
is the Riesz representer $\beta_0$ satisfying
\begin{align}
\mathbb{E}[\beta_0(A)(X)h(X)\mid A]
=
\Diff g_{S_0(A)}(h),
\qquad \forall h\in\mathcal{H}_A.
\label{eq:appendix_riesz}
\end{align}

\vspace{0.5em}
\noindent
\textbf{Step 1: Conditional Hilbert space structure.}
For each $A=a$, define the Hilbert space
\begin{align}
\mathcal H_a \coloneqq L^2(\mathbb P(\cdot\mid A=a))
\end{align}
with inner product
\begin{align}
\langle f,h\rangle_a \coloneqq \mathbb{E}[f(X)h(X)\mid A=a].
\end{align}
By Assumption~\ref{assumption:design}, for every $h\in\mathcal H_a$, we have
\begin{align}
\|h\|_{\mathcal U}^2
=
\int_\Omega h(x)^2\,\diff\mu(x)
=
\int_\Omega h(x)^2\,\xi_0(a)(x)\,\diff\mathbb P(x\mid a)
\le
C\,\|h\|_{\mathcal H_a}^2.
\label{eq:conditional_embedding_appendix}
\end{align}
Hence, $\mathcal H_a$ embeds continuously into $\mathcal U=L^2(\Omega,\mu)$.
Therefore, the Fr\'echet derivative
\begin{align}
\mathcal L_a(h)\coloneqq\Diff g_{S_0(a)}(h)
\end{align}
defines a continuous linear functional on $\mathcal H_a$ by restriction.

\vspace{0.5em}
\noindent
\textbf{Step 2: Existence of the Riesz representer.}
By the Riesz representation theorem, for each $a$, there exists a unique
$\beta_0(a)(\cdot)\in\mathcal H_a$ such that
\begin{align}
\mathcal L_a(h)
=
\langle \beta_0(a)(\cdot), h\rangle_a
=
\mathbb{E}[\beta_0(A)(X)h(X)\mid A=a],
\qquad \forall h\in\mathcal H_a.
\end{align}
This establishes the conditional Riesz identity on $\mathcal H_a$.

\vspace{0.5em}
\noindent
\textbf{Step 3: Variational characterization.}
We now use the standard variational characterization of the Riesz representer
in a Hilbert space \citep{Chernozhukov.2021}. For any Hilbert space $(\mathcal H,\langle\cdot,\cdot\rangle)$
and continuous linear functional $L(h)=\langle r,h\rangle$, the representer $r$
satisfies
\begin{align}
r = \arg\min_{b\in\mathcal H}
\left\{
\|b\|_{\mathcal H}^2 - 2 L(b)
\right\}.
\label{eq:appendix_variational}
\end{align}
Indeed, for any $b\in\mathcal H$, we obtain
\begin{align}
\|b\|_{\mathcal H}^2 - 2 L(b)
=
\|b\|_{\mathcal H}^2 - 2\langle r,b\rangle
=
\|b-r\|_{\mathcal H}^2 - \|r\|_{\mathcal H}^2,
\end{align}
which is uniquely minimized at $b=r$. This variational formulation is the basis of automatic Riesz regression in debiased machine learning \citep{Chernozhukov.2021, Chernozhukov.2022}. Our construction extends this idea to a conditional, operator-valued setting.

\vspace{0.5em}
\noindent
\textbf{Step 4: Conditional minimization.}
Applying \Eqref{eq:appendix_variational} with
\begin{align}
\mathcal H = \mathcal H_a,
\quad
L = \mathcal L_a,
\quad
r = \beta_0(a)(\cdot),
\end{align}
we obtain, for each fixed $a$,
\begin{align}
\beta_0(a)(\cdot)
=
\arg\min_{\beta(a)(\cdot)\in\mathcal H_a}
\left\{
\mathbb{E}[\beta(A)(X)^2\mid A=a]
-
2\,\Diff g_{S_0(a)}\big(\beta(a)(\cdot)\big)
\right\}.
\label{eq:appendix_conditional}
\end{align}

\vspace{0.5em}
\noindent
\textbf{Step 5: Averaging over $A$.}
Finally, observe that the population risk \Eqref{eq:appendix_risk} can be written as
\begin{align}
\mathcal R(\beta)
=
\mathbb{E}\Big[
\mathcal R_A\big(\beta(A)(\cdot)\big)
\Big],
\end{align}
where
\begin{align}
\mathcal R_A\big(\beta(A)(\cdot)\big)
\coloneqq
\mathbb{E}[\beta(A)(X)^2\mid A]
-
2\,\Diff g_{S_0(A)}\big(\beta(A)(\cdot)\big).
\end{align}
Since the objective separates across values of $A$, minimizing $\mathcal R(\beta)$
is equivalent to minimizing $\mathcal R_A(\beta(A)(\cdot))$ for almost every $A$.
By \Eqref{eq:appendix_conditional}, the unique minimizer is therefore
$\beta_0(A)(\cdot)$.

\medskip
\noindent
This means that $\beta_0$ is the unique minimizer of \Eqref{eq:appendix_risk}, which completes the proof.

\end{proof}

\clearpage

\section{Real-data experiment on ETTh1}
\label{sec:rwd}

We additionally evaluate \method on the ETTh1 electricity-transformer dataset \citep{haoyietal-informer-2021}. This experiment complements the controlled simulations in the main paper with real input and output trajectories. Specifically, all covariate and oil temperature (OT) measurements are taken from the real dataset; only the sparse observation pattern of the future OT trajectory is imposed synthetically. This allows us to retain a complete future trajectory for evaluation while reproducing the partial observation setting considered throughout the paper.

\textbf{Data and prediction task.}
We sort ETTh1 chronologically and use all seven numerical channels. Each sample consists of an input trajectory $A_i$ containing the previous 96 hourly observations of all channels and an output trajectory $S_0(A_i)$ containing the subsequent $24$ hourly OT measurements. We use a chronological split with $512$ training, $64$ validation, and $128$ test trajectories.
The plug-in estimator and \method use the \emph{same frozen fitted solution operator} within each of the $50$ seed repetitions.

\textbf{Functionals.}
We consider two functionals of the future OT trajectory. First, we use the time-normalized 24-hour area under the curve. Thus, in this experiment AUC is numerically equal to the average future OT rather than the unnormalized degree-hour integral. Second, we consider a smooth approximation of the daily maximum via soft-cmax.

\textbf{Sparse observation design.}
To emulate irregularly observed future trajectories, exactly $K=8$ of the $24$ future OT measurements are retained for each sample. Locations are sampled without replacement using the categorical draw weights
\begin{align}
p_j(\rho)
=
\frac{1-\rho}{24}
+
\rho q_j,
\end{align}
where
\begin{align}
q_j
=
\begin{cases}
0.02/24+0.98/8, & j=1,\ldots,8,\\
0.02/24, & j=9,\ldots,24,
\end{cases}
\end{align}
for $\rho\in\{0,0.25,0.5,0.75,1\}$. At $\rho=0$, all $24$ future locations have equal draw weight $1/24$. Increasing $\rho$ progressively concentrates observations in the first eight hours of the forecast horizon.

\begin{table*}[h]
\centering
\setlength{\tabcolsep}{3.5pt}
\small
\resizebox{\textwidth}{!}{%
\begin{tabular}{@{}lcccccccccc@{}}
\toprule
& \multicolumn{2}{c}{$\rho=0$}
& \multicolumn{2}{c}{$\rho=0.25$}
& \multicolumn{2}{c}{$\rho=0.5$}
& \multicolumn{2}{c}{$\rho=0.75$}
& \multicolumn{2}{c}{$\rho=1$} \\
\cmidrule(lr){2-3}
\cmidrule(lr){4-5}
\cmidrule(lr){6-7}
\cmidrule(lr){8-9}
\cmidrule(lr){10-11}

& AUC & Soft Cmax
& AUC & Soft Cmax
& AUC & Soft Cmax
& AUC & Soft Cmax
& AUC & Soft Cmax \\
\midrule

Plug-in
& $1.226\pm0.043$ & $1.171\pm0.041$
& $1.210\pm0.043$ & $1.181\pm0.043$
& $1.224\pm0.044$ & $1.215\pm0.044$
& $1.310\pm0.057$ & $1.340\pm0.062$
& $1.976\pm0.094$ & $2.212\pm0.117$ \\

\method
& $0.112\pm0.010$ & $0.239\pm0.014$
& $0.128\pm0.013$ & $0.264\pm0.020$
& $0.167\pm0.017$ & $0.305\pm0.022$
& $0.178\pm0.016$ & $0.316\pm0.020$
& $0.638\pm0.094$ & $0.874\pm0.149$ \\

\quad Rel. improvement
& $\greentext{90.8\%}$ & $\greentext{79.6\%}$
& $\greentext{89.5\%}$ & $\greentext{77.7\%}$
& $\greentext{86.3\%}$ & $\greentext{74.9\%}$
& $\greentext{86.4\%}$ & $\greentext{76.4\%}$
& $\greentext{67.7\%}$ & $\greentext{60.5\%}$ \\
\bottomrule
\end{tabular}%
}
\caption{\textbf{ETTh1 real-data results.} RMSE $\pm$ SE over 50 repetitions for normalized AUC (mean future OT) and Soft Cmax under increasingly concentrated sparse-observation designs. The underlying transformer trajectories are real; only the future observation locations are synthetically subsampled. $\Rightarrow$ \emph{\method improves over the plug-in estimator for both functionals and at every sampling-irregularity level.}}
\label{tab:etth1_real}
\end{table*}

\textbf{Results.}
Table~\ref{tab:etth1_real} show that, across both functionals, \method achieves lower RMSE than the plug-in estimator for every value of $\rho$. For normalized AUC, the improvement ranges from $90.8\%$ under uniform sampling to $67.7\%$ under the most strongly concentrated early-block design. For soft cmax, the corresponding improvement ranges from $79.6\%$ to $60.5\%$.

As $\rho$ increases, observations become progressively concentrated in the beginning of the 24-hour forecast window and provide decreasingly direct information about later hours. Accordingly, estimation becomes more difficult for both methods, particularly at $\rho=1$. Nevertheless, \method retains a substantial advantage throughout.

These results show that \method is directly applicable to real-world data. Hence, the gains from the proposed debiasing procedure extend beyond the simulated pharmacokinetic and PDE settings: when applied to real transformer trajectories, \method consistently reduces functional estimation error for both a linear trajectory summary and a nonlinear peak functional.

\clearpage

\section{{Additional results}}\label{sec:additional_results}

\subsection{\rebuttal{Input-specific functionals for Darcy flow}}
\label{sec:darcy_input_specific}

\begin{wraptable}{r}{0.4\linewidth}
\vspace{-0.3cm}
\centering
\scriptsize
\setlength{\tabcolsep}{3pt}
\resizebox{\linewidth}{!}{%
\begin{tabular}{@{}lccc@{}}
\toprule
\rebuttal{$\kappa$}
& \rebuttal{Plug-in}
& \rebuttal{\method (ours)}
& \rebuttal{Rel. improvement} \\
\midrule
\rebuttal{$0.0$}
& \rebuttal{$1.678\pm0.127$}
& \rebuttal{$\mathbf{1.535\pm0.112}$}
& \rebuttal{$\greentext{8.5\%}$} \\
\rebuttal{$0.2$}
& \rebuttal{$1.374\pm0.111$}
& \rebuttal{$\mathbf{1.268\pm0.098}$}
& \rebuttal{$\greentext{7.7\%}$} \\
\rebuttal{$0.4$}
& \rebuttal{$1.126\pm0.096$}
& \rebuttal{$\mathbf{1.038\pm0.087}$}
& \rebuttal{$\greentext{7.8\%}$} \\
\rebuttal{$0.6$}
& \rebuttal{$0.922\pm0.084$}
& \rebuttal{$\mathbf{0.853\pm0.077}$}
& \rebuttal{$\greentext{7.5\%}$} \\
\rebuttal{$0.8$}
& \rebuttal{$0.771\pm0.078$}
& \rebuttal{$\mathbf{0.716\pm0.072}$}
& \rebuttal{$\greentext{7.2\%}$} \\
\rebuttal{$1.0$}
& \rebuttal{$0.634\pm0.068$}
& \rebuttal{$\mathbf{0.591\pm0.064}$}
& \rebuttal{$\greentext{6.7\%}$} \\
\bottomrule
\end{tabular}%
}
\vspace{-0.3cm}
\caption{\rebuttal{
\textbf{Input-specific functional estimation on Darcy flow.}
Strong RMSE $\pm$ SE ($\times 10^4$) across $\kappa$.
Relative improvements are computed against the direct
operator plug-in.
}}
\label{tab:darcy_input_specific}
\end{wraptable}

\rebuttal{
We further evaluate input-specific functional estimation on the Darcy flow benchmark, using the same target functional as in
the main Darcy experiments. Following Theorem~\ref{thm:input_specific}, we learn a conditional residual correction and compare against the direct operator plug-in using the same learned solution operator. We evaluate strong RMSE across all six values of $\kappa$.
}

\rebuttal{
\noindent\textbf{Results.}
Table~\ref{tab:darcy_input_specific} shows that \method consistently outperforms the direct plug-in across all six settings, reducing strong RMSE by $6.7$--$8.5\%$. These results further support the effectiveness of conditional debiasing for input-specific prediction beyond the PK benchmark.
}

\subsection{CLT-based confidence intervals}\label{sec:confidence_intervals}

We report the empirical coverage for CLT-based confidence intervals in Table~\ref{tab:ci_coverage} on the AUC functional with the pharmacokinetics data. Using the CLT is grounded in Proposition~\ref{prop:asy_appendix}. We find that \method is reliable, while the empirical coverage from simple plug-in estimation is \textbf{not}.

\begin{table*}[h]
\centering
\setlength{\tabcolsep}{2.5pt}
\small
\resizebox{\textwidth}{!}{%
\begin{tabular}{@{}lcccccccccccccccccc@{}}
\toprule
& \multicolumn{2}{c}{$\rho=0$}
& \multicolumn{2}{c}{$\rho=0.125$}
& \multicolumn{2}{c}{$\rho=0.25$}
& \multicolumn{2}{c}{$\rho=0.375$}
& \multicolumn{2}{c}{$\rho=0.5$}
& \multicolumn{2}{c}{$\rho=0.625$}
& \multicolumn{2}{c}{$\rho=0.75$}
& \multicolumn{2}{c}{$\rho=0.875$}
& \multicolumn{2}{c}{$\rho=1$} \\
\cmidrule(lr){2-3}\cmidrule(lr){4-5}\cmidrule(lr){6-7}\cmidrule(lr){8-9}
\cmidrule(lr){10-11}\cmidrule(lr){12-13}\cmidrule(lr){14-15}\cmidrule(lr){16-17}\cmidrule(lr){18-19}

& Cov. & Length
& Cov. & Length
& Cov. & Length
& Cov. & Length
& Cov. & Length
& Cov. & Length
& Cov. & Length
& Cov. & Length
& Cov. & Length \\
\midrule
Plug-in
& $0.68$ & $0.02554$
& $0.88$ & $0.02555$
& $0.78$ & $0.02517$
& $0.76$ & $0.02530$
& $0.76$ & $0.02500$
& $0.80$ & $0.02543$
& $0.76$ & $0.02510$
& $0.76$ & $0.02569$
& $0.78$ & $0.02496$ \\
\method (ours)
& $0.94$ & $0.02747$
& $0.94$ & $0.02792$
& $0.94$ & $0.02736$
& $0.92$ & $0.02738$
& $0.90$ & $0.02729$
& {$0.96$} & $0.02768$
& $0.94$ & $0.02752$
& {$0.96$} & $0.02804$
& $0.92$ & $0.02749$ \\
\bottomrule
\end{tabular}%
}
\caption{Empirical coverage and mean confidence interval length for $95\%$ CLT-based confidence intervals across overlap levels $\rho$ over 50 repeats.}
\label{tab:ci_coverage}
\end{table*}

\paragraph{\rebuttal{Variance stability under irregular sampling.}}

\rebuttal{
To further assess the effect of increasingly concentrated observation designs, we analyze the variance of the debiased score across all
nine sampling regimes for AUC, smooth TAT, and soft Cmax. Using the saved subject-level pseudo-outcomes, we directly compute
\begin{align}
\widehat V_\psi
=
\frac{1}{n-1}\sum_{i=1}^{n}
(\tilde\psi_i-\hat\theta)^2,
\qquad
L=2z_{0.975}\sqrt{\widehat V_\psi/n},
\end{align}
where $L$ is the length of the corresponding CLT-based
confidence interval. We additionally evaluate the RMSE of $\hat\beta$ relative to the oracle Riesz representer on the output grid.
Importantly, this measures Riesz \emph{estimation error}, not the sampling variance of $\hat\beta$. All quantities are evaluated over the same $50$ independent repetitions.
}

\begin{wraptable}{r}{0.4\textwidth}
\vspace{-0.5cm}
\centering
\small
\setlength{\tabcolsep}{3pt}
\resizebox{\linewidth}{!}{%
\begin{tabular}{@{}lcccc@{}}
\toprule
\rebuttal{Functional}
& \rebuttal{$\Delta$ Riesz RMSE}
& \rebuttal{$\Delta$ Score variance}
& \rebuttal{$\Delta$ CI length}
& \rebuttal{CI coverage} \\
\midrule
\rebuttal{AUC}
& \rebuttal{$+153\%$}
& \rebuttal{$+0.28\%$}
& \rebuttal{$+0.08\%$}
& \rebuttal{$90$--$96\%$} \\
\rebuttal{Smooth TAT}
& \rebuttal{$+113\%$}
& \rebuttal{$+0.85\%$}
& \rebuttal{$+0.39\%$}
& \rebuttal{$90$--$96\%$} \\
\rebuttal{Soft Cmax}
& \rebuttal{$+109\%$}
& \rebuttal{$+3.57\%$}
& \rebuttal{$+1.47\%$}
& \rebuttal{$90$--$98\%$} \\
\bottomrule
\end{tabular}%
}
\vspace{-0.3cm}
\caption{\rebuttal{
Robustness under increasingly concentrated sampling.
Relative changes from $\rho=0$ to $\rho=1$ in oracle Riesz RMSE,
debiased-score variance, and mean confidence interval length.
Coverage ranges are reported across all nine sampling regimes.
}}
\label{tab:variance_stability}
\vspace{-0.3cm}
\end{wraptable}

\rebuttal{
\noindent\textbf{Results.}
Table~\ref{tab:variance_stability} reveals a striking distinction. Although Riesz estimation error increases substantially
($109$--$153\%$) as sampling becomes more concentrated, the variance of the final debiased score increases by only $0.28$--$3.57\%$, and confidence interval lengths by just $0.08$--$1.47\%$. Across all nine sampling regimes, score-variance fluctuations remain small and nonmonotonic, while empirical coverage remains reasonably close to the nominal $95\%$ level without systematic
deterioration.}

\rebuttal{Thus, \emph{even substantial increases in Riesz estimation error do not translate into comparable increases in the variance or
interval length of the final estimator}. These findings provide additional \textbf{empirical evidence of the robustness of our Neyman-orthogonal construction to nuisance estimation errors}. The conclusions concern the investigated designs, which retain
positive sampling overlap, rather than arbitrarily weak overlap.
}

\subsection{\rebuttal{Comparison with direct-solver Monte Carlo}}\label{sec:direct_solver}

\rebuttal{
We additionally compare \method against direct-solver Monte Carlo on the pharmacokinetics benchmark. We consider AUC, TAT, and Cmax across all nine sampling regimes $\rho\in\{0,0.125,\ldots,1\}$ over $50$ independent seeds.
}

\rebuttal{This comparison serves as a \emph{sanity check}. While \method achieves accuracy comparable to direct-solver Monte Carlo in this
benchmark, it retains the \emph{fundamental advantage of neural operators in amortizing computation across inputs}, particularly when repeated
full numerical solves are prohibitively expensive or unavailable.}

\paragraph{\rebuttal{Estimation accuracy.}}
\rebuttal{We first determine how many independent full numerical solves direct Monte Carlo requires to match the empirical RMSE of \method.
For each functional $g$, we estimate $\sigma_g^2=\mathrm{Var}(g(S_0(A)))$ from the $50$ independent reference pools, where each contains $M=2000$ fully solved trajectories. Since the direct-Monte-Carlo estimator is unbiased, its RMSE relative to an independent finite reference is
\begin{align}
\mathrm{RMSE}_{\mathrm{MC}}(B)
=\sigma_g\sqrt{1/B+1/M},
\end{align}
where $B$ denotes the number of full solves. We thus compute the equivalent solve budget $B_{\mathrm{eq}}
=\sigma_g^2/(\mathrm{RMSE}_{\method}^2-\sigma_g^2/M)$, which accounts for finite-reference uncertainty.
}

\rebuttal{\textbf{Results:} Figure~\ref{fig:direct_mc_rmse} compares the resulting analytical Monte Carlo accuracy curves with the empirical RMSE of \method. Across the three functionals and all sampling regimes, direct Monte Carlo requires approximately
\textbf{$166$--$355$ complete numerical solves} to match the accuracy achieved by \method using sparse,
noisy observations from $256$ test trajectories.}

\begin{figure*}[h]
\centering
\begin{minipage}{0.32\textwidth}
    \centering
    \includegraphics[width=\linewidth]{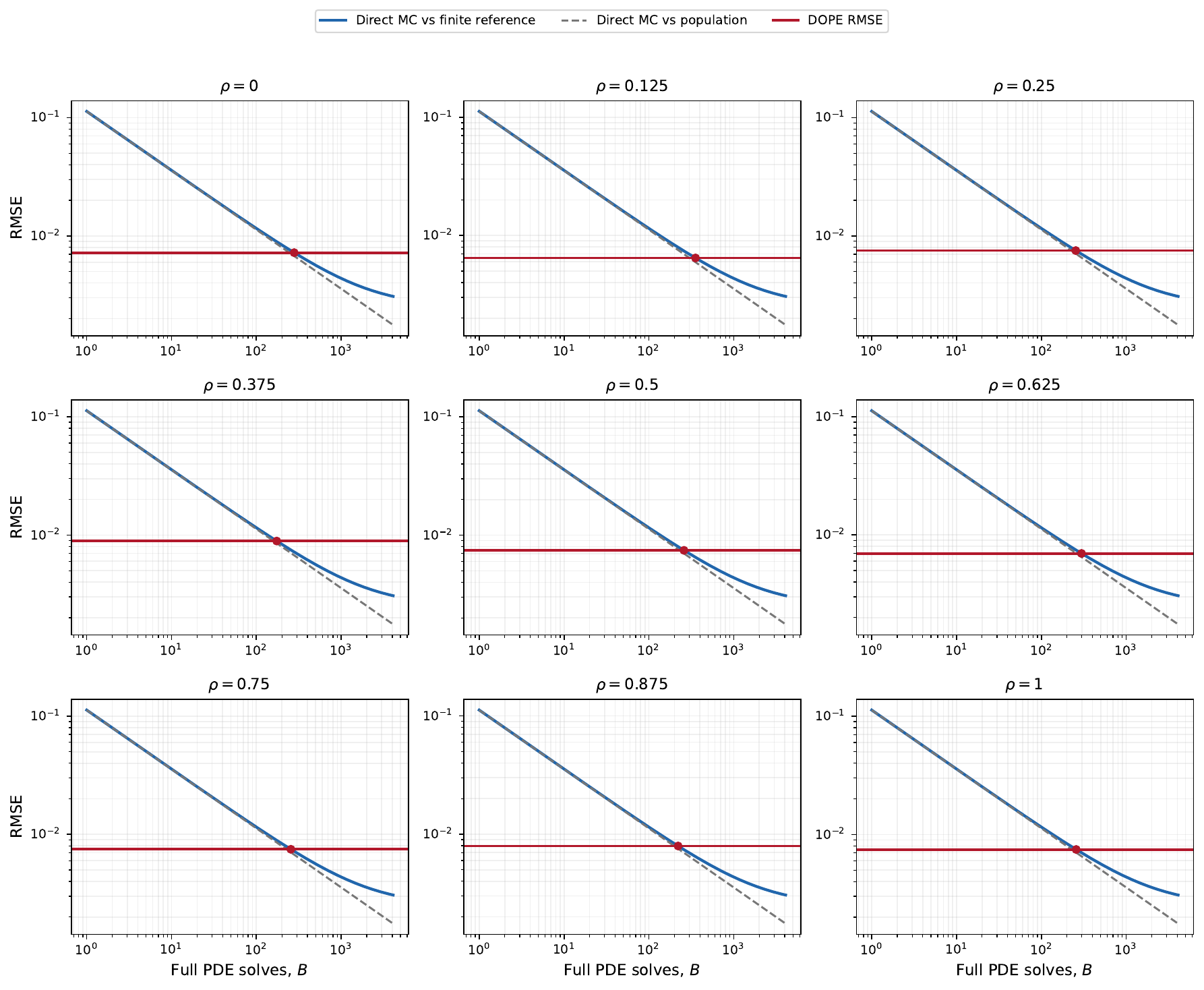}
    \rebuttal{(a) AUC}
\end{minipage}%
\hfill
\begin{minipage}{0.32\textwidth}
    \centering
    \includegraphics[width=\linewidth]{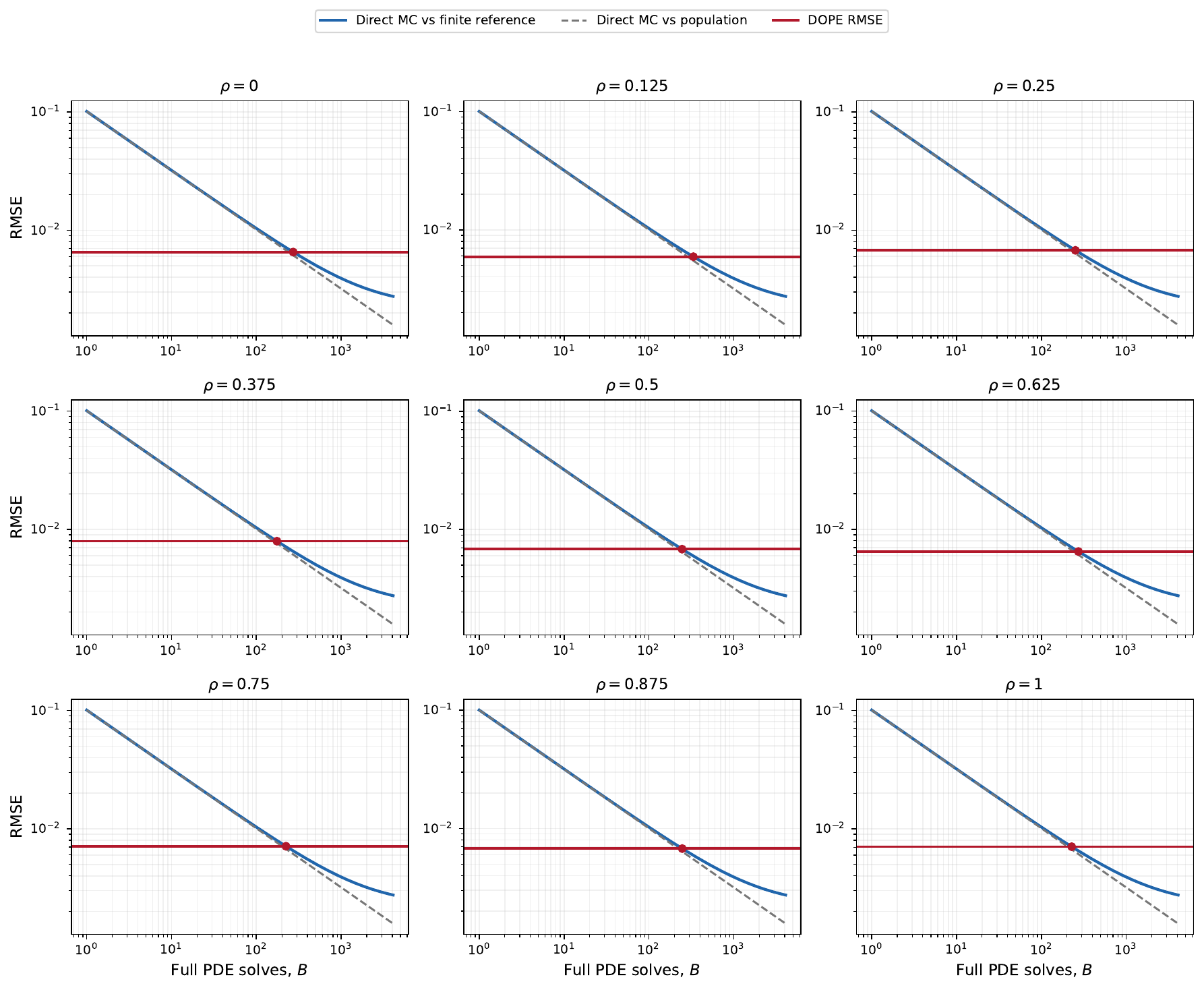}
    \rebuttal{(b) Smooth TAT}
\end{minipage}%
\hfill
\begin{minipage}{0.32\textwidth}
    \centering
    \includegraphics[width=\linewidth]{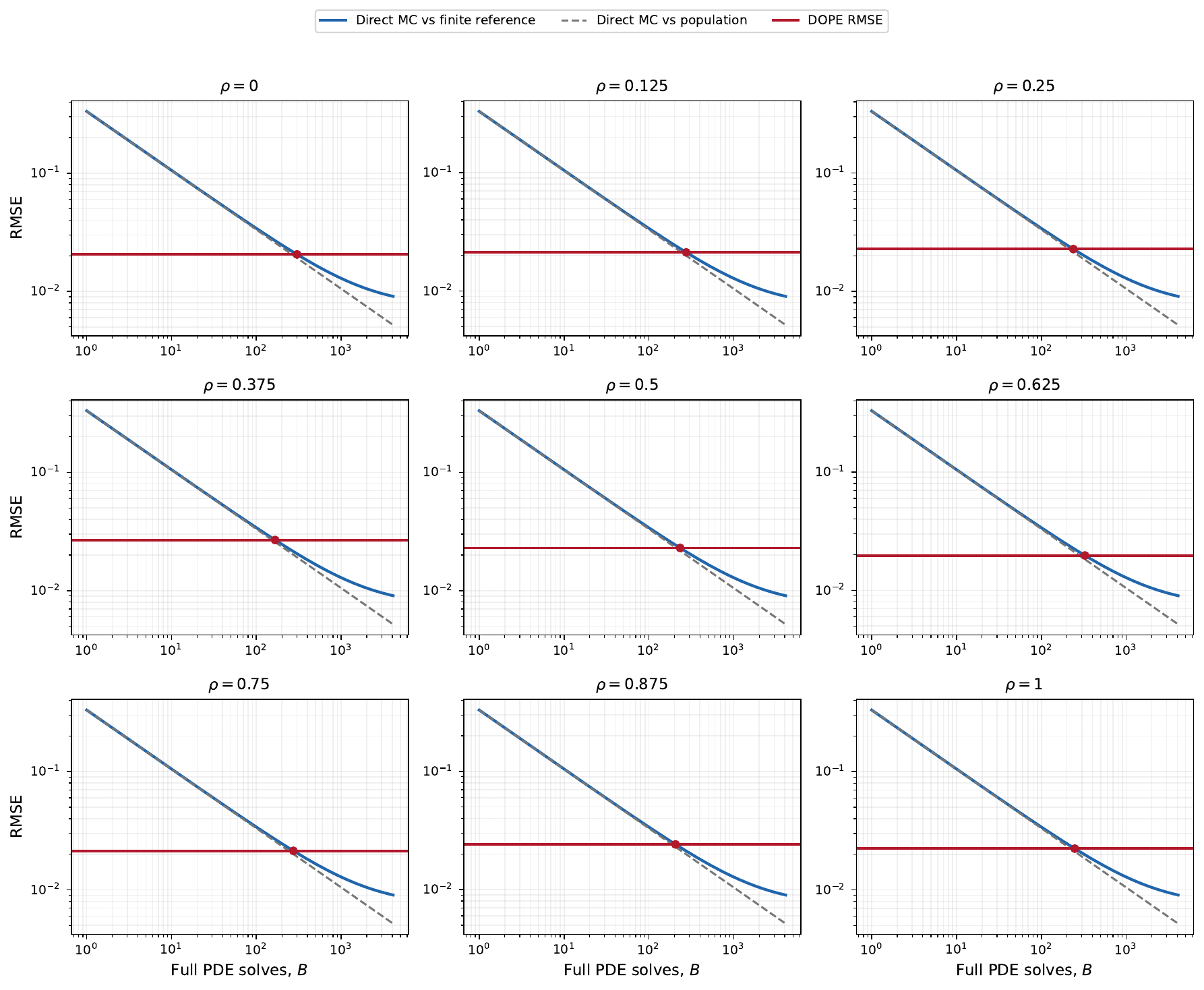}
    \rebuttal{(c) Soft Cmax}
\end{minipage}
\caption{\rebuttal{
\textbf{Direct-solver Monte Carlo: estimation accuracy.}
Analytical RMSE as a function of the number of independent full solves,
compared with the empirical RMSE of \method across sampling regimes.
Finite-reference uncertainty is included.
}}
\label{fig:direct_mc_rmse}
\end{figure*}

\begin{figure*}[h]
\centering
\begin{minipage}{0.32\textwidth}
    \centering
    \includegraphics[width=\linewidth]{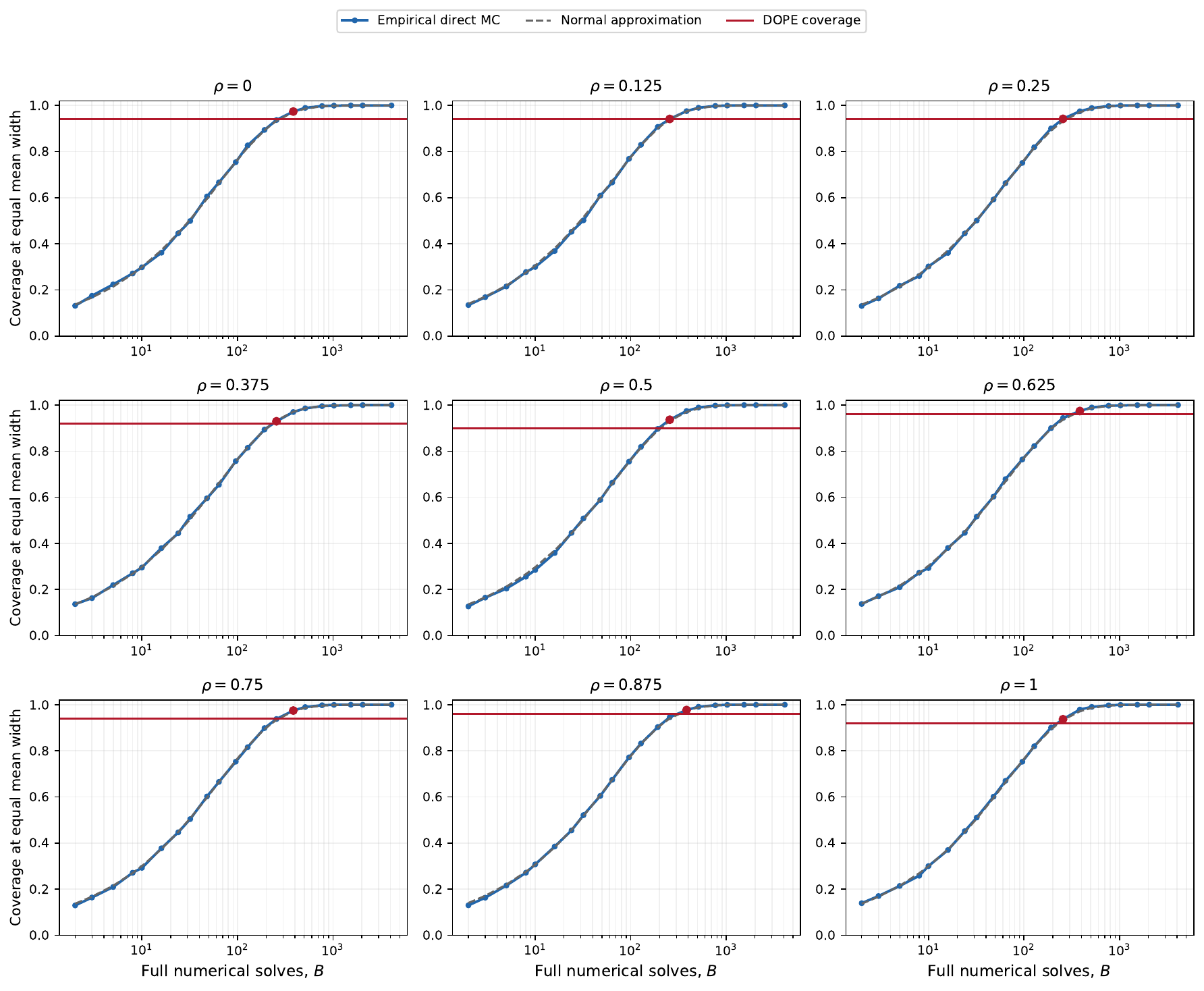}
    \rebuttal{(a) AUC}
\end{minipage}%
\hfill
\begin{minipage}{0.32\textwidth}
    \centering
    \includegraphics[width=\linewidth]{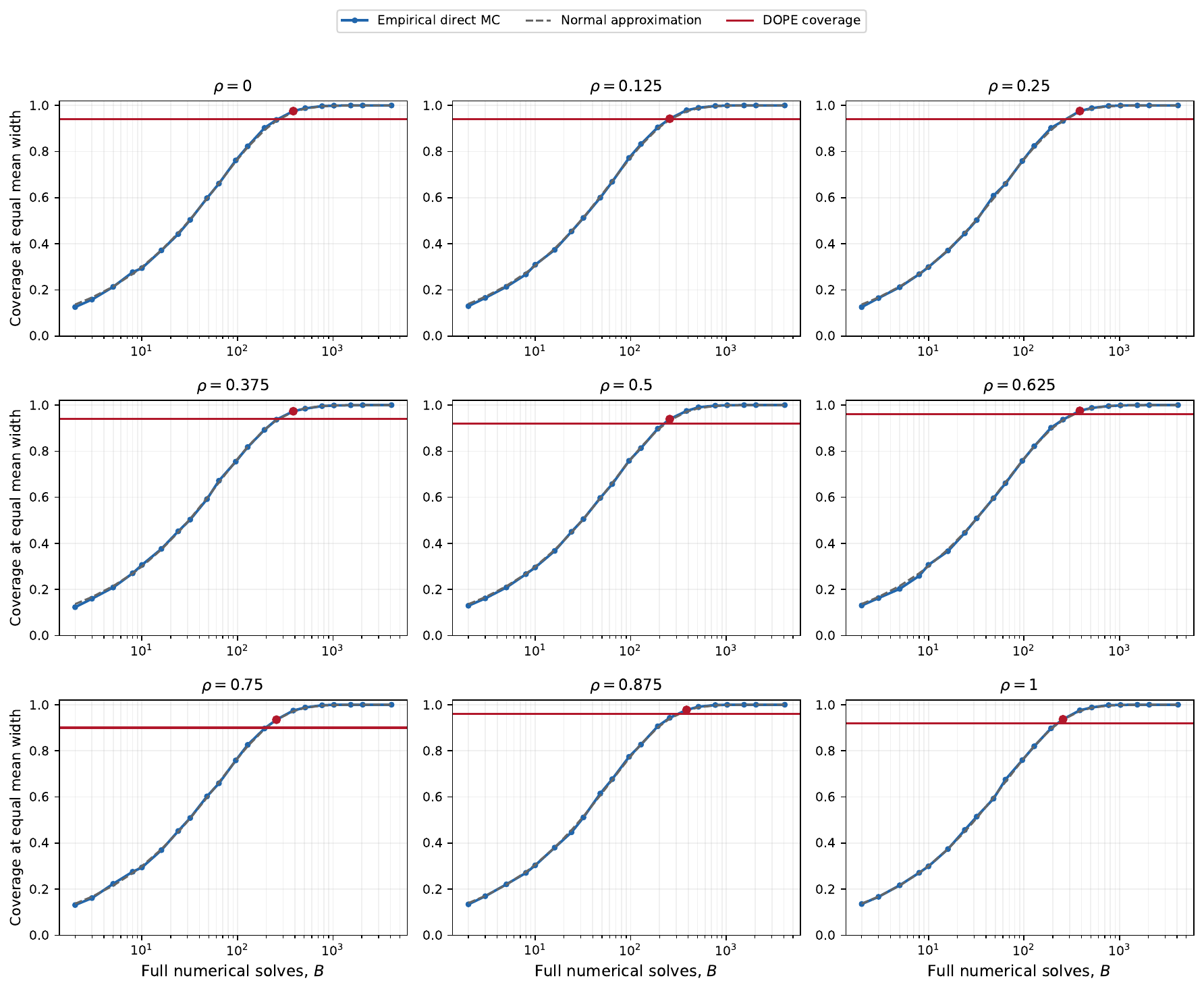}
    \rebuttal{(b) Smooth TAT}
\end{minipage}%
\hfill
\begin{minipage}{0.32\textwidth}
    \centering
    \includegraphics[width=\linewidth]{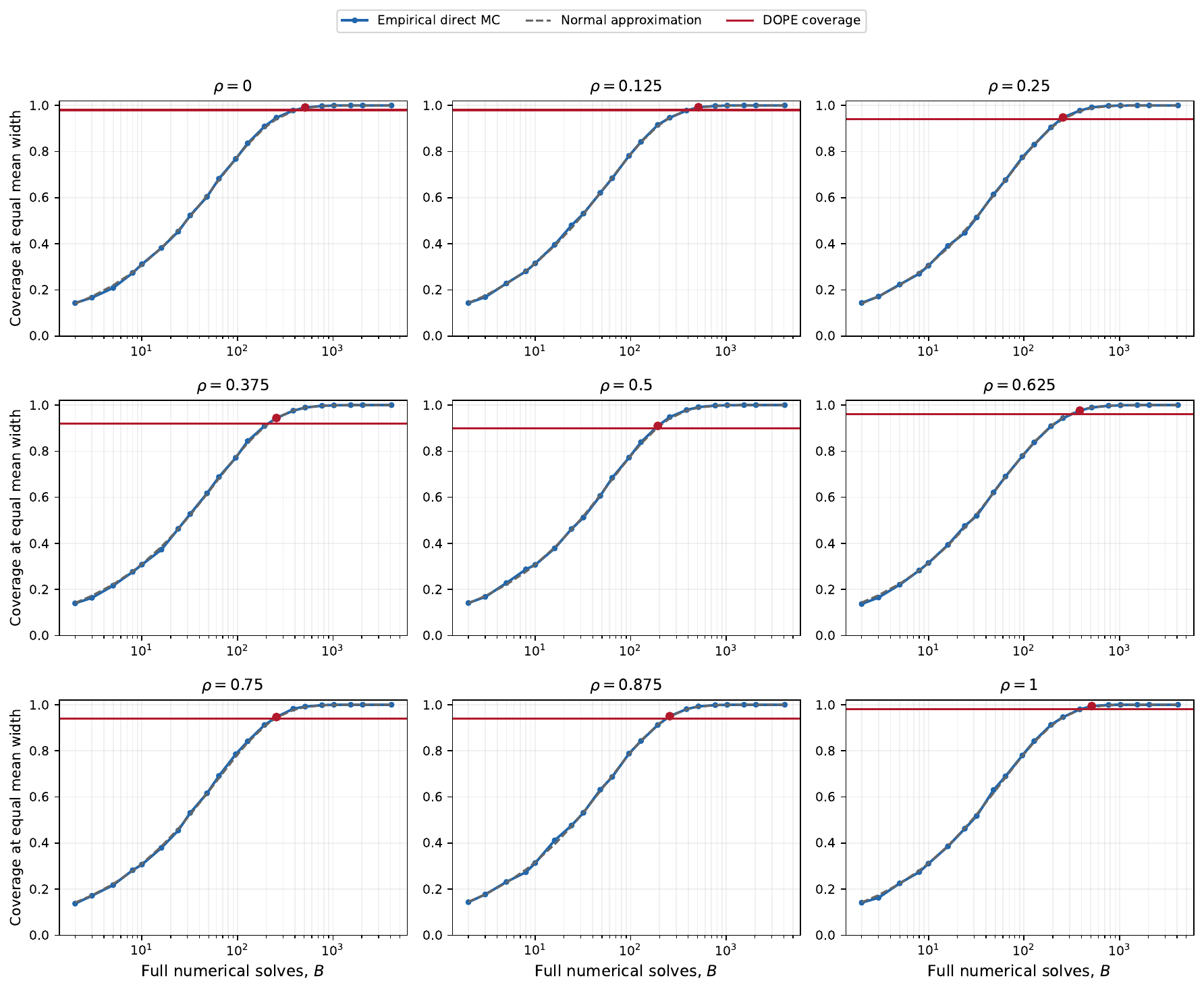}
    \rebuttal{(c) Soft Cmax}
\end{minipage}
\caption{\rebuttal{
\textbf{Direct-solver Monte Carlo: confidence-interval coverage.}
Empirical coverage versus full-solver budget, using intervals
with the same mean length as those of \method.
Horizontal reference lines indicate the empirical coverage
achieved by \method.
}}
\label{fig:direct_mc_coverage}
\end{figure*}

\paragraph{\rebuttal{Confidence interval coverage.}}
\rebuttal{
We further compare inferential precision by fixing the interval length to the mean length of the nominal $95\%$ confidence intervals
produced by \method. For each direct-solver budget $B$, we construct intervals centered on the Monte Carlo estimate with this same fixed width and measure their empirical coverage. To ensure independence, the $50$ reference pools are partitioned into five folds, with Monte Carlo samples drawn from separate donor pools. Coverage is estimated using $1000$ resampled estimates per fold and budget. We report the smallest tested $B$ attaining the empirical coverage of \method, using isotonic smoothing only to identify the crossing.
}

\rebuttal{
\textbf{Results:} Figure~\ref{fig:direct_mc_coverage} shows the coverage comparison at equal interval width. Matching the empirical coverage of
\method requires approximately \textbf{$192$--$512$ independent complete numerical solves} across the three functionals and sampling regimes.
These are equal-width comparison intervals.
}

\rebuttal{
\noindent\textbf{Discussion.}
These sanity checks demonstrate that \method achieves statistical accuracy and inferential precision comparable to similar if not larger amounts of complete numerical solves, despite using only sparse, noisy observations. Crucially, the fundamental motivation for neural operators is to amortize the computational cost of repeated numerical solves, particularly when full solutions are expensive or unavailable at evaluation time. Like other neural-operator methods, \method inherits this advantage; its specific contribution is to additionally provide debiased estimation and statistically principled inference from partial observations. Thus, our comparison establishes competitive accuracy even in a setting
where direct solves are readily available, while the broader advantages of neural operators become especially important
for computationally demanding PDEs.
}


\subsection{\rebuttal{Sensitivity to Riesz regularization}}
\label{sec:riesz_regularization_sensitivity}

\rebuttal{
We investigate the sensitivity of \method to Riesz regularization
using the PK smooth TAT functional at $\rho=0.5$.
Across $10$ independent repetitions, we vary
$\lambda\in\{0,10^{-3},10^{-2},10^{-1},1\}$,
keeping the learned solution operator $\hat S$ fixed within
each repetition and retraining only $\hat\beta$.
We evaluate the Riesz estimation RMSE against the oracle
$\beta_0$ and the resulting \method estimation RMSE.
Regularization is selected by cross-validation using the
held-out, unregularized Riesz objective, without access to
oracle weights or test outcomes.
}

\begin{wrapfigure}{r}{0.48\columnwidth}
\vspace{-0.5cm}
\centering
\includegraphics[width=\linewidth]{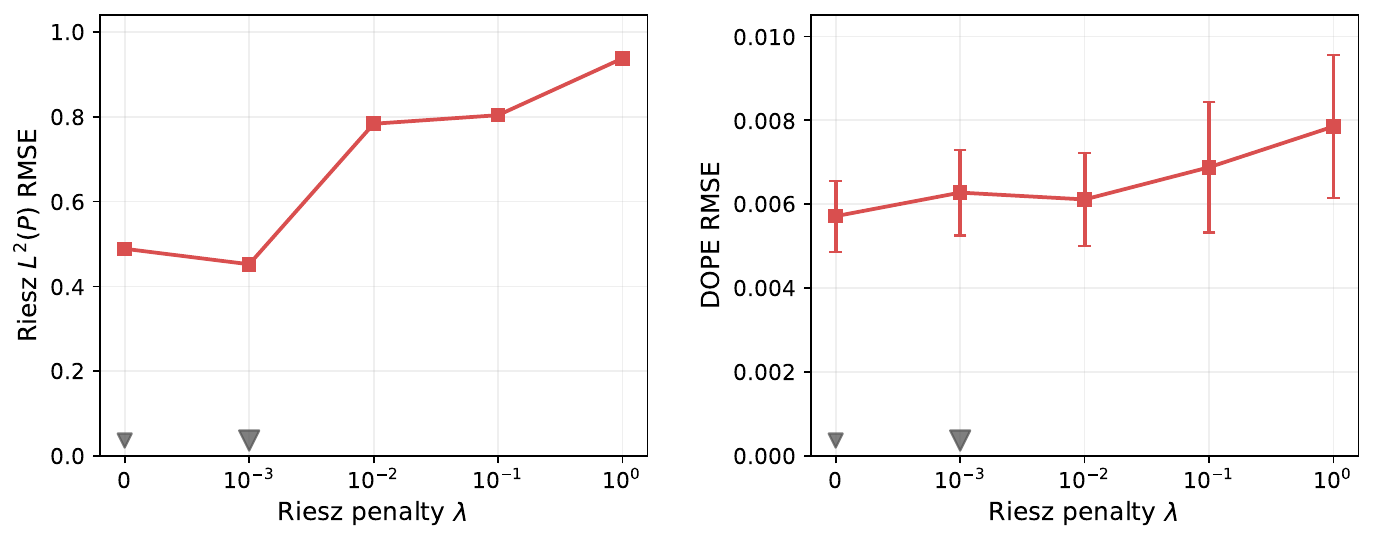}
\vspace{-0.2cm}
\caption{\rebuttal{
\textbf{Sensitivity to Riesz regularization.}
Riesz estimation RMSE (left) and \method estimation RMSE
(right) across regularization strengths $\lambda$ for
smooth TAT at $\rho=0.5$, over $10$ independent repetitions.
Error bars indicate standard errors.
}}
\label{fig:riesz_regularization_sensitivity}
\vspace{-0.2cm}
\end{wrapfigure}
\rebuttal{
\noindent\textbf{Results.}
Figure~\ref{fig:riesz_regularization_sensitivity} shows that the final \method estimator is substantially less sensitive to regularization than the learned Riesz representer. As $\lambda$ increases from $10^{-3}$ to $1$, Riesz RMSE more than doubles ($0.452$ to $0.938$), whereas \method RMSE increases  moderately ($0.00628$ to $0.00785$). Cross-validation consistently avoids excessive regularization.
}

\rebuttal{
Importantly, regularization does not alter the \textbf{Neyman-orthogonality} of the debiasing score at the true nuisance functions; it affects estimation accuracy through the remaining second-order nuisance errors. These results demonstrate that the Riesz objective provides
an effective, oracle-free criterion for selecting $\lambda$, while the final estimator remains comparatively robust to imperfect Riesz estimation. Excessive regularization can nevertheless degrade performance.
}


\subsection{Experiments with DeepONet backbone}\label{sec:deeponet}

We repeat our main experimental results from Section~\ref{sec:experiments} with DeepONet \citep{Lu.2021NO} as the neural backbone of both \method and the plug-in estimator. The new instantiation confirms our previous results: Table~\ref{tab:rmse_auc_smooth_tat_deeponet} \method clearly improves upon the plug-in estimator for both functionals and all sampling levels. At the same time, this also shows that \method is broadly applicable across different neural operator architectures.

\begin{table*}[h]
\centering
\setlength{\tabcolsep}{2.5pt} 
\small
\resizebox{\textwidth}{!}{%
\begin{tabular}{@{}lcccccccccccccccccc@{}}
\toprule
& \multicolumn{2}{c}{$\rho=0$}
& \multicolumn{2}{c}{$\rho=0.125$}
& \multicolumn{2}{c}{$\rho=0.25$}
& \multicolumn{2}{c}{$\rho=0.375$}
& \multicolumn{2}{c}{$\rho=0.5$}
& \multicolumn{2}{c}{$\rho=0.625$}
& \multicolumn{2}{c}{$\rho=0.75$}
& \multicolumn{2}{c}{$\rho=0.875$}
& \multicolumn{2}{c}{$\rho=1$} \\
\cmidrule(lr){2-3}\cmidrule(lr){4-5}\cmidrule(lr){6-7}\cmidrule(lr){8-9}
\cmidrule(lr){10-11}\cmidrule(lr){12-13}\cmidrule(lr){14-15}\cmidrule(lr){16-17}\cmidrule(lr){18-19}

& AUC & TAT
& AUC & TAT
& AUC & TAT
& AUC & TAT
& AUC & TAT
& AUC & TAT
& AUC & TAT
& AUC & TAT
& AUC & TAT \\
\midrule
\method: structured $w_g$ (ours)
& $0.80\pm0.08$ & $6.39\pm0.10$
& $0.69\pm0.07$ & $6.13\pm0.11$
& $0.89\pm0.08$ & $6.51\pm0.11$
& $1.08\pm0.10$ & $6.33\pm0.12$
& $0.86\pm0.08$ & $6.41\pm0.12$
& $0.76\pm0.07$ & $6.37\pm0.12$
& $0.92\pm0.10$ & $6.42\pm0.13$
& $0.84\pm0.08$ & $6.04\pm0.13$
& $0.87\pm0.09$ & $6.45\pm0.13$ \\
\method: oracle $\beta_0$ (ours)
& $0.72\pm0.06$ & $0.95\pm0.07$
& $0.73\pm0.09$ & $0.84\pm0.07$
& $0.80\pm0.07$ & $0.92\pm0.08$
& $1.02\pm0.11$ & $1.02\pm0.07$
& $0.82\pm0.07$ & $0.97\pm0.10$
& $0.75\pm0.07$ & $0.96\pm0.07$
& $0.90\pm0.09$ & $1.06\pm0.10$
& $0.88\pm0.09$ & $1.06\pm0.09$
& $0.85\pm0.09$ & $0.93\pm0.09$ \\
\midrule
Plug-in
& $1.78\pm0.21$ & $6.78\pm0.18$
& $1.70\pm0.16$ & $6.93\pm0.17$
& $1.90\pm0.15$ & $6.79\pm0.19$
& $1.93\pm0.19$ & $6.74\pm0.20$
& $1.77\pm0.15$ & $7.05\pm0.18$
& $1.48\pm0.13$ & $7.07\pm0.15$
& $2.03\pm0.18$ & $7.23\pm0.19$
& $2.31\pm0.22$ & $6.87\pm0.23$
& $2.06\pm0.18$ & $6.79\pm0.20$ \\
\method (\textbf{ours})
& $0.73\pm0.07$ & $6.34\pm0.11$
& $0.73\pm0.09$ & $6.11\pm0.11$
& $0.83\pm0.07$ & $6.52\pm0.11$
& $1.05\pm0.10$ & $6.36\pm0.13$
& $0.84\pm0.07$ & $6.49\pm0.11$
& $0.76\pm0.07$ & $6.45\pm0.12$
& $0.91\pm0.10$ & $6.56\pm0.12$
& $0.86\pm0.09$ & $6.21\pm0.13$
& $0.88\pm0.09$ & $6.61\pm0.13$ \\
\quad Rel. improvement
& $\greentext{58.9 \%}$ & $\greentext{6.5 \%}$
& $\greentext{57.2 \%}$ & $\greentext{11.8 \%}$
& $\greentext{56.5 \%}$ & $\greentext{4.0 \%}$
& $\greentext{45.6 \%}$ & $\greentext{5.7 \%}$
& $\greentext{52.6 \%}$ & $\greentext{7.9 \%}$
& $\greentext{49.1 \%}$ & $\greentext{8.7 \%}$
& $\greentext{55.3 \%}$ & $\greentext{9.2 \%}$
& $\greentext{62.8 \%}$ & $\greentext{9.6 \%}$
& $\greentext{57.4 \%}$ & $\greentext{2.6 \%}$ \\
\bottomrule
\end{tabular}%
}
\caption{\textbf{Comparison: plug-in estimator vs. \method.} RMSE $\pm$ SE ($\times 100$) across different functionals (area under the concentration curve [AUC] and time-above-threshold [TAT]) for different overlap levels $\rho$. $\Rightarrow$ \emph{\method outperforms the plug-in estimator.}}
\label{tab:rmse_auc_smooth_tat_deeponet}
\end{table*}

\subsection{PPI extension}\label{sec:ppi}
We report performance of a PPI-style variant of \method as in \Eqref{eq:PPI}, where we increase the number of unlabeled trajectories during inference. We again analyze the performance of \method against the plug-in estimator. Specifically, during inference, we increase the number $n_2$ of input trajectories without any observed point evaluations. We use the soft continuous maximum (soft cmax) functional (see Supplement~\ref{sec:riesz_examples}) in Table~\ref{tab:ppi} and on AUC in Table~\ref{tab:auc_ppi}. 

\begin{table*}[h]
\centering
\small
\resizebox{\textwidth}{!}{%
\begin{tabular}{lcccccccc}
\toprule
& $n_2=0$ & $n_2=64$ & $n_2=128$ & $n_2=256$ & $n_2=512$ & $n_2=1024$ & $n_2=2048$ & $n_2=4096$ \\
\midrule

Plug-in
& $2.90 \pm 0.28$
& $2.87 \pm 0.27$
& $2.69 \pm 0.23$
& $2.48 \pm 0.18$
& $2.43 \pm 0.21$
& $2.40 \pm 0.21$
& $2.35 \pm 0.21$
& $2.21 \pm 0.19$ \\
\quad Rel. improvement
& $0.0\%$
& $0.8\%$
& $7.3\%$
& $14.4\%$
& $16.3\%$
& $17.2\%$
& $19.0\%$
& $23.6\%$ \\

\midrule
\method
& $2.75 \pm 0.23$
& $2.64 \pm 0.19$
& $2.46 \pm 0.17$
& $2.27 \pm 0.15$
& $2.16 \pm 0.13$
& $2.01 \pm 0.12$
& $1.87 \pm 0.12$
& $1.77 \pm 0.11$ \\
\quad Rel. improvement
& $0.0\%$
& \greentext{$4.0\%$}
& \greentext{$10.6\%$}
& \greentext{$17.5\%$}
& \greentext{$21.5\%$}
& \greentext{$26.9\%$}
& \greentext{$32.0\%$}
& \greentext{$35.4\%$}\\
\bottomrule
\end{tabular}%
}
\vspace{-0.3cm}
\caption{\textbf{PPI results (cmax).} RMSE $\pm$ SE ($\times 100$), where compare the standard estimates against PPI-style estimates with additional unlabeled samples of size $n_2$. $\Rightarrow$ \emph{DOPE benefits from access to large unlabeled samples.}}
\label{tab:ppi}
\vspace{-0.3cm}
\end{table*}

\underline{Results:} Table~\ref{tab:ppi} shows that \method clearly benefits from additional, unlabeled inputs during inference time, and the gain is larger for \method than for the plug-in estimator. This is highly relevant for many scientific scenarios, where generating inputs is cheap, but measuring point evaluations is expensive.

Further, our results in Table~\ref{tab:auc_ppi} are consistent with the smooth continuous maximum functional. Both \method and the plug-in estimator benefit from abundant unlabeled trajectories, but \method clearly has stronger performance gains.

\begin{table}[H]
\centering
\small
\resizebox{\textwidth}{!}{%
\begin{tabular}{lcccccccc}
\toprule
& $n_2=0$ & $n_2=64$ & $n_2=128$ & $n_2=256$ & $n_2=512$ & $n_2=1024$ & $n_2=2048$ & $n_2=4096$ \\
\midrule

Plug-in
& $1.06 \pm 0.09$
& $1.05 \pm 0.09$
& $0.99 \pm 0.08$
& $0.92 \pm 0.08$
& $0.90 \pm 0.07$
& $0.89 \pm 0.08$
& $0.84 \pm 0.07$
& $0.77 \pm 0.06$ \\
\quad Rel. improvement
& $0.0\%$
& $1.1\%$
& $6.5\%$
& $12.8\%$
& $14.9\%$
& $15.6\%$
& $20.6\%$
& $26.9\%$ \\

\midrule
\method
& $0.72 \pm 0.06$
& $0.70 \pm 0.06$
& $0.59 \pm 0.05$
& $0.55 \pm 0.05$
& $0.48 \pm 0.05$
& $0.43 \pm 0.04$
& $0.36 \pm 0.04$
& $0.30 \pm 0.03$ \\
\quad Rel. improvement
& $0.0\%$
& \greentext{$2.4\%$}
& \greentext{$17.6\%$}
& \greentext{$23.3\%$}
& \greentext{$33.6\%$}
& \greentext{$40.6\%$}
& \greentext{$50.1\%$}
& \greentext{$58.6\%$}\\
\bottomrule
\end{tabular}%
}
\caption{\textbf{PPI results (cmax).} RMSE $\pm$ SE ($\times 100$) across unlabeled sample sizes $n_2$. The results are consistent with the PPI experiments in Table~\ref{tab:ppi}.}
\label{tab:auc_ppi}
\end{table}

\clearpage


\clearpage

\section{Pseudo-code for automatic DML with neural operators}\label{sec:algorithm}

\subsection{Average functionals} 

\begin{algorithm}[h]
\caption{Cross-fitted \methodlong}
\label{alg:algorithm}
{\small
\begin{algorithmic}[1]
\REQUIRE Observations $\mathcal D=\{O_i\}_{i=1}^n$, where
$O_i=(A_i,K_i,\{(X_{ik},Y_{ik})\}_{k=1}^{K_i})$; number of folds $J$; common discretization
$\{x_\ell\}_{\ell=1}^m\subset\Omega$; discretized functional $g_m:\mathbb R^m\to\mathbb R$;
\rebuttal{candidate regularization grid $\Lambda\subset[0,\infty)$}; neural operator model classes
$\mathcal S_0$ for $S$ and $\mathcal B$ for $\beta_0$
\ENSURE One-step estimate $\hat\theta$

\STATE Split $\{1,\dots,n\}$ into disjoint folds $I_1,\dots,I_J$
\FOR{$j=1,\dots,J$}
    \STATE Define training and evaluation indices
    \begin{align}
    I_j^{\mathrm{train}} \gets \{1,\dots,n\}\setminus I_j,
    \qquad
    I_j^{\mathrm{eval}} \gets I_j
    \end{align}

    \STATE \textbf{Train neural operator for the solution map:}
    \begin{align}
    \hat S^{(-j)}
    \gets
    \arg\min_{S_\eta \in\mathcal S}
    \frac{1}{|I_j^{\mathrm{train}}|}
    \sum_{i\in I_j^{\mathrm{train}}}
    \frac{1}{K_i}\sum_{k=1}^{K_i}
    \big(Y_{ik}-S(A_i)(X_{ik})\big)^2
    \end{align}

    \STATE \textbf{Instantiate the debiasing-weight model as a neural operator:}
    let $\beta_\eta\in\mathcal B$ be a neural operator mapping
    \begin{align}
    A \mapsto \beta_\eta(A,\cdot)
    \end{align}
    Optionally, if $w_g$ is known in closed form, use the structured parameterization
    \begin{align}
    \beta_\eta(A)(x)=\xi_\eta(A)(x)\,w_g(\hat S^{(-j)}(A))(x)
    \end{align}
    with $\xi_\eta$ parameterized by a neural operator.

    \STATE \rebuttal{\textbf{Select Riesz regularization:}
    Using inner cross-validation within $I_j^{\mathrm{train}}$,
    select $\lambda\in\Lambda$ minimizing the held-out,
    unregularized empirical Riesz loss. Refit on all
    $I_j^{\mathrm{train}}$ using the selected $\lambda$ below.}

    \STATE \textbf{Train neural operator for the debiasing weight via Riesz regression:}
    \begin{align}
    \hat\beta^{(-j)}
    \gets
    \arg\min_{\beta_\eta\in\mathcal B}
    \frac{1}{|I_j^{\mathrm{train}}|}
    \sum_{i\in I_j^{\mathrm{train}}}
    \left[
    \frac{1}{K_i}\sum_{k=1}^{K_i}\beta_\eta(A_i)(X_{ik})^2
    -
    2\,\mathrm{JVP}(g_m,u_i;b_i)
    \right]
    +\lambda \|\eta\|^2
    \end{align}
    where
    \begin{align}
    u_i\gets \big(\hat S^{(-j)}(A_i)(x_\ell)\big)_{\ell=1}^m \in \mathbb R^m,
    \qquad
    b_i\gets \big(\beta_\eta(A_i)(x_\ell)\big)_{\ell=1}^m \in \mathbb R^m,
    \end{align}
    and
    \begin{align}
    \mathrm{JVP}(g_m,u_i;b_i)
    \gets
    \left.\frac{\diff}{\diff t}g_m(u_i+t b_i)\right|_{t=0}.
    \end{align}

    \STATE \textbf{Evaluate pseudo-outcomes on the held-out fold:}
    \FOR{each $i\in I_j^{\mathrm{eval}}$}
        \STATE Compute
\begin{align}
\tilde\psi_i
\gets
g_m\!\left(\big(\hat S^{(-j)}(A_i)(x_\ell)\big)_{\ell=1}^m\right)
+
\frac{1}{K_i}\sum_{k=1}^{K_i}
\hat\beta^{(-j)}(A_i)(X_{ik})
\big(
Y_{ik}-\hat S^{(-j)}(A_i)(X_{ik})
\big).
\end{align}
    \ENDFOR
\ENDFOR

\STATE \textbf{Aggregate cross-fitted estimate:}
\begin{align}
\hat\theta
\gets
\frac{1}{n}\sum_{i=1}^n \tilde\psi_i
\end{align}

\RETURN One-step estimate $\hat\theta$
\end{algorithmic}
}
\end{algorithm}

\clearpage

\subsection{\rebuttal{Input-specific functionals}}

\begin{algorithm}[h]
\caption{\rebuttal{Cross-fitted \methodlong for input-specific functionals}}
\label{alg:algorithm_input_specific}
{\small
\begin{algorithmic}[1]
\REQUIRE \rebuttal{Observations $\mathcal D=\{O_i\}_{i=1}^n$;
functional $g_m$; scalar-output neural operator class $\mathcal R$;
remaining inputs as in Algorithm~\ref{alg:algorithm}}
\ENSURE \rebuttal{Debiased input-specific predictor $\hat \theta(a)$}

\STATE \rebuttal{\textbf{Compute cross-fitted pseudo-outcomes:}
Run Algorithm~\ref{alg:algorithm} through the computation of
$\tilde\psi_i$, retaining the fitted operators
$\hat S^{(-j)}$ and $\hat\beta^{(-j)}$ for each fold $I_j$.}

\STATE \rebuttal{\textbf{Extract residual corrections:}
For each $i\in I_j$, compute
\begin{align}
\hat c_i
\gets
\tilde\psi_i
-
g_m\!\left(
\big(\hat S^{(-j)}(A_i)(x_\ell)\big)_{\ell=1}^m
\right).
\end{align}}

\STATE \rebuttal{\textbf{Learn the conditional correction:}
Train a scalar-output neural operator
\begin{align}
\hat r
\gets
\arg\min_{r_\phi\in\mathcal R}
\frac{1}{n}\sum_{i=1}^n
\big(\hat c_i-r_\phi(A_i)\big)^2.
\end{align}}

\STATE \rebuttal{\textbf{Preserve the plug-in predictor:}
For a new input $a$, compute
\begin{align}
\hat \theta_{\mathrm{plug}}(a)
\gets
\sum_{j=1}^J \frac{|I_j|}{n}
g_m\!\left(
\big(\hat S^{(-j)}(a)(x_\ell)\big)_{\ell=1}^m
\right).
\end{align}}

\STATE \rebuttal{\textbf{Add the learned debiasing correction:}
\begin{align}
\boxed{
\hat \theta(a)
=
\underbrace{\hat \theta_{\mathrm{plug}}(a)}_{\text{Plug-in}}
+
\underbrace{\hat r(a)}_{\text{Debiasing correction}}.
}
\end{align}}

\RETURN \rebuttal{$\hat \theta(a)$}
\end{algorithmic}
}
\end{algorithm}

\rebuttal{Here, $\theta_0(a)=g(S_0(a))$ denotes the true input-specific functional, and $\hat \theta_{\mathrm{plug}}(a)$ its fold-averaged plug-in estimate. The second-stage operator $\hat r(a)$ learns the conditional expectation of the residual corrections, approximating the fold-averaged correction $r_{S,\beta}(a)$ from Theorem~\ref{thm:input_specific}. Thus, $\hat \theta(a)=\hat \theta_{\mathrm{plug}}(a)+\hat r(a)$ preserves the plug-in prediction while learning only its debiasing correction. The theoretical second-order nuisance bias guarantee additionally depends on the accuracy of this second-stage regression when implemented.}

\rebuttal{\noindent\textbf{Role of cross-fitting:}
Cross-fitting \citep{Chernozhukov.2022b, Chernozhukov.2022} ensures that each residual correction is evaluated using nuisance operators $\hat S$ and $\hat\beta$ trained without the corresponding subject's observations. Specifically, we construct the corrections on held-out folds and subsequently use them as training targets for the second-stage regression. This avoids in-sample overfitting bias in the pseudo-outcomes and supports the second-order nuisance bias guarantee in Theorem~\ref{thm:input_specific}. Importantly, Neyman-orthogonality is a property of the score itself; cross-fitting facilitates its statistical implementation and does not by itself estabish orthogonality.}

\clearpage

\section{Empirical objective of \method via automatic differentiation}\label{sec:automatic_differentiation}

\begin{tcolorbox}[
  colback=BrickRed!8,
  colframe=BrickRed!75!black,
  boxrule=0.8pt,
  arc=1.5mm,
  left=1.2mm,
  right=1.2mm,
  top=1mm,
  bottom=1mm,
  title=\textbf{Empirical objective and implementation of \method},
  fonttitle=\bfseries,
  coltitle=white,
  enhanced,
  breakable
]
We directly parameterize the correction weight $\hat\beta(\cdot)(\cdot)$ using a neural
operator, and minimize the empirical analogue of \Eqref{eq:primal_riesz}. Specifically, let $\hat S$, $\hat \beta$ denote a cross-fitted estimators of $S_0$, $\beta_0$, respectively. For each observation $A_i$, we evaluate both $\hat S(A_i)(\cdot)$ and $\hat\beta(A_i)(\cdot)$ on a common discretization $\{x_\ell\}_{\ell=1}^m \subset \Omega$, which yields
\begin{align}
u_i \coloneqq \big(\hat S(A_i)(x_\ell)\big)_{\ell=1}^m \in \mathbb R^m,
\qquad
b_i \coloneqq \big(\hat\beta(A_i)(x_\ell)\big)_{\ell=1}^m \in \mathbb R^m.
\end{align}
Let $g_m:\mathbb R^m \to \mathbb R$ denote the discretized implementation of
the functional $g$.\footnotemark  The empirical risk is then given by
\begin{align}
\min_{\hat\beta}
\frac{1}{n}\sum_{i=1}^n
\left[
\frac{1}{K_i}\sum_{k=1}^{K_i}
\hat\beta(A_i)(X_{ik})^2
-
2\,\mathrm{JVP}(g_m,u_i;b_i)
\right]
+\lambda\|\hat\beta\|^2,
\label{eq:primal_empirical_jvp}
\end{align}
where the directional derivative is computed as a Jacobian--vector product
\begin{align}
\mathrm{JVP}(g_m,u_i;b_i)
=
\left.\frac{\diff}{\diff t} g_m(u_i+t\,b_i)\right|_{t=0}
=
J_{g_m}(u_i)\,b_i.
\end{align}
This quantity can be evaluated efficiently via automatic differentiation (e.g., in
forward mode) and thus \textit{without} explicitly forming the Jacobian matrix.
\end{tcolorbox}
\footnotetext{Of note, while our theory is formulated in function space for the continuous functional $g$, the implemented estimator targets a numerical approximation to the population quantity. The resulting discretization and quadrature error is separate from the statistical error analyzed here.}

Importantly, even though the objective of \method in \Eqref{eq:primal_riesz} is formulated in function space, our empirical implementation in \method reduces to differentiating the scalar computational graph $t\mapsto g_m(u_i+t b_i)$ at $t=0$. In particular, \method requires only that both $\hat S(A)(\cdot)$ and $\hat\beta(A)(\cdot)$ can be evaluated on a common discretization of $\Omega$, which is naturally satisfied by neural operator architectures.

\clearpage
\section{Implementation details}\label{sec:implementation_details}

$\bullet$~\textbf{Neural operator implementation.}
To make the comparison in Section~\ref{sec:experiments} and Supplement~\ref{sec:additional_results} fair, \method and the plug-in baseline use the \textit{exact same }neural backbone. Specifically, the plug-in estimator only consists of the solution operator $\hat S$, which is shared for both methods (see Table~\ref{tab:implementation_backbones}).

\begin{table}[h]
\centering

\small
\resizebox{\textwidth}{!}{%
\begin{tabular}{llll}
\toprule
Setting & Operator(s) & Backbone & Architecture specification \\
\midrule
Pharmacokinetics (main)
& $\hat S$, $\hat\beta$
& FNO \citep{Li.2021NO}
& $4$ input channels, $32$ hidden channels, $1$ output channel, $3$ Fourier layers, $12$ retained modes
\\

2D Darcy-flow (main)
& $\hat S$, $\hat\beta$
& FNO \citep{Li.2021NO}
& $3$ input channels, $24$ hidden channels, $1$ output channel, $3$ Fourier layers, $8\times 8$ retained modes
\\

Pharmacokinetics (ablation)
& $\hat S$, $\hat\beta$
& DeepONet \citep{Lu.2021NO}
& $4$ input channels, $32$ branch hidden channels, $32$ trunk hidden channels, latent dimension $32$, $1$ output channel
\\
\bottomrule
\end{tabular}
}\caption{Neural operator architectures used in the main experiments and ablation study. For fairness, \method and the plug-in baseline share the same solution-operator backbone $\hat S$ within each setting.}
\label{tab:implementation_backbones}
\end{table}

\rebuttal{All experiments use disjoint training, validation, and test splits: $\hat S$ (shared by \method and plug-in) and $\hat\beta$ are fitted
exclusively on training subjects. Both estimators are evaluated on an independent test split; test observations are never used to train either nuisance model, but enter \method only through its debiasing correction.}

\textbf{Pharmacokinetics dataset:} Both nuisance components, namely the solution operator $\hat S$ and the debiasing operator $\hat\beta$, are implemented using the same Fourier neural operator (FNO) \citep{Li.2021NO} architecture. The shared model configuration uses the FNO family with $4$ input channels, $32$ hidden channels, $1$ output channel, $3$ Fourier layers, and $12$ retained modes. 
Both modules are trained with Adam optimization \citep{Kingma.2015} settings given by learning rate $10^{-3}$, weight decay $10^{-5}$, batch size $8$, and a maximum of $20$ epochs (see Table~\ref{tab:implementation_training}). 

\rebuttal{\textbf{Observation-only baseline} As an additional baseline, we estimate the target functionals using only each subject's $K=24$ sparse, noisy observations and the known sampling probabilities, without a learned surrogate.For AUC, we use direct inverse-design weighting,
\begin{align}\hat\theta_{\mathrm{obs}} =
n^{-1}\sum_i K^{-1}\sum_{k=1}^K
\xi_i(X_{ik})Y_{ik}
,
\end{align}
where $\xi_i(t_j)=w_j/p_j(\rho)$, and $w_j$ denotes the normalized trapezoidal quadrature weights on the $128$-point grid. For nonlinear  TAT, we independently reconstruct each trajectory using design-weighted Gaussian RBF kernel ridge regression, then evaluate the original functional on the reconstructed grid. The kernel bandwidth ($h=0.1$, with time normalized to $[0,1]$) and regularization ($\lambda=10^{-3}$) are selected once using inverse-design-weighted leave-one-observation-out validation. The prespecified grid is $h\in\{0.025,0.05,0.1,0.2,0.4\},$ and $\lambda\in
\{10^{-6},10^{-4},10^{-3},10^{-2},10^{-1}\}$. The baseline uses the same sparse test observations and independent
ground-truth reference as \method, with no access to full test trajectories.
}

\textbf{2D Darcy-flow dataset:} Here, $\hat S$ and $\hat \beta$ are also implemented with the same Fourier neural architecture. Specifically, they use the FNO family with $3$ input channels, $24$ hidden channels, $1$ output channel, $3$ Fourier layers, and $8$ retained modes for both $x$- and $y$-dimension, respectively. Both are learned with Adam optimization \citep{Kingma.2015} and learning rate $10^{-3}$, weight decay $10^{-5}$, batch size $8$, and a maximum of $20$ epochs (see Table~\ref{tab:implementation_training}).

\textbf{Ablation instantiation:} In Supplement~\ref{sec:additional_results}, we report a variant of our main results on the pharmacokinetics dataset, where we replace the Fourier neural operator with a  DeepONet \citep{Lu.2021NO}). For both the solution operator $\hat S$ of the plug-in estimator and \method, as well as the debiasing operator $\hat \beta$ of \method, the DeepONets have $4$ input channels, $32$ {branch} hidden channels, $32$ {trunc} hidden channels, a latent dimension of $32$, and $1$ output channel. They are trained via Adam \citep{Kingma.2015} with a learning rate of $10^{-3}$, a weight decay of $10^{-5}$, batch size $8$, and a maximum of $20$ epochs  (see Table~\ref{tab:implementation_training}).

\begin{table}[h]
\centering

\small
\resizebox{\textwidth}{!}{%
\begin{tabular}{lll}
\toprule
Component & Setting & Specification \\
\midrule
Optimizer
& all models
& Adam-style optimization with learning rate $10^{-3}$ and weight decay $10^{-5}$
\\

Batch size / epochs
& all models
& batch size $8$, maximum of $20$ epochs
\\

Solution operator training
& $\hat S$
& trained to predict the latent trajectory from $A_i$ using mean squared error evaluated only at the sampled observation locations
\\

Debiasing operator training
& $\hat\beta$ for \method
& trained separately via the implemented Riesz-style objective with the same optimizer hyperparameters as $\hat S$
\\

Riesz penalty
& $\hat\beta$ for \method
& $\lambda_{\mathrm{Riesz}} = 0.001$ (PK) / $0.01$ (Darcy)
\\

Cross-fitting
& final experiments
& $2$-fold cross-fitting on the \rebuttal{training} split
\\

Repeated runs
& final experiments
& results aggregated over $50$ repetitions
\\
\bottomrule
\end{tabular}
}\caption{Training and evaluation details used throughout the experiments.}
\label{tab:implementation_training}
\end{table}

$\bullet$~\textbf{Solution operator training.}
For both \method and the plug-in estimator, the neural operator $\hat S$ is trained to predict the full latent concentration trajectory from the functional input $A_i$. Its loss is the mean squared prediction error evaluated only at the subject-specific sampled observation times.

$\bullet$~\textbf{Debiasing operator training.}
For \method, the correction operator $\hat\beta$ is parameterized by the same FNO architecture, and trained separately through the implemented Riesz-style objective used by the one-step estimator. The corresponding debiasing module uses the same optimizer hyperparameters as the solution module and includes the penalty parameter $\lambda_{\mathrm{Riesz}} = 0.1.$

$\bullet$~\textbf{Cross-fitting and evaluation.}
The final experiments use $2$-fold cross-fitting on the \rebuttal{training} split. 
The reported results aggregate repeated-run summaries over $50$ repetitions.

$\bullet$~\textbf{Runtime.} For each method, training took approximately $30$ minutes over $50$ repeats with an AMD Ryzen 7 Pro CPU and 32GB of RAM. The runtime was comparable for \method and the plug-in estimator.

\clearpage


\section{Additional details on the  data}\label{sec:dgp}

\subsection{Pharmacokinetic model}\label{sec:pk_dgp}

We consider a pharmacokinetics benchmark in which each subject is characterized by a function-valued input $A_i$ consisting of a subject-specific dosing-rate profile together with subject-level clearance and volume parameters. The resulting latent concentration trajectory $u(A_i, \cdot)$ follows a standard one-compartment pharmacokinetic model \citep{Kwon.2002PK, Upton.2016PK}, where clearance governs elimination, volume determines how the administered drug is translated into concentration, and the time-varying dosing profile allows for heterogeneous infusion- or pulse-like regimens across subjects. 

The corresponding scientific targets are not the full concentration trajectories themselves, but rather low-dimensional summaries such as exposure, peak burden, or time above a clinically relevant threshold. 
We observe each trajectory only at sparse time points and vary the degree of input-dependent sampling through an irregularity parameter $\rho$, where $\rho=0$ corresponds to approximately uniform monitoring and larger values of $\rho$ increasingly concentrate observations near the subject-specific peak region.

We now introduce the exact data-generating process for our pharmacokinetics benchmark study. We summarize all parameters in Table~\ref{tab:pk_dgp_overview}.

\begin{table}[h]
\centering
\caption{Overview of the pharmacokinetics benchmark configuration used in the main study.}
\label{tab:pk_dgp_overview}
\small
\begin{tabular}{ll}
\toprule
Component & Specification \\
\midrule
Time horizon & $T=24$ \\
Grid size & $\Delta=128$ \\
Reference measure & normalized trapezoid rule on $[0,T]$ \\
Input channels & $[r_i(t), \log CL_i, \log V_i]$ \\
PK parameter means & $(0.0, 1.0)$ for $(\log CL_i,\log V_i)$ \\
PK parameter standard deviations & $(0.25, 0.20)$ \\
PK parameter correlation & $0.4$ \\
Dosing family & \texttt{one\_or\_two\_pulse} \\
Number of pulses & $M_i \in \{1,2\}$ uniformly \\
Pulse amplitude range & $[0.5, 2.0]$ \\
Pulse duration range & $[1.0, 4.0]$ \\
Observation design & \texttt{peak\_window\_mixture} \\
Observations per subject & $K=24$ \\
Noise standard deviation & $\sigma_\varepsilon = 0.002$ \\
Peak bandwidth & $h=5.0$ \\
Irregularity grid & $\rho \in \{0.00,0.125,\dots,1.00\}$ \\
Dataset split sizes & $(256,64,256)$ for train/val/test \\
Repeated runs & $50$ \\
Cross-fitting & $2$ folds \\
\bottomrule
\end{tabular}
\end{table}

$\bullet$~\textbf{Governing model:} We consider a one-dimensional pharmacokinetics benchmark  on the time horizon $[0,T]$ with $T=24$. The latent concentration trajectory $u_i$ for subject $i$ solves the one-compartment dynamics
\begin{align}
\frac{\diff u_i}{\diff t}(t) = -\frac{CL_i}{V_i}u_i(t) + \frac{1}{V_i}r_i(t),
\qquad
u_i(0)=0,
\end{align}
where $r_i$ is a subject-specific dosing-rate profile and $(CL_i,V_i)$ are clearance and volume parameters.

At the operator-learning level, each subject is represented through the function-valued input
\begin{align}
A_i(t_\delta) = \bigl[r_i(t_\delta), \log CL_i, \log V_i\bigr] \in \mathbb{R}^3,
\qquad \delta=1,\dots,\Delta,
\end{align}
where the last two channels are constant over the grid. We use a uniform grid of size $\Delta=128$ with points
\begin{align}
t_\delta = \frac{\delta-1}{\Delta-1}T, \qquad \delta=1,\dots,\Delta.
\end{align}

The subject-specific PK parameters are sampled as
\begin{align}
\begin{pmatrix}
 CL_i^{\text{log}} \\
 V_i^{\text{log}}
\end{pmatrix}
\sim
\mathcal{N}
\left(
\begin{pmatrix}
0.0 \\
1.0
\end{pmatrix},
\Sigma
\right),
\qquad
\Sigma =
\begin{pmatrix}
0.25^2 & 0.4 \cdot 0.25 \cdot 0.20 \\
0.4 \cdot 0.25 \cdot 0.20 & 0.20^2
\end{pmatrix},
\end{align}
and we set
\begin{align}
CL_i = \exp(CL_i^{\text{log}}), \qquad V_i = \exp( V_i^{\text{log}}).
\end{align}

The dosing profile follows the \texttt{one\_or\_two\_pulse} family. Specifically,
\begin{align}
M_i \sim \mathrm{Unif}\{1,2\}.
\end{align}
For each pulse $m=1,\dots,M_i$, we draw independently
\begin{align}
D_{im} \sim \mathrm{Unif}[1.0,4.0], \qquad
S_{im} \sim \mathrm{Unif}[0,T-D_{im}], \qquad
c_{im} \sim \mathrm{Unif}[0.5,2.0].
\end{align}
The dosing-rate profile is then
\begin{align}
r_i(t_\delta) = \sum_{m=1}^{M_i} c_{im}\,
\mathbf{1}\!\left\{S_{im} \le t_\delta \le \min \{ T,S_{im}+D_{im} \} \right\}.
\end{align}

$\bullet$~\textbf{Implementation level:} Trajectories are generated on the grid using the recursion
\begin{align}
u_{i,1}=0,
\end{align}
and, for $\delta=2,\dots,\Delta$,
\begin{align}
u_{i,\delta}
=
e^{-\kappa_i \Delta t}u_{i,\delta-1}
+
\frac{r_i(t_{\delta-1})}{V_i}\frac{1-e^{-\kappa_i \Delta t}}{\kappa_i},
\qquad
\kappa_i=\frac{CL_i}{V_i},
\end{align}
with the small-$\kappa_i$ fallback
\begin{align}
u_{i,\delta} = u_{i,\delta-1} + \Delta t \frac{r_i(t_{\delta-1})}{V_i}.
\end{align}

Quadrature uses normalized trapezoid weights
\begin{align}
w_1 = w_\Delta = \frac{1}{2(\Delta-1)}, \qquad
w_\delta = \frac{1}{\Delta-1}, \quad \delta=2,\dots,\Delta-1,
\end{align}
so that $\sum_{\delta=1}^\Delta w_\delta = 1$. The corresponding discrete inner product is
\begin{align}
\langle f,h\rangle_w = \sum_{\delta=1}^\Delta w_\delta f_\delta h_\delta.
\end{align}
\textbf{Sampling mechanism and irregularity parameter.}
We observe each latent trajectory only at \emph{sparse} time points. In all
reported pharmacokinetics experiments, each subject contributes $K=24$
observations. The latent trajectory is represented on an equally spaced grid
$\{t_\delta\}_{\delta=1}^{\Delta}$ over $[0,24]$, with $\Delta=128$.

We define the fixed early-time window mass function
\begin{align}
q^{\mathrm{early}}_{\delta}
=
\frac{\mathbf{1}\{|t_\delta-c|\le h\}}
{\sum_{\ell=1}^{\Delta}\mathbf{1}\{|t_\ell-c|\le h\}},
\qquad
c=1.0,
\qquad
h=1.0.
\end{align}
Thus, $q^{\mathrm{early}}$ is uniform over the grid points in the interval
$[0,2]$ and assigns zero mass elsewhere. The observation design is the mixture
\begin{align}
p_{\delta}(\rho)
=
\{1-\gamma(\rho)\}\frac{1}{\Delta}
+
\gamma(\rho)q^{\mathrm{early}}_{\delta},
\qquad
\gamma(\rho)=0.45\sqrt{\rho}.
\end{align}
Hence, $\rho=0$ corresponds to uniform monitoring over the grid, while larger
values of $\rho$ increasingly concentrate observations in the fixed early-time
window. At $\rho=1$, the early-window component receives mixture weight $0.45$.
In the main study, we use
\begin{align}
\rho \in
\{0.00,0.125,0.25,0.375,0.50,0.625,0.75,0.875,1.00\}.
\end{align}

Given $p_{\delta}(\rho)$, we draw
\begin{align}
J_{ik}\overset{\mathrm{iid}}{\sim}
\operatorname{Categorical}\{p_1(\rho),\ldots,p_\Delta(\rho)\},
\qquad k=1,\ldots,K,
\end{align}
with replacement, sort the resulting indices, and observe
\begin{align}
Y_{ik}
=
u_i(t_{J_{ik}})+\varepsilon_{ik},
\qquad
\varepsilon_{ik}\sim\mathcal{N}(0,\sigma_\varepsilon^2),
\qquad
\sigma_\varepsilon=0.002.
\end{align}

For \textbf{oracle calculations}, the inverse-design weight on the grid is
\begin{align}
\xi_i(t_\delta)
=
\frac{w_\delta}{p_\delta(\rho)},
\end{align}
where $w_\delta$ denotes the normalized trapezoidal quadrature weight on the
time grid.

\begin{table}[H]
\centering
\small
\resizebox{\textwidth}{!}{%
\begin{tabular}{llll}
\toprule
Functional & Definition & Implementation & Hyperparameters \\
\midrule
\texttt{auc}
& $    g(u) = (1 / T) \int_0^T u(t) \diff t  $
& $\displaystyle g(u)=\sum_{\delta=1}^\Delta w_\delta u_\delta$
& none \\
\texttt{tat}
& $    g(u) = (1 / T) \int_0^T \sigma (\kappa  (u(t) - c^\star)) \diff t $
& $\displaystyle g(u)=\sum_{\delta=1}^\Delta w_\delta\,\sigma\!\bigl(\kappa(u_\delta-c^\star)\bigr)$
& $\kappa=8.0,\; c^\star=0.5$ \\
\texttt{soft\_cmax}
& $g(u) = (1 / \lambda)  \log \big[ (1 / T) \int_0^T \exp(\lambda  u(t)) \diff t\big]$
& $\displaystyle g(u)=\frac{1}{\lambda}\log\!\left(\sum_{\delta=1}^\Delta w_\delta e^{\lambda u_\delta}\right)$
& $\lambda=6.0$ \\
\bottomrule
\end{tabular}
}
\caption{Functional-specific hyperparameters used in the reported pharmacokinetics experiments.}
\label{tab:pk_functionals_overview}
\end{table}

$\bullet$~\textbf{Functional estimation:} The population target is the mean functional
\begin{align}
\theta = \mathbb{E}[g(S_0(A))],
\end{align}
where we use $2000$ datapoints to compute the ground-truth average functional. 

Table~\ref{tab:pk_functionals_overview} lists the functional-specific hyperparameters. More details on the closed-form functional derivatives and their Riesz representers are given in Supplement~\ref{sec:riesz_examples}. All experiments are repeated $50$ times.

\clearpage

\subsection{Darcy-flow model}\label{sec:darcy_dgp}

We consider a two-dimensional Darcy-flow benchmark in which each sample is characterized by a spatially varying coefficient field $a_i$ on the unit square, and the latent scientific state is the corresponding solution field $u(a_i,\cdot)$ of an elliptic PDE. The operator-learning task is therefore to map coefficient fields to solution fields, but the inferential target is not the full field itself: instead, we study low-dimensional functionals of the solution, namely smooth excess-above-threshold summaries of the spatial field. The Darcy-flow model is \emph{standard for benchmarking neural operators} \citep{Kovachki.2023NO}.

We observe each latent solution only at a sparse set of noisy spatial locations and allow the observation design to depend on the realized coefficient and solution field. The observation design is fully adaptive, and concentrates samples in regions with small coefficients and rapidly varying solutions. 

We now introduce the exact data-generating process for our Darcy-flow benchmark study. We summarize all benchmark parameters in Table~\ref{tab:darcy_dgp_overview}.

\begin{table}[h]
\centering
\caption{Overview of the Darcy-flow benchmark configuration used in the main study.}
\label{tab:darcy_dgp_overview}
\small
\begin{tabular}{ll}
\toprule
Component & Specification \\
\midrule
Spatial domain & $\Omega=(0,1)^2$ \\
Grid size & $H=W=17$ \\
Reference measure & normalized 2D trapezoid rule on $\Omega$ \\
Coefficient family & \texttt{piecewise\_constant\_transformed\_gaussian\_field} \\
Coarse coefficient grid & $5\times 5$ \\
Number of cosine basis modes & $3\times 3$ \\
Log-coefficient scale & $0.65$ \\
Lower coefficient bound & $0.2$ \\
Observation design & \texttt{inverse\_energy\_mixture} \\
Observations per sample & $K=24$ \\
Noise standard deviation & $\sigma_\varepsilon = 0.01$ \\
Adaptive design temperature & $\lambda = 7.0$ \\
Adaptive design floor & $\varepsilon = 10^{-3}$ \\
Irregularity parameter & $\rho = 1.0$ \\
Dataset split sizes & $(256,64,64)$ for train/val/test \\
Truth-pool size & $2000$ \\
Repeated runs & $50$ \\
Cross-fitting & $2$ folds \\
\bottomrule
\end{tabular}
\end{table}

$\bullet$~\textbf{Governing model:}
We consider the elliptic Darcy problem on $\Omega=(0,1)^2$,
\begin{align}
-\nabla \cdot \bigl(a_i(x)\nabla u_i(x)\bigr) = 1,
\qquad x\in (0,1)^2,
\end{align}
with homogeneous Dirichlet boundary conditions
\begin{align}
u_i(x)=0,
\qquad x\in \partial(0,1)^2.
\end{align}
Thus, the forcing term is constant and equal to $1$, and all sample-specific heterogeneity enters through the random coefficient field $a_i$.

At the operator-learning level, each sample is represented by a coefficient field on a tensor-product grid with
\begin{align}
x_p = \frac{p-1}{H-1}, \qquad p=1,\dots,H,
\qquad
y_q = \frac{q-1}{W-1}, \qquad q=1,\dots,W,
\end{align}
where $H=W=17$, and grid locations
\begin{align}
X_{pq}=(x_p,y_q)\in [0,1]^2.
\end{align}

The coefficient field is sampled from the family
\texttt{piecewise\_constant\_transformed\_gaussian\_field}. First, define a coarse $5\times 5$ grid by
\begin{align}
\bar x_r = \frac{r-1}{5-1}, \qquad r=1,\dots,5,
\qquad
\bar y_s = \frac{s-1}{5-1}, \qquad s=1,\dots,5.
\end{align}
For each pair of modes $(m_x,m_y)\in\{1,2,3\}^2$, we use the cosine basis
\begin{align}
\phi_{m_x,m_y}(\bar x_r,\bar y_s)
=
\sqrt{2}\cos(\pi m_x \bar x_r)\,
\sqrt{2}\cos(\pi m_y \bar y_s).
\end{align}
Independent Gaussian coefficients $\xi_{m_x,m_y}\sim\mathcal N(0,1)$ are drawn, and the latent coarse field is
\begin{align}
Z_{rs}
=
\frac{0.65}{3}
\sum_{m_x=1}^3\sum_{m_y=1}^3
\xi_{m_x,m_y}\,\phi_{m_x,m_y}(\bar x_r,\bar y_s).
\end{align}
The coarse coefficient field is then obtained by exponentiation and lower clipping,
\begin{align}
\bar a_{rs} = \exp(Z_{rs}),
\qquad
\bar a_{rs}\leftarrow \max \{ \bar a_{rs},0.2 \}.
\end{align}
Finally, the coarse field is mapped to the $17\times 17$ grid by nearest-cell replication in both spatial directions, yielding a piecewise-constant coefficient field on the fine grid.

$\bullet$~\textbf{Implementation level:}
The PDE is solved on the $17\times 17$ grid by a finite-difference discretization of the divergence-form operator. Let
\begin{align}
h_x=h_y=\frac{1}{17-1}.
\end{align}
On interior grid points, the code assembles the linear system corresponding to
\begin{align}
-\partial_x(a\,\partial_x u)-\partial_y(a\,\partial_y u)=1.
\end{align}
To preserve flux continuity across coefficient jumps, harmonic means are used at cell interfaces. For example, at the west interface,
\begin{align}
a_{i-\frac12,j}^{(h)}
=
\frac{2 a_{ij} a_{i-1,j}}{a_{ij}+a_{i-1,j}},
\end{align}
with analogous formulas for east, north, and south interfaces.

The resulting interior system has dimensions
\begin{align}
(17-2) \times (17-2)=225,
\end{align}
is assembled explicitly as a dense matrix, and is solved using \texttt{numpy.linalg.solve}. Boundary values are then set to zero to obtain the full $17\times 17$ solution field.

Quadrature uses a normalized two-dimensional trapezoid rule. If $w^{(x)}$ and $w^{(y)}$ denote the one-dimensional trapezoid weights,
\begin{align}
w^{(x)}_1=w^{(x)}_H=\frac12, \qquad
w^{(x)}_p=1 \ \text{for } p=2,\dots,H-1,
\end{align}
and analogously in the $y$-direction, then the unnormalized two-dimensional weights are
\begin{align}
\widetilde w_{pq}=w^{(x)}_p w^{(y)}_q.
\end{align}
These are normalized to sum to one, i.e.,
\begin{align}
w_{pq}
=
\frac{\widetilde w_{pq}}{\sum_{r=1}^H\sum_{s=1}^W \widetilde w_{rs}}.
\end{align}
The corresponding discrete inner product is
\begin{align}
\langle f,h\rangle_w
=
\sum_{p=1}^H\sum_{q=1}^W w_{pq} f_{pq} h_{pq}.
\end{align}
\textbf{Sampling mechanism and irregularity parameter.}
We observe each latent solution field only at a sparse set of spatial
locations. In all reported Darcy-flow experiments, each sample contributes
\begin{align}
K=24
\end{align}
noisy observations. The solution is represented on a $17\times 17$ spatial
grid. Let $a_{pq}$ denote the coefficient field and $u_{pq}$ the corresponding
latent solution at grid location $x_{pq}$.

The observation design is \texttt{inverse\_energy\_mixture}. First, we define
the inverse-coefficient factor
\begin{align}
a^{-1}_{pq}
=
\frac{1}{\max\{a_{pq},10^{-8}\}}.
\end{align}
We then compute the discrete solution-gradient magnitude
\begin{align}
G_{pq}
=
\left[
\bigl(u_{p+1,q}-u_{p-1,q}\bigr)^2
+
\bigl(u_{p,q+1}-u_{p,q-1}\bigr)^2
\right]^{1/2},
\end{align}
where grid indices are interpreted periodically, as in the implementation.
The inverse-energy saliency map is
\begin{align}
S_{pq}
=
a^{-1}_{pq}
\bigl\{
0.5+|u_{pq}|+G_{pq}
\bigr\}.
\end{align}
After centering and normalizing this field,
\begin{align}
\widetilde S_{pq}
=
\frac{S_{pq}-\overline S}
{\max_{r,s}|S_{rs}-\overline S|},
\end{align}
with $\widetilde S_{pq}=0$ if the denominator is zero, we define the adaptive
design distribution
\begin{align}
q_{pq}
=
\frac{
\exp\!\bigl\{7|\widetilde S_{pq}|\bigr\}+10^{-3}
}{
\sum_{r,s}
\left[
\exp\!\bigl\{7|\widetilde S_{rs}|\bigr\}+10^{-3}
\right]
}.
\end{align}

Let $G_x=G_y=17$ denote the grid dimensions. The sampling distribution is the
mixture
\begin{align}
p_{pq}(\rho)
=
(1-\rho)\frac{1}{G_xG_y}
+
\rho q_{pq},
\qquad
\rho\in[0,1].
\end{align}
Thus, $\rho=0$ corresponds to uniform spatial sampling and $\rho=1$ to fully
adaptive inverse-energy sampling. The reported Darcy-flow experiments use
\begin{align}
\rho=1.
\end{align}
The six reported values of $\kappa$ index the functional, rather than the
sampling design.

Given $p_{pq}(\rho)$, we draw
\begin{align}
J_{ik}\overset{\mathrm{iid}}{\sim}
\operatorname{Categorical}\{p_{pq}(\rho)\}_{p,q},
\qquad
k=1,\ldots,K,
\end{align}
with replacement, sort the flattened grid indices, and observe
\begin{align}
Y_{ik}
=
u_i(X_{J_{ik}})+\varepsilon_{ik},
\qquad
\varepsilon_{ik}\sim\mathcal N(0,\sigma_\varepsilon^2),
\qquad
\sigma_\varepsilon=0.01.
\end{align}

For \textbf{oracle calculations}, the inverse-design weight at grid location $x_{pq}$
is
\begin{align}
\xi_i(x_{pq})
=
\frac{w_{pq}}{p_{pq}(\rho)},
\end{align}
where $w_{pq}$ is the normalized two-dimensional trapezoidal quadrature weight.

\begin{table}[h]
\centering
\small
\resizebox{\textwidth}{!}{%
\begin{tabular}{llll}
\toprule
Functional & Definition & Implementation & Hyperparameters \\
\midrule
\texttt{smooth\_excess\_above\_threshold}
&
$\displaystyle
g_{\mathrm{exc}}(u)
=
\int_\Omega
\frac{1}{\kappa}\log\!\bigl(1+\exp(\kappa^\star(u(x)-c))\bigr)\,\diff\mu(x)
$
&
$\displaystyle
g_{\mathrm{exc}}(u)
=
\sum_{p=1}^{H}\sum_{q=1}^{W}
w_{pq}\,
\frac{1}{\kappa}
\log\!\bigl(1+\exp(\kappa(u_{pq}-c))\bigr)
$
&
$c=0.5,\; \kappa^\star = 7.5 + 2.5\kappa; 
\kappa\in\{0.0,0.2,0.4,0.6,0.8,1.0\}$
\\
\bottomrule
\end{tabular}
}\caption{Functional-specific hyperparameters used in the reported Darcy-flow experiments.}
\label{tab:darcy_functionals_overview}
\end{table}

$\bullet$~\textbf{Functional estimation:} The population target is the mean functional
\begin{align}
\theta=\mathbb E[g(S_0(A))],
\end{align}
where the benchmark ground truth is computed from a disjoint truth pool of size $2000$. The estimator itself is evaluated on the held-out test split of size $64$, and all reported results are aggregated over $50$ repeated runs.

For the smooth excess-above-threshold functional, the corresponding Riesz representer under the benchmark reference measure is the pointwise sigmoid field
\begin{align}
w_g(u)(x)=\sigma\bigl(\kappa^\star(u(x)-c)\bigr),
\qquad
\sigma(z)=\frac{1}{1+e^{-z}}.
\end{align}
In discrete form, the same pointwise expression is evaluated on the spatial grid (see Table~\ref{tab:darcy_functionals_overview}). More details on the closed-form functional derivatives and their Riesz representers are given in Supplement~\ref{sec:riesz_examples}.

\clearpage

\section{Examples of functional derivatives and their Riesz representers}\label{sec:riesz_examples}
In some cases, the functional derivatives $\Diff g_u (h)$ and their Riesz representers $w_g(u)$ can be computed analytically. In the following, we summarize the functionals used in our experiments (Section~\ref{sec:experiments}), and provide closed-form solutions for the functional derivatives and their Riesz representers, which allows us to properly validate our results.

Note that these closed-form solutions are not always easily tractable but can involve complex derivations. Yet, \method is applicable even when these closed-form solutions are \textbf{\underline{not}} available (see Section~\ref{sec:onestep} and Section~\ref{sec:autodml}). In the following, all integrals are w.r.t. Lebesgue-measure $\mu$.

$\bullet$~\textbf{Normalized AUC:}
The normalized area under the (concentration) curve is given by
\begin{align}
    g(u) = (1 / T) \int_0^T u(t) \diff t    .
\end{align}

The corresponding Fréchet derivative in direction $h$ is
\begin{align}
    \Diff g_{u}(h) = (1 / T) \int _0^T h(t)  \diff t.
\end{align}
    
Then, its Riesz representer under $\mu(\diff t)=\diff t/T$ is
\begin{align}
    w_{}(u)(t) = 1.
\end{align}
For our experiments, this functional is especially interesting as it isolates the inverse-design component; $w_g$ is constant, and therefore $\beta_0(\cdot)(\cdot) = \xi_0(\cdot)(\cdot)$.

$\bullet$~\textbf{Time-above-threshold:}
We define a smooth version of the time-above-threshold (TAT) as
\begin{align}
    g(u) = (1 / T) \int_0^T \sigma (\kappa  (u(t) - c^\star)) \diff t    ,
\end{align}
where $c^\star$ is a clinically meaningful concentration threshold, $\kappa > 0$ controls the sharpness of the smooth threshold, and $\sigma(z) = 1 / (1 + \exp(-z))$. This is a smooth version of the time-above-threshold / time-in-range functional, since
\begin{align}
\sigma(\kappa x)
\;\xrightarrow[\kappa \to \infty]{}\;
\mathbf{1}\{x \ge 0\},
\end{align}
so, pointwise in \(t\),
\begin{align}
\sigma\big(\kappa (u(t)-c^\star)\big)
\;\to\;
\mathbf{1}\{u(t) \ge c^\star\}.
\end{align}
Therefore,
\begin{align}
g_{\mathrm{tat}}(u)
=
\frac{1}{T}\int_0^T \sigma\big(\kappa (u(t)-c^\star)\big)\, \diff t
\;\to\;
\frac{1}{T}\int_0^T \mathbf{1}\{u(t) \ge c^\star\}\, \diff t.
\end{align}

The corresponding Fréchet derivative in direction h is
\begin{align}
    \Diff g_u(h)
      = (1 / T) \int_0^T
          \kappa  \sigma(\kappa  (u(t) - c^\star))
                (1 - \sigma(\kappa  (u(t) - c^\star)))
                h(t) \diff t.
\end{align}
Its Riesz representer under $\mu(\diff t)=\diff t/T$ is
\begin{align}
    w_g(u)(t)
      = \kappa \sigma(\kappa  (u(t) - c^\star))
                 (1 - \sigma(\kappa  (u(t) - c^\star))).
\end{align}
    
This functional is PK-motivated and gives a nontrivial, threshold-localized sensitivity profile.

$\bullet$~\textbf{Soft cmax:}
We define the soft cmax functional as
\begin{align}
    g(u) = (1 / \lambda)  \log \big[ (1 / T) \int_0^T \exp(\lambda  u(t)) \diff t\big],
\end{align}   
where the temperature parameter $\lambda > 0$ controls how sharply this approximates peak concentration.
This is a standard smooth proxy for peak concentration ``cmax'' (continuous maximum functional)
\begin{align}
\text{cmax}(u) \coloneqq \sup_{t \in [0,T]} u(t),
\end{align}
since
\begin{align}
\frac{1}{\lambda}\log\!\left(\frac{1}{T}\int_0^T e^{\lambda u(t)}\, \diff t\right)
\;\longrightarrow\;
\sup_{t \in [0,T]} u(t)
\quad \text{as } \lambda \to +\infty.
\end{align}

Its Fréchet derivative in direction $h$ is
\begin{align}
    \Diff g_u (h)
      = (1 / T) \int_0^T
          [ \exp(\lambda  u(t)) / ((1 / T) \int_0^T \exp(\lambda u(s)) \diff s) ]
           h(t) \diff t.
\end{align}
Equivalently, under $\mu(\diff t)=\diff t/T$,
\begin{align}
    D_gu(h)
      =   \int_0^T w_g(u)(t) h(t) \mu(\diff t),
\end{align}
with Riesz representer
\begin{align}
    w_g(u)(t)
      = \exp(\lambda u(t)) / [ (1 / T) \int_0^T \exp(\lambda  u(s)) \diff s ].
\end{align}
    
This functional is peak-sensitive, nonlinear, and temporally varying, while still having a closed-form derivative.

$\bullet$~\textbf{Smooth excess-above-threshold:}
We define the smooth excess-above-threshold functional as
\begin{align}
    g(u)
    =
    \int_\Omega
    \frac{1}{\kappa}\log\!\bigl(1+\exp(\kappa^\star (u(x)-c))\bigr)\,\diff\mu(x),
\end{align}
where the sharpness parameter $\kappa^\star>0$ controls how rapidly the functional transitions around the threshold $c$.
This is a standard smooth proxy for the excess-above-threshold functional
\begin{align}
    \mathrm{excess}_c(u)
    \coloneqq
    \int_\Omega (u(x)-c)_+\,\diff\mu(x),
\end{align}
since
\begin{align}
    \frac{1}{\kappa^\star}\log\!\bigl(1+\exp(\kappa^\star (u(x)-c))\bigr)
    \;\longrightarrow\;
    (u(x)-c)_+
    \quad \text{as } \kappa^\star \to +\infty
\end{align}
pointwise in $x$, and hence
\begin{align}
    g(u)
    \;\longrightarrow\;
    \int_\Omega (u(x)-c)_+\,\diff\mu(x)
    \quad \text{as } \kappa^\star \to +\infty.
\end{align}

Its Fréchet derivative in direction $h$ is
\begin{align}
    \Diff g_u(h)
    =
    \int_\Omega
    \sigma\!\bigl(\kappa^\star (u(x)-c)\bigr)\,h(x)\,\diff\mu(x),
    \qquad
    \sigma(z)=\frac{1}{1+e^{-z}}.
\end{align}
Equivalently,
\begin{align}
    \Diff g_u(h)
    =
    \int_\Omega w_g(u)(x)\,h(x)\,\diff\mu(x),
\end{align}
with Riesz representer
\begin{align}
    w_g(u)(x)
    =
    \sigma\!\bigl(\kappa^\star (u(x)-c)\bigr).
\end{align}

This functional is nonlinear and threshold-sensitive, while still admitting a closed-form derivative and a simple Riesz representer. Note that in our experiments, we use a linear transformation $\kappa^*\coloneqq f(\kappa)$ (see Supplement~\ref{sec:darcy_dgp}).


\end{document}